\pdfoutput=1
\documentclass[11pt]{article}
\usepackage{ACL2023}
\newtoggle{publish}
\toggletrue{publish}

\usepackage{custom}
\newcommand{\mymacro}[1]{{\color{MacroColor} #1}}
\newcommand{\mathhl}[2][ETHRed!25]{\mathchoice%
    {\colorbox{#1}{$\displaystyle#2$}}%
    {\colorbox{#1}{$\textstyle#2$}}%
    {\colorbox{#1}{$\scriptstyle#2$}}%
    {\colorbox{#1}{$\scriptscriptstyle#2$}}}%

\newcommand{\defn}[1]{\textbf{#1}}

\newcommand{\paroutline}[3][false]{%
    \ifnum\pdfstrcmp{#1}{true}=0
        #3% Print only the second argument
    \else
        [\textit{\textcolor{DiverseMagenta}{#2}}] \textcolor{AccentBlue}{#3}% Print both arguments
    \fi
}

\newcommand{\pfsaAcr}{{\mymacro{PFSA}}\xspace}
\newcommand{\ppdaAcr}{{\mymacro{PPDA}}\xspace}

\newcommand{\inv}[1]{{\mymacro{#1^{-1}}}}

\newcommand{\pdens}{{\mymacro{ p}}}

\newcommand{\qdens}{{\mymacro{ q}}}

\newcommand{\acronymText}[1]{\mymacro{\text{#1}}}
\newcommand{\ptm}{{\mymacro{\acronymText{PTM}}}\xspace}

\newcommand{\fst}{{\mymacro{\acronymText{FST}}}\xspace}

\newcommand{\twopda}{{\mymacro{\acronymText{2PDA}}}\xspace}
\newcommand{\cfg}{{\mymacro{\acronymText{CFG}}}\xspace}

\newcommand{\Q}{{\mymacro{ \mathbb{Q}}}}
\newcommand{\Qnonnegative}{{\mymacro{\Q_{\geq 0}}}}
\newcommand{\R}{{\mymacro{ \mathbb{R}}}}

\newcommand{\func}{{\mymacro{ f}}}

\newcommand{\vfunc}{{\mymacro{ \mathbf{f}}}}

\newcommand{\norm}[1]{{\mymacro{ \left\lVert #1 \right\rVert}}}
\newcommand{\abs}[1]{{\mymacro{ \left| #1 \right|}}}

\newcommand{\innerProd}[2]{{\mymacro{ \langle #1, #2 \rangle }}}

\newcommand{\alphabet}{{\mymacro{ \Sigma}}}

\newcommand{\stackalphabet}{{\mymacro{ \Gamma}}}
\newcommand{\oalphabet}{\mymacro{\Xi}}
\newcommand{\yalphabet}{\mymacro{\Upsilon}}
\newcommand{\outalphabet}{{\mymacro{ \Delta}}}
\newcommand{\eosalphabet}{{\mymacro{ \overline{\alphabet}}}}
\newcommand{\bosalphabet}{{\mymacro{ \underline{\alphabet}}}}
\newcommand{\lang}{\mymacro{L}}
\newcommand{\kleene}[1]{{\mymacro{#1^*}}}

\newcommand{\regfn}{{\mymacro{\phi}}}

\newcommand{\str}{{\mymacro{\boldsymbol{y}}}}
\newcommand{\cotstr}{{\cotMacro{\str}}}
\newcommand{\cotstrlt}{{\mymacro{\str_{<\tstep}^{\scaleto{\thought}{7pt}}}}}
\newcommand{\cotsym}{{\cotMacro{\sym}}}

\newcommand{\cotbos}{{\mymacro{\bos^{\scaleto{\thought}{7pt}}}}}

\newcommand{\strlt}{{\mymacro{ \str_{<\tstep}}}}
\newcommand{\strlet}{{\mymacro{ \str_{\leq\tstep}}}}

\newcommand{\strlen}{{\mymacro{T}}}

\newcommand{\strx}{{\mymacro{\boldsymbol{x}}}}
\newcommand{\stry}{{\mymacro{\boldsymbol{y}}}}
\newcommand{\strz}{{\mymacro{\boldsymbol{z}}}}
\newcommand{\sym}{{\mymacro{y}}}
\newcommand{\eossym}{{\mymacro{\overline{\sym}}}}
\newcommand{\bossym}{{\mymacro{\underline{\sym}}}}

\newcommand{\symx}{{\mymacro{x}}}
\newcommand{\symy}{{\mymacro{y}}}
\newcommand{\symz}{{\mymacro{z}}}
\newcommand{\stacksym}{{\mymacro{\stacksymbol{\gamma}}}}

\newcommand{\defeq}{\mathrel{\stackrel{\textnormal{\tiny def}}{=}}}

\newcommand{\NTo}[1]{{\mymacro{\left[ #1 \right]}}}

\newcommand{\set}[1]{{\mymacro{\left\{ #1 \right\}}}}

\newcommand{\justification}[1]{%
    \refstepcounter{equation}%
    \tag{\theequation \textcolor{black!50}{, \footnotesize{#1}}}
}

\newcommand{\idx}{{\mymacro{ n}}}
\newcommand{\idxn}{{\mymacro{ n}}}
\newcommand{\idxd}{{\mymacro{ d}}}
\newcommand{\idxi}{{\mymacro{ i}}}
\newcommand{\idxj}{{\mymacro{ j}}}

\newcommand{\nstates}{{\mymacro{ |\states|}}}
\newcommand{\nsymbols}{{\mymacro{ |\alphabet|}}}

\newcommand{\tstep}{{\mymacro{ t}}}

\newcommand{\finaltstep}{{\mymacro{ T}}}

\newcommand{\measure}{\mymacro{\mu}}

\newcommand{\pLM}{\mymacro{\pdens}}
\newcommand{\pLMA}{\mymacro{\pdens}_\automaton}

\newcommand{\cotMacro}[1]{\mymacro{{#1}^{\scaleto{\thought}{7pt}}}}
\newcommand{\pLMcot}{\mymacro{\cotMacro{\pdens}}}

\newcommand{\qLM}{\mymacro{\qdens}}
\newcommand{\pLNSM}{\mymacro{\pdens}}

\newcommand{\bos}{{\mymacro{\textsc{bos}}}}
\newcommand{\eos}{{\mymacro{\textsc{eos}}}}

\newcommand{\embedDim}{{\mymacro{ R}}}
\usepackage{stmaryrd}
\newcommand{\onehot}[1]{{\mymacro{\llbracket#1\rrbracket}}}

\newcommand{\inEmbedding}{{\mymacro{ \vr}}}
\newcommand{\inEmbeddingFun}[2][]{{\mymacro{ \inEmbedding\!\left(#2\right)}}}

\newcommand{\inEmbedSymt}{{\mymacro{ \inEmbeddingFun{\sym_\tstep}}}}

\newcommand{\embedMtx}{{\mymacro{ \mE}}}
\newcommand{\eembedMtx}{{\mymacro{ \emE}}}

\newcommand{\precision}{{\mymacro{\psi}}}
\newcommand{\precisionFun}[1]{{\precision\left(#1\right)}}

\newcommand{\symt}{{\mymacro{ \sym_{\tstep}}}}

\newcommand{\symone}{{\mymacro{ \sym_{1}}}}

\newcommand{\biasVech}{{\mymacro{ \vb}}}

\newcommand{\zero}{{\mymacro{\boldsymbol{0}}}}
\newcommand{\one}{{\mymacro{\boldsymbol{1}}}}

\newcommand{\automaton}{{\mymacro{ \mathcal{A}}}}
\newcommand{\wfsa}{{\mymacro{ \automaton}}}
\newcommand{\wfsaFun}[1]{{\wfsa\left(#1\right)}}
\newcommand{\transducer}{{\mymacro{ \mathcal{T}}}}
\newcommand{\wfst}{{\mymacro{ \transducer}}}

\newcommand{\stateq}{{\mymacro{ q}}}

\newcommand{\states}{{\mymacro{ Q}}}
\newcommand{\actions}{{\mymacro{A}}}

\newcommand{\trans}{{\mymacro{ \delta}}}

\newcommand{\weight}{{\mymacro{ \textnormal{w}}}}

\newcommand{\prefixweight}{{\mymacro{\widetilde\weight}}}
\newcommand{\apath}{{\mymacro{ \boldsymbol \pi}}}
\newcommand{\pathlen}{{\mymacro{ N}}}
\newcommand{\paths}{{\mymacro{ \Pi}}}
\newcommand{\initial}{{\mymacro{ I}}}
\newcommand{\final}{{\mymacro{ F}}}
\newcommand{\initf}{{\mymacro{ \lambda}}}
\newcommand{\finalf}{{\mymacro{ \rho}}}
\newcommand{\initfFun}[1]{{\mymacro{\initf\left(#1\right)}}}
\newcommand{\finalfFun}[1]{{\mymacro{\finalf\left(#1\right)}}}

\newcommand{\qinit}{{\mymacro{ q_{\iota}}}}
\newcommand{\qfinal}{{\mymacro{ q_{\varphi}}}}

\newcommand{\wfsatuple}{{\mymacro{ \left( \alphabet, \states, \initf, \finalf, \trans \right)}}}

\newcommand{\edgenoweight}[3]{#1 \xrightarrow{#2} #3}
\newcommand{\edge}[4]{{\mymacro{#1 \xrightarrow{#2 / #3} #4}}}

\newcommand{\wfsttuple}{{\mymacro{ \left( \alphabet, \oalphabet, \states, \trans, \initf, \finalf \right)}}}
\newcommand{\fsttuple}{{\mymacro{ \left( \states, \alphabet, \oalphabet, \initial, \final, \trans \right)}}}

\newcommand{\yield}{{\mymacro{\textbf{s}}}}

\newcommand{\stateeq}{{\mymacro{ \sim_{\trans}}}}

\newcommand{\tm}{{\mymacro{ \mathcal{M}}}}
\newcommand{\ptmtuple}{\left(\states, \alphabet, \tapealphabet, \trans, \qinit, \qfinal \right)}
\newcommand{\tapealphabet}{\mymacro{ \Gamma}}
\newcommand{\posAt}[1]{\mymacro{c}\left(#1\right)}
\newcommand{\visitedAt}[1]{\mymacro{L}\left(#1\right)}
\newcommand{\lastVisit}[1]{\mymacro{\ell}\left(#1\right)}
\newcommand{\tmdir}{\mymacro{d}}
\newcommand{\tmleft}{\mymacro{ L}}
\newcommand{\tmright}{\mymacro{ R}}

\newcommand{\tmops}{\{\tmleft,\tmright\}}
\newcommand{\blanksym}{\mymacro{\sqcup}}
\newcommand{\tapesym}{\mymacro{\stacksym}}

\newcommand{\qptm}{{\mymacro{\Q\ptm}}}
\newcommand{\qptmtuple}{\mymacro{\left( \states, \alphabet, \tapealphabet, \trans_\tm, \qinit, \qfinal \right)}}

\newcommand{\elmanrnntuple}{{\mymacro{ \left( \alphabet, \sigmoid, \hiddDim, \recMtx, \inMtx, \biasVech, \initstate\right)}}}

\newcommand{\rnn}{{\mymacro{ \mathcal{R}}}}

\newcommand{\dpfsaAcr}{{\mymacro{DPFSA}}}
\newcommand{\recMtx}{{\mymacro{ \mU}}}
\newcommand{\inMtx}{{\mymacro{ \mV}}}
\newcommand{\outMtx}{{\mymacro{ \mE}}}
\newcommand{\eOutMtx}{{\mymacro{ \eembedMtx}}}

\newcommand{\hiddDim}{{\mymacro{ D}}}

\newcommand{\simplexFun}[1]{{\mymacro{ \boldsymbol{\Delta}}^{#1}}}

\newcommand{\Simplexnminus}{{\mymacro{ \boldsymbol{\Delta}^{N-1}}}}

\newcommand{\hiddState}{{\mymacro{ \vh}}}
\newcommand{\hiddStatet}{{\mymacro{ \hiddState_\tstep}}}

\newcommand{\hiddStatetminus}{{\mymacro{ \hiddState_{\tstep - 1}}}}

\newcommand{\vhzero}{{\mymacro{ \vh_0}}}

\newcommand{\hiddStateZero}{{\mymacro{ \vhzero}}}

\newcommand{\initstate}{{\mymacro{\boldsymbol{\eta}}}}

\newcommand{\vhtminus}{{\mymacro{ \vh_{t-1}}}}

\newcommand{\grammar}{{\mymacro{ \mathcal{G}}}}

\newcommand{\NT}[1]{{\mymacro{ \mathrm{#1}}}}

\newcommand{\enc}{{\mymacro{\mathsf{enc}}}}

\newcommand{\negterm}[1]{{\mymacro{ {\raise.17ex\hbox{$\scriptstyle\sim$}} #1}}}

\newcommand{\ifcondition}{\textbf{if }}
\newcommand{\otherwisecondition}{\textbf{otherwise}}

\newcommand{\wpdatuple}{\left(\states, \alphabet, \stackalphabet, \trans, \initialconfig, \finalconfig \right)}

\newcommand{\initialconfig}{\pdaconfig{\NT{S}}{\qinit}}
\newcommand{\finalconfig}{\pdaconfig{\eps}{\qfinal}}

\newcommand{\pushdown}{\mymacro{ \mathcal{P}}}
\newcommand{\pda}{\mymacro{ \pushdown}}

\newcommand{\twoPdaEdgenoweight}[8]{\mymacro{ #1 \xrightarrow[#6 \rightarrow #8]{#2,#3, #5 \rightarrow #7} #4}}

\newcommand{\arun}{{\mymacro{ \apath}}}
\newcommand{\runs}{{\mymacro{ \paths}}}
\newcommand{\stackseq}{{\mymacro{ {\boldsymbol{\gamma}}}}}

\newcommand{\pdaconfig}[2]{{\mymacro{ \left(#2, #1\right)}}}

\newcommand{\stacksymbol}[1]{{\mymacro{ #1 }}}

\newcommand{\atrans}{{\mymacro{ \tau}}}

\newcommand{\ignore}[1]{}
\newcommand{\expandLater}[1]{}

\newcommand{\tfheadnum}{\mymacro{H}}

\newcommand{\qTransf}{\mymacro{Q}}
\newcommand{\kTransf}{\mymacro{K}}
\newcommand{\vTransf}{\mymacro{V}}
\newcommand{\oTransf}{\mymacro{O}}
\newcommand{\fTransf}{\mymacro{F}}

\newcommand{\attn}{\mymacro{\texttt{Att}}}

\newcommand{\tfheadCombine}{\mymacro{\mathcal{H}}}

\newcommand{\staticRepr}{{\mymacro{\mathcal{R}}}}
\newcommand{\tf}{\mymacro{\mathcal{T}}}
\newcommand{\tfFun}[1]{\tf\left(#1\right)}
\newcommand{\transformernetwork}{\mymacro{\mathcal{T}}}
\newcommand{\tfpLM}{\mymacro{\pLM_\transformernetwork}}
\newcommand{\tfencfun}{\mymacro{\enc}}

\newcommand{\tfscorefun}{\mymacro{\func}}

\newcommand{\hardmax}{\mymacro{\mathrm{hardmax}}}
\newcommand{\hardmaxAvg}{\mymacro{\hardmax}}

\newcommand{\tflayer}{\mymacro{\mathrm{\mathcal{L}}}}

\newcommand{\tflayerinputsy}{\mymacro{\vx}}

\newcommand{\tflayeroutputsy}{\mymacro{\vz}}

\newcommand{\tflayeridx}{\mymacro{\ell}}
\newcommand{\tfnumlayer}{\mymacro{L}}

\newcommand{\posEnc}{\mymacro{\Pi}}

\newcommand{\posEncFun}[1]{\mymacro{\posEnc\left(#1\right)}}

\newcommand{\posInEmbedding}{\mymacro{\inEmbedding_{\scaleto{\posEnc}{4pt}}}}
\newcommand{\posInEmbeddingFun}[1]{\mymacro{\posInEmbedding\left(#1\right)}}

\newcommand{\epsalphabet}{\mymacro{\alphabet_\eps}}

\def\1{\mathbf{1}}

\def\eps{{\mymacro{ \varepsilon}}}

\def\sep{{\mymacro{X\xspace}}}

\def\rvz{{{\mymacro{ \mathbf{z}}}}}

\def\va{{{\mymacro{ \mathbf{a}}}}}
\def\vb{{{\mymacro{ \mathbf{b}}}}}

\def\vh{{{\mymacro{ \mathbf{h}}}}}

\def\vk{{{\mymacro{ \mathbf{k}}}}}

\def\vq{{{\mymacro{ \mathbf{q}}}}}
\def\vr{{{\mymacro{ \mathbf{r}}}}}
\def\vs{{{\mymacro{ \mathbf{s}}}}}

\def\vv{{{\mymacro{ \mathbf{v}}}}}

\def\vx{{{\mymacro{ \mathbf{x}}}}}
\def\vy{{{\mymacro{ \mathbf{y}}}}}
\def\vz{{{\mymacro{ \mathbf{z}}}}}

\def\evs{{{\mymacro{ s}}}}

\def\evx{{{\mymacro{ x}}}}

\def\mA{{{\mymacro{ \mathbf{A}}}}}

\def\mE{{{\mymacro{ \mathbf{E}}}}}

\def\mI{{{\mymacro{ \mathbf{I}}}}}

\def\mK{{{\mymacro{ \mathbf{K}}}}}

\def\mO{{{\mymacro{ \mathbf{O}}}}}

\def\mQ{{{\mymacro{ \mathbf{Q}}}}}

\def\mU{{{\mymacro{ \mathbf{U}}}}}
\def\mV{{{\mymacro{ \mathbf{V}}}}}
\def\mW{{{\mymacro{ \mathbf{W}}}}}
\def\mX{{{\mymacro{ \mathbf{X}}}}}

\def\sS{{{\mymacro{ \mathcal{S}}}}}

\def\emE{{\mymacro{ E}}}

\newcommand{\N}{{\mymacro{ \mathbb{N}}}}
\newcommand{\Nzero}{{\mymacro{ \mathbb{N}_{\geq 0}}}}

\newcommand{\projfunc}{{\mymacro{\vfunc}}}
\newcommand{\projfuncEosalphabetminus}{{\mymacro{\vfunc}}}

\newcommand{\ReLU}{{\mymacro{ \mathrm{ReLU}}}}
\newcommand{\sparsemaxfunc}[2]{{\mymacro{ \mathrm{sparsemax}\!\left(#1\right)_{#2}}}} % vector, index}
\newcommand{\ReLUfunc}[1]{{\mymacro{ \ReLU\!\left(#1\right)}}} % vector, index}
\newcommand{\heaviside}{{\mymacro{\mathrm{H}}}}
\newcommand{\heavisideFun}[1]{{\mymacro{ \heaviside\left(#1\right)}}}
\newcommand{\sigmoid}{{\mymacro{ \sigma}}}

\DeclareMathOperator*{\argmax}{{\mymacro{ argmax}}}
\DeclareMathOperator*{\argmin}{{\mymacro{ argmin}}}

\newcommand{\bigO}[1]{{\mymacro{ \mathcal{O}\left(#1\right)}}}

\newcommand{\activation}{\mymacro{\boldsymbol{\alpha}}}
\newcommand{\mlpActivation}{\mymacro{\boldsymbol{\beta}}}

\newcommand{\weakeq}{\mymacro{\ \cong\ }} % \stackrel{w.e.}{=}

\usepackage{bbm}
\newcommand{\ind}[1]{\mathbbm{1} \left\{ #1 \right\}}
\title{On the Representational Capacity of Neural Language Models\\ with Chain-of-Thought Reasoning}

\author{
Franz Nowak\thanks{Equal contribution.}%
~\;~\;~
Anej Svete\footnotemark[1]% 
~\;~\;~
Alexandra Butoi%
~\;~\;~Ryan Cotterell\\
\texttt{\{\href{mailto:franz.nowak@inf.ethz.ch}{franz.nowak}, \href{mailto:anej.svete@inf.ethz.ch}{anej.svete},
\href{mailto:alexandra.butoi@inf.ethz.ch}{alexandra.butoi}, \href{mailto:ryan.cotterell@inf.ethz.ch}{ryan.cotterell}\}@inf.ethz.ch}\\
    {%
\setlength{\fboxsep}{2.5pt}%
\setlength{\fboxrule}{2.5pt}%
\fcolorbox{white}{white}{
    \includegraphics[width=.15\linewidth]{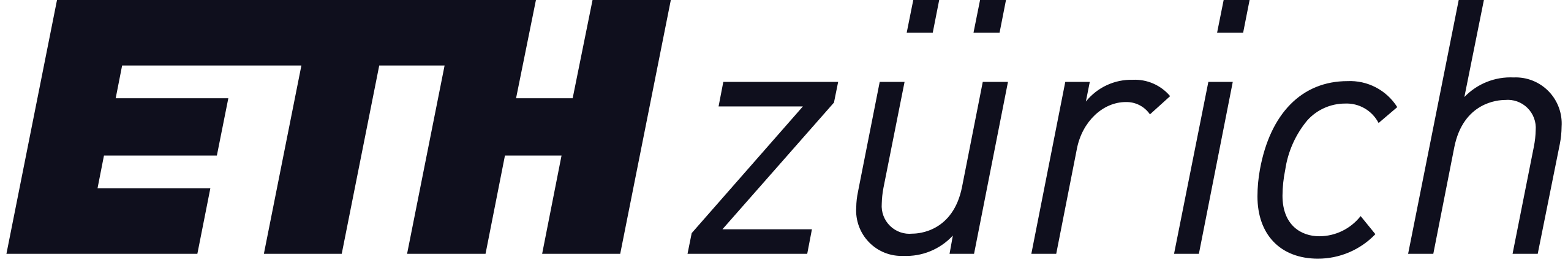}
}
}}

\begin{document}
\maketitle
\begin{abstract}
    The performance of modern language models (LMs) has been improved by \underline{c}hain-\underline{o}f-\underline{t}hought (CoT) reasoning, i.e., the process of generating intermediate results that guide the model towards a final answer.
    A possible explanation for this improvement is that CoT reasoning extends an LM's computational power, as RNNs and transformers with additional scratch space are known to be Turing complete.
    Comparing LMs to Turing machines, however, introduces a category error---Turing machines decide language membership, whereas LMs define \emph{distributions} over strings.
    To bridge this gap, we formalize CoT reasoning in a probabilistic setting.
    We present several results on the representational capacity of recurrent and transformer LMs with CoT reasoning, showing that they can represent the same family of distributions over strings as \emph{probabilistic} Turing machines.\looseness=-1

    \vspace{0.5em}
    {\includegraphics[width=1.36em,height=1.25em]{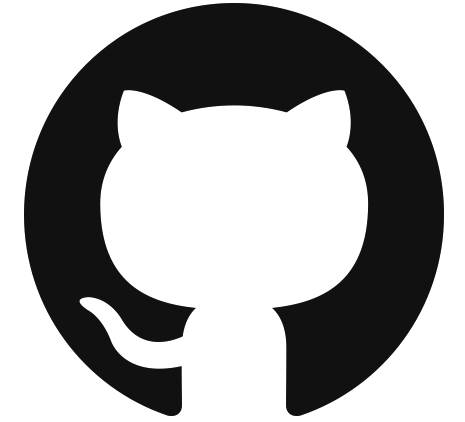}\hspace{2pt}\parbox{\dimexpr\linewidth-2\fboxsep-2\fboxrule}{\url{https://github.com/rycolab/cot-lms}}}
\end{abstract}

\section{Introduction}
\label{introduction}
Motivated by how humans solve complex problems by noting down intermediary results, \citet{wei2023chain} introduced \defn{chain-of-thought} (CoT) reasoning.\footnote{We use the more general term CoT \emph{reasoning} over the original term CoT \emph{prompting} as prompting is just one way to elicit CoT reasoning \citep{wang2024chainofthought}.}
CoT reasoning helps language models (LMs) solve reasoning tasks by allowing them to store intermediary results in a scratch space that is \emph{not} part of the final output.
The empirical success of CoT reasoning has made it an established component of modern neural LMs \citep{nye2021work,wei2022emergent,wei2023chain, suzgun2022challenging,kojima2023large,wies2023subtask}---seemingly overnight.
Such empirical success motivates a thorough understanding of the abilities of CoT-augmented LMs.
While existing theoretical treatments have shed light on some aspects of this framework, we are still far from a concrete theoretical understanding of CoT \citep{feng2023revealing}.\looseness=-1

\begin{figure}
    \centering
    \relsize{-1}
    \begin{tikzpicture}
        \draw[draw=none, fill=ETHYellow!80, rounded corners] (-3.75, -3.4) rectangle (3.75, 1.85);
        \draw[dashed, fill=ETHYellow!50, rounded corners] (-3.55, -3.2) rectangle (3.55, -0.75);
        \draw[dotted, thick, ETHGray!90] (0.1, -0.6) -- (0.1, 1);
        \draw[dotted, thick, ETHGray!90] (0.1, -3.) -- (0.1, -1.4);
        \draw[dotted, thick, ETHGray!90] (0.1, -4.2) -- (0.1, -3.5);
        \node[align=center] at (0, 1.425) {\textcolor{brown!80!black}{\underline{\textbf{Probabilistic unrestricted languages}}}};
        \node[align=center] at (-2.1, 0.7) {Unbounded precision \\ Transformers};
        \node[align=center] at (0.3, 0.5) {$\autorightleftharpoons{\cref{thm:cot-transformer-ptm}}{\cref{thm:ptm-cot-rnn-transformer}}$};
        \node[align=center] at (2.4, 0.2) {\textcolor{black}{Probabilistic} \\ \textcolor{black}{Turing machines}};
        \node[align=center] at (0.3, -0.1) {$\autoleftrightharpoons{\phantom{\cref{thm:ptm-cot-rnn-transformer}}}{\cref{thm:cot-rnn-ptm}}$};
        \node[align=center] at (-2.1, -0.2) {Unbounded precision \\ Elman RNNs};
        \node[align=center] at (0, -1.125) {\textcolor{ETHBrown!80!ETHYellow}{\underline{\textbf{Probabilistic regular languages}}}};
        \node[align=center] at (2.3, -2.15) {Probabilistic \\ Finite State \\ Automata};
        \node[align=center] at (-1.8, -1.8) {Fixed-precision \\ Elman RNNs};
        \node[align=center] at (0.3, -1.9) {$\autorightleftharpoons{\cref{thm:pfsa-rnn}}{\cref{thm:rnn-pfsa}}$};
        \node[align=center] at (-1.8, -2.6) {Fixed-precision \\ Transformers};
        \node[align=center] at (0.3, -2.5) {$\autoleftarrow{\phantom{\cref{thm:pfsa-transformer}}}{\cref{thm:pfsa-transformer}}$};
        \node[align=center] at (-1.7, -3.9) {Neural architecture \\ with CoT reasoning};
        \node[align=center] at (1.8, -3.9) {Corresponding model \\ of computation};

    \end{tikzpicture}
    \caption[figure1]{A schematic overview of our main results, showing which neural language models extended with CoT reasoning (left) correspond to which probabilistic models of computation (right). }
    \label{fig:figure-1}
    \vspace{-15pt}
\end{figure}

CoT reasoning can be interpreted as encouraging LMs to perform additional sequential computation steps while storing the result of that computation.
Naturally, this suggests that CoT reasoning should be formalized in a manner that exploits our understanding of well-known, stateful models of computation.
Theory of computation provides a natural toolbox for this avenue of exploration.
Indeed, the internal configurations of many neural language models have been previously linked to intermediate steps in neural LMs \citep{perez-etal-2021-attention,merrill2024the,feng2023revealing}.

A \defn{language model} is definitionally a distribution over strings from some fixed alphabet.%
\footnote{An alphabet is a finite, non-empty set.}
This definition is somewhat at odds with the way existing work has investigated the representational capacity of CoT-augmented LMs.
Firstly, most results in this area concern unweighted language \emph{recognition}, which is different from the inherently probabilistic task of language modeling.
Secondly, Turing completeness results such as
\citeposs{perez-etal-2021-attention} construction---a seminal result showing the Turing completeness of transformer LMs--- usually require \emph{additional symbols} not present in the original alphabet to correctly simulate a Turing machine.
Because the output alphabets of the transformer LM and the Turing machine it simulates fundamentally differ, it is difficult to discuss their \emph{equivalence}.
This motivates a more fine-grained treatment of the representational capacity of neural LMs in which:
\begin{enumerate*}[label=\textit{(\arabic*)}]
    \item LMs are treated as \emph{probabilistic} models of computation, i.e., they assign weights to strings rather than deciding language membership, and
    \item language model equivalence is defined on the output string level, and the analysis takes into account the additional information required to be encoded in additional outputs to achieve the equivalence.
\end{enumerate*}\looseness=-1

We address the first point by analyzing CoT-augmented LMs using \emph{probabilistic} models of computation, which provide a convenient way of describing families of distributions over strings with standard models of sequential reasoning.
To address the second point, we define a new type of relationship between language models, which we call \defn{regular reducibility}.
Intuitively, an LM is regularly reducible into another one if the strings generated by the latter are sufficiently simple transformations of those generated by the former.
An instance of such a transformation is the deletion of the intermediary computation steps of CoT reasoning.
Formalizing CoT reasoning in this framework, we find that it increases the computational power in both RNN and transformer LMs.
Concretely, we find that CoT-augmented constant-precision RNN LMs are equivalent to probabilistic finite-state automata (\pfsaAcr{}s).
This is in contrast to the constant-precision RNN LMs \emph{without} CoT reasoning, which are equivalent to \emph{deterministic} \pfsaAcr{}s \citep{svete-cotterell-2023-recurrent}.
Additionally, we show how Turing-complete linear-precision RNN LMs can be thought of as performing CoT reasoning.
Finally, we show that both linearly bounded precision RNN LMs and logarithmically bounded precision decoder-only transformer LMs with CoT-reasoning can simulate any probabilistic Turing machine.
Taken together, our results frame CoT reasoning in pure language modeling terms and describe its intuitive and formal connections to probabilistic models of computation.

\section{Preliminaries}
\label{sec:preliminaries}
In our analysis, we assume rational (rather than real) arithmetic for all definitions and computations.\looseness=-1

\subsection{Language Models}\label{language-models}
A \defn{formal language} $\lang$ is a subset of the Kleene closure $\kleene{\alphabet}$ of some alphabet $\alphabet$.
We call an element $\str$ of $\kleene{\alphabet}$ a \defn{string}.
We denote the empty string by $\eps$, and we assume that $\eps \notin \alphabet$.
A \defn{discrete semimeasure} over $\kleene{\alphabet}$ is a function $\measure \colon \kleene{\alphabet} \rightarrow [0,1]$ such that $\sum_{\str \in \kleene{\alphabet}} \measure \left(\str \right) \leq 1$ \cite{bawens2013, icard2020calibratinggm}.
If the semimeasure of all strings sums to one, i.e., $\sum_{\str\in \kleene{\alphabet}}\measure\left(\str\right) = 1$, then $\measure$ is called a \defn{probability measure}.%
\footnote{This definition differs from \citeposs{li2008kolmogorov} who instead define semimeasures over \emph{prefix strings}.}
Finally, for any alphabet $\alphabet$, we define the set $\epsalphabet \defeq \alphabet \cup \set{\eps}$.

\begin{definition}\label{def:lm}
    A \defn{language model} (LM) $\pLM$ is a semimeasure over $\kleene{\alphabet}$.
    If $\pLM$ is a probability measure, it is called a \defn{tight} language model.
\end{definition}
\begin{definition}\label{def:lm-equivalence}
    Two LMs $\pLM$ and $\qLM$ are \defn{weakly equivalent} if $\pLM(\str) = \qLM(\str)$ for all $\str\in\kleene{\alphabet}$.
\end{definition}
Most modern LMs are autoregressive, i.e., they define $\pLM\left(\str\right)$ through conditional distributions of the next symbol given the string produced so far and the probability of ending the string, i.e.,
\begin{equation} \label{eq:autoregressive-lm}
    \pLM\left(\str\right) \defeq \pLM\left(\eos\mid \str\right) \prod_{\tstep = 1}^{|\str|} \pLM\left(\sym_\tstep\mid\str_{<\tstep}\right).
\end{equation}
Here, $\eos$ denotes the special \underline{e}nd-\underline{o}f-\underline{s}tring symbol, which specifies that the generation of a string has halted.
The inclusion of $\eos$ allows (but does not guarantee) that a language model $\pLM$, autoregressively, constitutes a probability measure over $\kleene{\alphabet}$ \citep{du-etal-2023-measure}; a model defined as in \cref{eq:autoregressive-lm} may sum to \emph{less than} 1 in a pathological case.
For any alphabet $\alphabet$, we define the set $\eosalphabet \defeq \alphabet \cup \set{\eos}$.
The conditional probability distributions are usually defined based on \emph{vectorial} representations of $\strlt$ computed by a \defn{language encoder}
$\enc\colon \kleene{\alphabet} \to \R^\hiddDim$ \citep{chan-etal-homotopies}.\looseness=-1
\begin{definition} \label{def:repr-lm}
A \defn{representation-based} LM is any LM that can be written as an autoregressive language model (\cref{eq:autoregressive-lm}) where the conditional distributions over the next symbol $\eossym_\tstep\in\eosalphabet$ are given by\looseness=-1
\begin{equation}\label{eq:repr-lm}
        \pLNSM\left(\eossym_\tstep \mid \strlt\right) \defeq \projfuncEosalphabetminus(\outMtx \; \enc\left(\strlt\right))_{\eossym_\tstep},
\end{equation}
where $\enc\colon \kleene{\alphabet} \to \R^\hiddDim$ is a language encoder, $\outMtx \in \R^{|\eosalphabet| \times \hiddDim}$ is an output matrix, and $\projfuncEosalphabetminus$ is a projection function.\footnote{A common choice for $\projfuncEosalphabetminus$ is the softmax. Since our analyses use rational arithmetic, we instead opt for sparsemax \citep{sparsemax}. However, all of our results can be extended to the use of the more common softmax function through the use of log activations and the extended real numbers $\R\cup\{-\infty, \infty\}$ \citep{svete-etal-2024-lower}.}
\end{definition}
\subsection{Regular Language Models} \label{sec:wfsas}
Probabilistic finite-state automata are a well-understood probabilistic computational model.\looseness=-1
\begin{definition}\label{def:stochastic-wfsa}
    A \defn{probabilistic finite-state automaton} (\pfsaAcr{}) is a tuple $\wfsatuple$ where $\alphabet$ is an alphabet, $\states$ is a finite set of states, $\trans \subseteq \states \times \alphabet \times \Qnonnegative \times \states$ is a finite set of weighted transitions
    where we write transitions  $\left(\stateq, \sym, w, \stateq^\prime\right) \in \trans$ as $\edge{\stateq}{\sym}{w}{\stateq^\prime}$,
    and $\initf, \finalf\colon \states \rightarrow \Qnonnegative$ are functions that assign each state its initial and final weight, respectively.
    Moreover, for all states $\stateq \in \states$, $\trans, \initf$ and $\finalf$ satisfy $\sum_{\stateq \in \states} \initf\left(\stateq\right) = 1$, and $\sum\limits_{\edge{\stateq}{\sym}{w}{\stateq^\prime} \in \trans} w + \finalf\left(\stateq\right) = 1$.
\end{definition}

Next, we will define some basic concepts related to PFSAs.
A \pfsaAcr{} $\automaton = \wfsatuple$ is \defn{deterministic} if $|\set{\stateq \mid \initfFun{\stateq} > 0}| = 1$ and, for every $\stateq \in \states, \sym \in \alphabet$, there is at most one $\stateq^\prime \in \states$ such that $\edge{\stateq}{\sym}{w}{\stateq^\prime} \in \trans$ with $w > 0$.%
\footnote{In this paper, we do \emph{not} distinguish between a transition for a given symbol with weight $w=0$ and the absence of a transition for that symbol.
That is, we assume there always exists a transition $\edge{\stateq}{\sym}{w}{\stateq^\prime} \in \trans$ for any $\stateq, \stateq^\prime \in \states$ and $\sym \in \alphabet$, albeit possibly with $w = 0$.
Such a choice turns out to be useful in our technical exposition.
}
Any state $\stateq$ where $\initfFun{\stateq}>0$ is called an \defn{initial state}, and if $\finalfFun{\stateq} > 0$, it is called a \defn{final state}.
A \defn{path} $\apath$ of length $\pathlen$ is a sequence of subsequent transitions in $\automaton$, denoted as\looseness=-1
\begin{equation}
    \!\edge{\stateq_1}{\sym_1}{w_1}{\edge{\stateq_2}{\sym_2}{w_2}{\stateq_3} \!\cdots\! \edge{\stateq_{\pathlen}}{\sym_{\pathlen}}{w_{\pathlen}}{\stateq_{\pathlen + 1}}}.
\end{equation}
The \defn{yield} of a path is $\yield\left(\apath\right)\defeq \sym_1 \cdots \sym_{\pathlen}$.
The \defn{prefix weight} $\prefixweight$ of a path $\apath$ is the product of the transition and initial weights, whereas the \defn{weight} of a path additionally has the final weight multiplied in.
In symbols, this means

\noindent\begin{minipage}{0.49\linewidth}
    \begin{equation} \label{eq:prefix-path-weight}
        \prefixweight(\apath)\defeq \prod_{\idx = 0}^\pathlen w_\idx,
    \end{equation}
\end{minipage}
\begin{minipage}{0.49\linewidth}
    \begin{equation}
        \weight(\apath)\defeq \prod_{\idx = 0}^{\pathlen+1} w_\idx,
    \end{equation}
\end{minipage}
with $w_0 \defeq \initf(\stateq_1)$ and $w_{\pathlen+1} \defeq \finalf(\stateq_{\pathlen+1})$.
We write $\paths(\automaton)$ for the set of all paths in $\automaton$ and we write $\paths(\automaton, \str)$ for the set of all paths in $\automaton$ with yield $\str$.
The sum of weights of all paths that yield a certain string $\str\in\kleene{\alphabet}$ is called the \defn{stringsum}, which we write as\looseness=-1
\begin{equation} \label{eq:pfsa-summation}
    \automaton \left( \str \right) \defeq \sum_{\apath \in \paths\left( \automaton, \str \right) }  \weight \left( \apath \right).
\end{equation}
The stringsum gives the probability of the string $\str$.
This way, \pfsaAcr{}s induce a particularly well-understood family of LMs.
\begin{definition}
    An LM $\pLM$ is a \defn{regular} LM if there exists a \pfsaAcr{} $\automaton$ whose induced language model $\pLMA$ is weakly equivalent to $\pLM$.\looseness=-1
\end{definition}

\paragraph{PFSAs and non-determinism.}
Although \pfsaAcr{}s share many properties with unweighted (boolean-weighted) finite-state automata, one important difference relates to determinization.
In the unweighted case, the class of deterministic and non-deterministic FSAs are equivalent, i.e., any non-deterministic FSA has an equivalent deterministic FSA that accepts the same language.
This result, however, does not hold for \pfsaAcr{}s: There exist \pfsaAcr{}s that admit no deterministic equivalent \citep{mohri-1997-finite, Buchsbaum1998}, meaning that non-deterministic \pfsaAcr{}s are strictly more expressive than deterministic ones.

\paragraph{Computing string probabilities under non-determinism.}
Autoregressive LMs (cf. \cref{eq:autoregressive-lm}) and \pfsaAcr{}s fall under the larger framework of models that specify probability distributions \emph{implicitly} \citep{icard2020calibratinggm}.\footnote{Other examples of such models include hidden Markov models, which are equivalent to \pfsaAcr{}s \citep{icard2020calibratinggm}.}
However, in autoregressive neural LMs with fixed precision, only one sequence of computational actions can yield a particular string, meaning that they can only model deterministic weighted regular languages \citep{svete-cotterell-2023-recurrent}. 
On the other hand, non-deterministic \pfsaAcr{}s compute string probabilities by \emph{additively}
combining the probabilities of all paths that yield the same string, and hence lie outside the grasp of such neural models.%
\footnote{Note that RNN LMs with linearly bounded precision can simulate non-deterministic \pfsaAcr{}s in real-time by encoding a probability distribution over the \pfsaAcr{}'s current state in the hidden state of the RNN \citep{svete-etal-2024-lower}.}
As we show later, CoT reasoning provides a principled way to overcome this limitation and allow fixed-precision neural models to simulate non-deterministic automata.\looseness=-1

\subsection{Regular Functions} \label{sec:regfn}
We now define a finite-state machine that, in addition to scanning, also \emph{outputs} strings.\looseness=-1
\begin{definition}\label{def:transducer}
    A \defn{finite-state transducer} ($\fst$) is a $6$-tuple $\transducer = \fsttuple$, where $\states$ is a finite set of \defn{states}, $\alphabet$ is an alphabet of \defn{input symbols}, $\outalphabet$ is an alphabet of \defn{output symbols}, $\initial, \final \subseteq \states$ are sets of initial and final states, respectively, and $\trans \subseteq \states \times \epsalphabet \times \outalphabet_\eps \times \states $ is a set of transitions.
\end{definition}
Similar to \pfsaAcr transitions, we write FST transitions of the form $\left(\stateq,\symx,\symy,\stateq^\prime\right)\in\trans$ as $\edgenoweight{\stateq}{\symx\colon\symy}{\stateq^\prime}$.
Finally, we give the following definition which we will use to formalize CoT reasoning.\looseness=-1
\begin{definition}
    A \defn{regular relation} is a relation $\regfn \subseteq \kleene{\alphabet} \times \kleene{\outalphabet}$ that is representable by an \fst.
    If $\regfn$ is a (partial) function, it is called a \defn{regular function}.\looseness=-1
\end{definition}
Like all relations, regular relations  $\regfn \subseteq\kleene{\alphabet} \times \kleene{\outalphabet}$ have a trivial inverse $\inv{\regfn} \subseteq\kleene{\outalphabet} \times \kleene{\alphabet}$.

\subsection{Turing Machines}\label{sec:turing-machines}
We consider the following definition of a probabilistic Turing machine.\looseness=-1
\begin{definition}\label{def:ptm}
    A \defn{probabilistic Turing machine} (PTM) is a two-tape machine with a working tape and an output tape specified by the $6$-tuple $\tm = \ptmtuple$, where $\states$ is a finite set of states,
    $\alphabet$ is an output alphabet, $\tapealphabet$ is a tape alphabet including the blank symbol $\blanksym$,    $\qinit,\qfinal \in \states$ are the initial and final states, and $\trans\subseteq \states \times \tapealphabet \times \epsalphabet \times \tmops \times \Q_{\geq0} \times \states \times \tapealphabet$ is a rationally weighted transition relation. 
    $\tmleft$ and $\tmright$ signify the PTM head moving left ($\tmleft$) or right ($\tmright$) on the tape after a transition.
    We write transitions $\left(\stateq, \tapesym, \sym, \tmdir, w, \stateq^\prime, \tapesym^\prime\right)\in\trans$ as $(\stateq,\tapesym) \xrightarrow{\sym, \tmdir/w} (\stateq^\prime, \tapesym^\prime)$.
    Moreover, we require that for any given $\stateq\in\states, \tapesym\in\tapealphabet$, the weights satisfy $\sum_{\substack{(\stateq,\tapesym) \xrightarrow{\sym, \tmdir/w} (\stateq^\prime, \tapesym^\prime)\in\trans}}w = 1$.\looseness=-1
\end{definition} 
A transition should be interpreted as follows: When $\tm$ in state $\stateq$, reads $\tapesym$ on the working tape, writes $\sym$ on the output tape, writes $\tapesym^\prime$ on the working tape, move the head in direction $\tmdir$ on the working tape.
Each computation step randomly selects a transition according to its weight $w$.\looseness=-1%
\footnote{For more details, see \cref{app:def-tm}.}
    
This definition of Turing machines straightforwardly induces a semimeasure over strings \citep[Remark 2.2]{nowak-etal-2023-representational}.
It is sometimes easier to prove claims about Turing machines through another equivalent machine \cite{hopcroft01}.
The probabilistic two-stack pushdown automaton is one example \citep{nowak-etal-2023-representational}.
See \cref{sec:twopda} for an overview.

\subsection{Recurrent Neural Language Models}\label{sec:rnnlms}
Recurrent neural LMs are LMs whose conditional probabilities are given by an RNN.%
\footnote{Throughout this paper, we will focus on Elman RNNs \citep{Elman1990} as they are the easiest to analyze and a special case of more common networks, e.g., those based on long short-term memory \citep[LSTM;][]{10.1162/neco.1997.9.8.1735} and gated recurrent units \citep[GRUs;][]{cho-etal-2014-properties}.}
\begin{definition} \label{def:elman-rnn}
    An \defn{Elman RNN} $\rnn = \elmanrnntuple$ is an RNN with the following hidden state recurrence:
    \begin{subequations}
        \begin{align}
            \hiddStateZero & = \initstate  \quad\quad\quad\quad\quad\quad\quad\quad\quad\quad\,\,\,{\color{gray}(t=0)},\label{eq:elman-initialization}               \\
            \hiddStatet    & = \activation\left(\recMtx \vhtminus + \inMtx \inEmbedSymt + \biasVech \right) \,\, {\color{gray}(t>0)},\label{eq:elman-update-rule}
        \end{align}
    \end{subequations}
    where $\hiddStatet \in \Q^\hiddDim$ is the state vector%
    \footnote{Throughout this paper all vectors are column vectors.\looseness=-1}
    at time step $\tstep$, $\initstate \in \Q^\hiddDim $ is an initialization parameter, $\symt\in\alphabet$ is the input symbol at time step $\tstep$, $\inEmbedding\colon \alphabet \to \Q^\embedDim$ is a symbol representation function, $\recMtx \in \Q^{\hiddDim \times \hiddDim}$ and $\inMtx \in \Q^{\hiddDim \times \embedDim}$ are parameter matrices, $\biasVech \in \Q^{\hiddDim}$ is a bias vector, and $\activation\colon\Q^\hiddDim\to\Q^\hiddDim$ is an element-wise, non-linear activation function.%
    \footnote{Common examples of $\activation$ include the Heaviside function $\heaviside(x) \defeq \ind{x > 0}$, the sigmoid function $\sigmoid(x) \defeq \frac{1}{1 + \exp(-x)}$, and the $\ReLU\left(x\right) \defeq \max\left(0, x\right)$.}
\end{definition}
Because $\hiddStatet$ hides the symbols consumed by the Elman RNN, we also use the evocative notation $\hiddState(\str)$ to denote the result of the application of \Cref{eq:elman-update-rule} over the string $\str = \sym_1 \cdots \sym_t$.
The notation $\hiddState(\str)$ makes it clear that an RNN LM implicitly defines a language encoder. 
\begin{definition} \label{def:elman-lm}
A representation-based LM is called an \defn{Elman LM} if
its representation function is defined by the hidden state of an Elman RNN $\enc\left(\strlt\right) \defeq \hiddState\left(\strlt\right)$.
\end{definition}
The most common choice for the projection function $\projfuncEosalphabetminus$ is the softmax, whose limitation is that it implies the LM has full support, i.e., an Elman LM with a softmax projection assigns positive probability to all strings in $\kleene{\alphabet}$.
To construct non-full-support Elman LMs, we instead use the \defn{sparsemax} \citep{sparsemax}:
\begin{equation}\label{eq:spmax}
    \sparsemaxfunc{\vx}{} \defeq \argmin_{\rvz\in \Simplexnminus} \norm{\rvz - \vx}^2_2.
\end{equation}
In contrast to the softmax function, sparsemax is the \defn{identity} on $\Simplexnminus$, i.e., we have $\sparsemaxfunc{\vx}{} = \vx$ for $\vx \in \Simplexnminus$.

\subsection{Neural Networks and Numerical Precision} \label{sec:precision}
All implementations of LMs on modern hardware require representations to be fixed-precision floating-point numbers or arbitrary-precision (rational) numbers. 
In the case of arbitrary precision, an important consideration in processing strings $\str \in \kleene{\alphabet}$ is the number of bits required to store the representations and how the number of bits scales with the length of the string, $|\str|$.
This motivates the following definition of precision.
\begin{definition}
    The \defn{precision} $\precisionFun{\str}$ of a representation-based neural LM is the number of bits required to represent the entries of $\enc\left(\str\right)$:
    \begin{equation}
        \precisionFun{\str} \defeq \max_{\idxd\in[\hiddDim]}\min_{\substack{p,q\in\N,\\ \frac{p}{q}=\enc\left(\str\right)_{\idxd}}} \lceil\log_2 p\rceil + \lceil\log_2 q\rceil.
    \end{equation}
    We say that a representation-based LM is of
    \begin{itemize}[itemsep=0pt]
        \item \defn{constant precision} if $\precisionFun{\str} = \bigO{1}$, i.e., if $\precisionFun{\str} \leq C$ for all $\str \in \kleene{\alphabet}$ and some $C \in \R$,
        \item \defn{logarithmically bounded precision} if $\precisionFun{\str} = \bigO{\log|\str|}$, i.e., if there exist $\strlen_0 \in\N$ and $C\in\R$ such that for all $\str \in \kleene{\alphabet}$ with $|\str| \geq \strlen_0$, $\precisionFun{\str} \leq C\log_2{|\str|}$,
        \item \defn{linearly bounded precision} if $\precisionFun{\str} = \bigO{|\str|}$, i.e., there exist $\strlen_0 \in\N$ and $C\in\R$ such that for all $\str \in \kleene{\alphabet}$ with $|\str| \geq \strlen_0$, $\precisionFun{\str} \leq C|\str|$, and
        \item \defn{unbounded precision} if $\precisionFun{\str}$ cannot be bounded by a function of $|\str|$.
    \end{itemize}
\end{definition}
The constructions in the existing literature range from constant to unbounded precision.
The RNNs considered by \citeposs{svete-cotterell-2023-recurrent} encoding of deterministic \pfsaAcr{}s, for example, results in an RNN that is of constant precision.
\citet{weiss-etal-2018-practical}, \citet{merrill-2019-sequential} and \citet{merrill-etal-2020-formal} consider models with logarithmically bounded precision.
In contrast, articles treating the Turing completeness of neural networks \citep{siegelmann-sontag-1992,nowak-etal-2023-representational} require unbounded precision to be able to represent the fact that a Turing machine may fail to halt.
Naturally, our constructions of Turing complete LMs will also require unbounded precision.

\subsection{Transformer Language Models} \label{sec:transformers}

Transformer LMs compute the conditional distributions $\pLNSM\left(\eossym_\tstep \mid \strlt\right)$ by means of self-attention.
Because transformer LMs necessitate the precise introduction of multiple sub-components, we postpone the full introduction to \cref{app:transformers} and only review the basics here.
We borrow much of the notation and formalization from \citet{svete-etal-2024-transformers}.

A transformer is a composition of multiple transformer \defn{layers}, each of which implements the \defn{attention mechanism}.
The attention mechanism works as follows.
It takes a \defn{\underline{q}uery} vector $\vq \in \R^{\hiddDim}$ and two matrices: The matrix $\mK \in \R^{N \times \hiddDim}$ of \defn{\underline{k}eys} and the matrix $\mV\in \R^{N \times \hiddDim}$ of \defn{\underline{v}alues} and computes a weighted average of the value vectors based on the compatibilities of the key vectors to the query vector, as scored by a scoring function $\tfscorefun$.

Attention weights are computed by normalizing the scores $\tfscorefun\left(\vq, \vk_1\right), \dots ,\tfscorefun\left(\vq, \vk_\tstep\right)$.
The choice of the normalization function has concrete implications on representational capacity \citep{hao-etal-2022-formal,svete-etal-2024-transformers}.
We focus on the \defn{hard attention} projection function.
\begin{definition}
    \defn{Hard attention} is computed with the $\hardmaxAvg$ projection function
    \begin{equation}
        \hardmaxAvg\left(\vx\right)_\idxd \defeq \begin{cases}
            \frac{1}{m} & \ifcondition \idxd\in \argmax\left(\vx\right) \\
            0           & \otherwisecondition,
        \end{cases}
    \end{equation}
    for $\idxd \in \NTo{\hiddDim}$, where $\vx \in \R^\hiddDim$ and $m \defeq |\argmax\left(\vx\right)|$ is the cardinality of the argmax set.
\end{definition}

A transformer layer uses the attention mechanism followed by
a position-wise MLP\footnote{For more details, see \cref{app:transformers}} to compute augmented representations $\vz_\tstep$ of the input representations $\mX_\tstep = \left(\vx_1^\top; \cdots; \vx^\top_\tstep\right) \in \R^{\tstep\times\hiddDim}$.
The query $\vq_\tstep$, the keys $\mK_\tstep$, and values $\mV_\tstep$ are all transformations of the input representations $\mX_\tstep$.
Initially, $\mX_\tstep$ are computed by some static representation function of the symbols and their positions.
Multiple transformer layers are stacked into a transformer, which computes the (deep) contextual representations of all symbols in the string.
The contextual representation of the final symbol in the string is then used to define the representation function $\enc$ of a transformer LM.

\section{CoT Reasoning and Weak Equivalence} \label{sec:cot-equivalence}
In this section, we argue that CoT reasoning can be seen as a way of comparing the input language of a formal automaton with the output language of an autoregressive LM.
We then introduce a formalization of CoT reasoning that allows us to characterize the representational capabilities of CoT-augmented LMs in terms of their weak equivalence to well-studied weighted automata.

\subsection{Equivalence and Augmented Alphabets}
Autoregressive LMs generate strings by outputting individual symbols $\eossym \in \eosalphabet$ one at a time until $\eos$ is generated.
The resulting string $\str$ is the output of the LM and the outputs generated in this manner implicitly define the distribution of the LM.
Models studied in existing work on the representational capacity of neural LMs, however, often do not (only) emit symbols from $\alphabet$.
Rather, they also emit symbols from some \emph{larger} alphabet $\outalphabet$ that includes symbols that encode additional information required to simulate a given formal model of computational.
For example, the symbols generated by \citeposs{perez-etal-2021-attention} transformer contain both the output symbols as well as (the changes to) the configuration of the Turing machine.\looseness=-1

It is not difficult to see how generating strings from an augmented alphabet can be seen as a form of CoT reasoning.
The outputs \emph{not} intended to be a part of the final output can simply be seen as the result of the additional computational steps performed by the CoT-augmented LM.
These outputs are later removed in post-processing and only the string with symbols from our target alphabet $\alphabet$ remains.
In this sense, \citeposs{perez-etal-2021-attention} construction works with a form of CoT reasoning without explicitly mentioning it.
This connection was made explicit in concurrent work by \citet{merrill2024the}.\looseness=-1

Outputting additional information is not regarded as an issue when the task is to simulate (unweighted) Turing machine runs on a given \emph{input string} because what matters, in that case, is only whether the model accepts or rejects the input \citep[\textit{inter alia}]{siegelmann-sontag-1992,perez-etal-2021-attention}.
However, when considering the LM induced by a PTM, we do care about the alphabet the distribution over strings is over.
Thus, given an LM over $\kleene{\alphabet}$, it follows that neural LMs outputting additional information cannot define the same distribution, i.e., they cannot be weakly equivalent.
Nevertheless, the ability to output additional information while generating a string seems natural and useful; it is the heart of CoT reasoning.
And, indeed, models that can only output symbols that are part of the final string are restricted to doing real-time computation \citep{weiss-etal-2018-practical,nowak-etal-2023-representational,svete-etal-2024-lower}.\looseness=-1

\subsection{Regular Reducibility} \label{sec:rr}
We now formalize the notion of CoT reasoning through the following definition which allows a model to output additional information \emph{not} considered part of the final output:
\begin{definition} \label{def:reg-reduce}
    An LM $\pLM$ over $\outalphabet$ is \defn{regularly reducible} to a LM $\qLM$ over $\alphabet$ if there exists a regular function
    $\regfn\colon \kleene{\outalphabet}\to\kleene{\alphabet}$ such that $\qLM \circ \regfn$ is weakly equivalent to $\pLM$.\looseness=-1
\end{definition}
We can use the function $\regfn$ to map the strings sampled from an LM, ones additionally encoding the intermediary steps of computation, into the final output, i.e., strings $\str \in \kleene{\alphabet}$.
\begin{definition}
We say an alphabet $\outalphabet$ satisfies the \defn{$\alphabet$-augmentation condition} for an alphabet $\alphabet$ if $\outalphabet \subseteq \eosalphabet_\eps \times \tapealphabet_\eps \setminus \set{\left(\eps, \eps\right)}$ for some alphabet $\tapealphabet$.
\end{definition}
This means if we care about strings from $\kleene{\alphabet}$ and we have an LM over the closure of a $\alphabet$-augmented alphabet $\outalphabet$, then each symbol the LM outputs encodes is either an output symbol from $\alphabet$, an intermediate computation symbol from $\tapealphabet$, or both.\footnote{While in practice the CoT and the final output come from the same alphabet, we use separate ones for ease of notation and analysis. This is without loss of generality; see \cref{app:alphabets}.}
The computation symbols can then easily be removed by applying the per-symbol projection function $\regfn\left(\left(\sym, \stacksym\right)\right) \defeq \sym$, lifted to a non-length-increasing homomorphism over strings $\regfn\colon \kleene{\outalphabet}\to\kleene{\alphabet}$ to the LM's output element-wise.%
\footnote{A non-length-increasing homomorphism is a function $h\colon\kleene{\outalphabet}\to\kleene{\alphabet}$ where $h(\eps)=\eps$, $h(\symx)=\symy$ for $\symx\in\outalphabet$ and $\symy\in\epsalphabet$, and that further satisfies $h(\symx_1\symx_2\cdots\symx_\strlen) = h(\symx_1)h(\symx_2)\cdots h(\symx_\strlen)$.}
This function can be implemented by a simple single-state transducer (making it a regular function):\looseness=-1
\begin{center}
    \begin{tikzpicture}[node distance = 15mm]
        \footnotesize
        \node[state, initial, accepting] (q0) [] { $0$ };
        \draw[-{Latex[length=2mm]}]
        (q0) edge[loop right] node{ $\outalphabet \ni \left(\sym, \stacksym\right) \colon \sym \in \alphabet$ } (q0) ;
    \end{tikzpicture}
\end{center}

\subsection{Properties of Regular Reducibility} \label{sec:rr-properties}
To motivate our formalization, in this section, we show several properties of regular reducibility that will allow us to reason about the representational capacity of CoT-augmented LMs later in \cref{sec:cot-lms}.
Particularly, as we will see, there exists a strong connection between regular reducibility and the addition of \emph{non-deterministic} choices in the model.
To exemplify this concretely, the next theorem shows that regular reducibility allows us to model non-deterministic \pfsaAcr{}s with deterministic ones, which, as discussed in \cref{sec:wfsas}, cannot be done with just deterministic \pfsaAcr{}s in general.
\begin{restatable}{reTheorem}{rrDPFSA} \label{thm:rr-pfsa}
    Let $\wfsa = \left(\alphabet, \states, \trans, \initf, \finalf\right)$ be a \pfsaAcr.
    Then, there exists a \emph{deterministic} \pfsaAcr $\wfsa^\prime$ over the alphabet $\alphabet \times \states$ and with the state space $\alphabet \times \states$ that is regularly reducible to $\wfsa$.
\end{restatable}
\begin{proof}
    See \cref{sec:rr-proofs}.
\end{proof}

A similar connection exists for \ppdaAcr{}s (see \cref{sec:rr-proofs}).
\cref{thm:rr-pfsa} captures a crucial property of regular reducibility: It allows us to \emph{simulate} non-determinism with a deterministic device.
As exemplified in the cases of regular LMs, this can concretely increase the representational capacity of a model class, as in the case of \cref{thm:rr-pfsa}.
On the other hand, regular reducibility never \emph{decreases} the representational capacity of a model class, since $\regfn$ can always be set to the identity function.
Due to the close connection between neural LMs and determinism \citep{svete-cotterell-2023-recurrent}, one could hope that a similar increase in capabilities could be achieved for neural LMs augmented with CoT reasoning as well.
In \cref{sec:cot-equivalence}, we show that this is indeed the case.\looseness=-1

\subsection{Regular Reducibility and CoT Reasoning} \label{sec:cot-lms}

After introducing the notion of regular reducibility and presenting some of its core properties, we now use it to define CoT-augmented LMs.\looseness=-1
\begin{definition}\label{def:cot-lm}
Given alphabets $\alphabet$ and $\outalphabet$ where $\outalphabet$ satisfies the $\alphabet$-augmentation condition, 
    an LM $\pLMcot$ over $\kleene{\alphabet}$ is called a \defn{CoT-augmented LM} induced by LM $\pLM$ over $\kleene{\outalphabet}$ if $\pLM$ is regularly reducible from $\pLMcot$.
    That is, there exists a regular function $\regfn\colon\kleene{\outalphabet}\to\kleene{\alphabet}$ such that $\pLMcot \defeq \pLM \circ \inv{\regfn}.$
\end{definition}
\cref{def:cot-lm} defines a CoT-augmented LM $\pLMcot$ through the LM $\pLM$ that generates strings from $\kleene{\outalphabet}$ and then applies the regular function $\regfn$ to the generated strings to obtain strings from $\kleene{\alphabet}$.
This allows $\pLM$ to output additional information while still defining a probability distribution over strings from $\kleene{\alphabet}$ (see \Cref{fig:example}).
The weight of a string $\str\in\kleene{\alphabet}$ under $\pLMcot$ is then computed as the sum over the weights of all strings in the preimage of $\regfn$, analogously to \cref{eq:pfsa-summation}:
\begin{equation}
     \pLMcot(\str) = \sum_{\strx\in \inv{\regfn}\left(\str\right)} \pLM(\strx). \label{eq:cot-summation}
\end{equation}
\begin{figure}
    \centering
    \begin{tikzpicture}[node distance = 5mm]

        \node at (-2, 0) {\footnotesize An LM $\pLM$ generates strings from $\kleene{\outalphabet} \subseteq \kleene{\left(\alphabet \times \states\right)}$:};

        \node at (-2, -0.65) (lmcot) {$\strx_1 = \textcolor{ETHGreen}{\stateq_0}, \textcolor{ETHPurple}{\sym_1}, \textcolor{ETHGreen}{\stateq_1}, \textcolor{ETHGreen}{\stateq_2}, \ldots, \textcolor{ETHGreen}{\stateq_{N}}, \textcolor{ETHPurple}{\sym_{\strlen - 1}}, \textcolor{ETHPurple}{\sym_\strlen} \sim \pLM$};
        \node at (-2, -1.3) (lmcot) {$\strx_2 = \textcolor{ETHGreen}{\stateq^\prime_0}, \textcolor{ETHGreen}{\stateq^\prime_1}, \textcolor{ETHPurple}{\sym_1}, \textcolor{ETHGreen}{\stateq_2}, \ldots, \textcolor{ETHPurple}{\sym_\strlen}, \textcolor{ETHGreen}{\stateq^\prime_{M - 1}}, \textcolor{ETHGreen}{\stateq^\prime_{M}} \sim \pLM$};

        \node at (-2, -2) {\footnotesize The function $\regfn$ removes the CoT steps:};

        \node at (-2, -2.5) (lm) {$\regfn\left(\strx_1\right) = \regfn\left(\strx_2\right) = \textcolor{ETHPurple}{\sym}_1, \textcolor{ETHPurple}{\sym}_2, \ldots, \textcolor{ETHPurple}{\sym}_{\strlen - 1}, \textcolor{ETHPurple}{\sym}_\strlen \sim \pLMcot$};
    \end{tikzpicture}
    \caption{
        An LM $\pLM$ can generate strings from a state-augmented alphabet $\outalphabet$.
        The intermediate outputs $\stateq_i$ allow it to condition on the previous states of the computation.
        By \emph{post-hoc} removing the intermediate outputs with $\regfn$, we obtain a CoT-augmented LM $\pLMcot$ over $\kleene{\alphabet}$ that generates more human-readable outputs.
    } \label{fig:example}
\end{figure}
This is the approach we take in \cref{sec:tm-turing-complete}.

\section{CoT and Representational Capacity} \label{sec:main-results}
We now connect the notions introduced in \cref{sec:cot-equivalence} to the representational capacity of neural LMs.
We begin in \cref{sec:nondet-minsky} by showing how CoT reasoning endows autoregressive LMs with \emph{non-determinism} by using scratch space to keep track of the current branch of execution.
We then use similar principles to emulate probabilistic Turing machines with Elman RNNs (\cref{sec:rnns-cot}) and transformers (\cref{sec:tm-turing-complete}).

\paragraph{Neural LMs and formal models of computation.}
Probabilistic models of computation are often presented as autoregressive \emph{generators} of strings that implicitly define probability (semi)measures \citep{icard2020calibratinggm}.
This interpretation is particularly important when addressing non-determinism, because it allows for multiple executions (for example, in the case of \pfsaAcr{}s, multiple different \emph{paths}) generating the same string.
Here, we treat neural LMs as autoregressive generators of strings and show that, by defining identical next-symbol distributions, they implicitly define the same semimeasures over strings as classical probabilistic models of computation.

\subsection{Neural LMs and Regular LMs} \label{sec:nondet-minsky}

\begin{figure}
    \centering

    \begin{tikzpicture}[
        tape node/.style={draw=ETHBlue!80,minimum size=0.85cm,fill=ETHBlue!20},
        doubletape node/.style={draw=ETHGreen!80,minimum size=0.85cm, minimum height=1cm,fill=ETHGreen!20},
        comb arrow/.style={-{Latex[length=2mm,width=1.25mm]},ETHRed!70},
        attn arrow/.style={-{Latex[length=2mm,width=1.25mm]},ETHGreen!100},
        infl arrow/.style={-{Latex[length=1.25mm,width=0.75mm]},ETHGreen!50},
        ]
        \foreach \i/\y in {0/$\bos$,1/$\sym_1$,2/$\cdots$,3/$\cdots$,4/$\cdots$,5/$\sym_{\tstep-1}$,6/$?$,7/$\cdots$} {
                \ifnum \i=7
                    \node[tape node,fill=ETHBlue!40] (tape-\i) at (0.85*\i,0) {\footnotesize \y};
                \else
                    \node[tape node,fill=ETHBlue!20] (tape-\i) at (0.85*\i,0) {\footnotesize \y};
                    \ifnum \i>7
                        \node[tape node,fill=ETHBlue!10] (tape-\i) at (0.85*\i,0) {\footnotesize \y};
                    \fi
                \fi
            }
        % Draw tape
        \foreach \i/\y in {0/$\stateq_0$,1/$\stateq_1$,5/$\stateq_{\tstep-1}$,6/$\stateq_\tstep$} {

                \node[draw=none] (state-\i) at (0.85*\i,1) {\footnotesize \y};
                \foreach \j in {0,1,2,3,4,5} {
                        \ifnum \i>\j
                            \draw[infl arrow] (tape-\j.north) to (state-\i.south);
                        \fi
                        \ifnum \i=\j
                            \draw[infl arrow] (tape-\j.north) to (state-\i.south);
                        \fi
                    }
            }

        % Draw tape
        \foreach \i/\y in {0/$\bos$,1/$\stateq_0$,2/$\substack{\sym_1 \\ \stateq_1}$,3/$\cdots$,4/$\substack{\sym_{\tstep - 2} \\ \stateq_{\tstep - 2}}$,5/$\substack{\sym_{\tstep - 1} \\ \mathhl{\stateq_{\tstep - 1}}}$,6/$?$,7/$\cdots$} {
                \ifnum \i<5
                    \node[doubletape node,fill=ETHGreen!10] (tape-\i) at (0.85*\i,-2.5) {\footnotesize \y};
                \else
                    \ifnum \i=5
                        \node[doubletape node,fill=ETHGreen!15] (tape-\i) at (0.85*\i,-2.5) {\footnotesize \y};
                    \else
                        \ifnum \i<6
                            \node[doubletape node,fill=ETHGreen!25] (tape-\i) at (0.85*\i,-2.5) {\footnotesize \y};
                        \else
                            \ifnum \i=6
                                \node[doubletape node,fill=ETHGreen!35] (tape-\i) at (0.85*\i,-2.5) {\footnotesize \y};
                            \else
                                \node[doubletape node,fill=ETHGreen!5] (tape-\i) at (0.85*\i,-2.5) {\footnotesize \y};
                            \fi
                        \fi
                    \fi
                \fi
            }
        \draw[attn arrow] (tape-6.north) to[out=120,in=60] (tape-5.north);

        \node[draw=none] (prob) [above = 6mm of tape-6] {$\scriptstyle \pLM\left(\stateq_\tstep, \eossym_\tstep \mid \textcolor{ETHRed}{\stateq_{\tstep - 1}}\right)$};
        \draw[comb arrow] (tape-6.north) to[out=60,in=270] (prob.south);

    \end{tikzpicture}
    \caption{If $\strlet$ uniquely determines a subpath in the \pfsaAcr, the current state $\stateq_\tstep$ can be uniquely determined \emph{without} knowing the previous state (top).
    If, however, the substring $\strlet$ can lead to multiple states, the relevant next-symbol distribution cannot be determined.
    Storing the \emph{sampled} states as part of the output fixes this by keeping track of the sampled \pfsaAcr subpath (bottom).}
    \label{fig:rnn-minsky}
\end{figure}

We first discuss the connection between CoT-augmented neural LMs and \pfsaAcr{}s.

\subsubsection{Recurrent Neural LMs}

\citeposs{Minsky1954} construction provided one of the first connections between a neural network and a formal computational model.
It showed that RNNs can emulate (deterministic) FSAs, where the RNN accepts string by activating a particular neuron after reading the string.
The relationship in the probabilistic case was explored by \citet{svete-cotterell-2023-recurrent}, who show the equivalence of pure Elman RNNs (without CoT reasoning) and \emph{deterministic} \pfsaAcr{}s.\footnote{The relationship between RNN LMs and \emph{non}-deterministic \pfsaAcr{}s was explored by \citet{svete-etal-2024-lower}, albeit with linearly bounded-precision RNNs.}
This illustrates an important distinction between the deterministic and non-deterministic frameworks and thus a discrepancy between general \pfsaAcr{}s and RNN LMs.
Intuitively, the discrepancy comes from the fact that there is no non-determinism in the recurrence of an RNN, which is thus unable to capture the possibly non-deterministic decisions of the \pfsaAcr.
CoT reasoning endows an RNN with exactly this non-determinism because it allows the RNN to sample and refer back to \emph{trajectories} rather than only symbols by annotating each of its outputs with the current state of the \pfsaAcr.\footnote{Recall that non-determinism means that multiple possible transitions to different states can yield the same symbol (\cref{sec:wfsas}).}
Upon reading the previously randomly generated state of the automaton captured in the output, it can follow the randomly sampled generating trajectories.

The following two theorems show that RNN LMs with fixed precision and CoT reasoning are weakly equivalent to general PFSA.
This requires showing the correspondence in both directions, i.e., that \begin{enumerate*}[label=\textit{(\arabic*)}]
    \item the distribution induced by any \pfsaAcr can be generated by a constant-precision RNN LM with CoT, and
    \item any CoT-augmented constant-precision RNN LM can be emulated by a \pfsaAcr.
\end{enumerate*}
\begin{restatable}{reTheorem}{minskyRNN} \label{thm:rnn-pfsa}
    For any regular LM, there exists a weakly equivalent CoT-augmented constant-precision Elman RNN LM.
\end{restatable}
\begin{proof}[Proof intuition.]
    We show how, using the RNN's recurrence and output sampling step, we can implement the transition function of any \pfsaAcr.
    The RNN starts by sampling an output symbol containing an initial state $\stateq_0$ according to the initial distribution $\initf$ without emitting a language symbol.
    This output gets fed back into the RNN at the next time step, allowing it to \emph{read} the sampled state, and the next symbol--state pair is sampled according to the conditional distribution defined by the output state, as illustrated by \cref{fig:rnn-minsky}.
    Finally, the states in the generated string are removed by a regular function, leaving only the language output.
    See \cref{app:proofs} for the detailed proof.\looseness=-1
\end{proof}

\begin{restatable}{reTheorem}{minskyRNNBack} \label{thm:pfsa-rnn}
    For any constant-precision CoT-augmented Elman RNN LM, there exists a weakly equivalent \pfsaAcr.
\end{restatable}
\vspace{-6pt}
\begin{proof}
    See \cref{app:proofs}.
\end{proof}
\cref{thm:rnn-pfsa,thm:pfsa-rnn} establish the equivalence between CoT-augmented \emph{constant-precision} Elman RNN LMs and general regular LMs.
This illuminates the added representational capacity awarded by CoT reasoning: Storing the current FSA state in the output string and removing it later allows the model to handle non-determinism which is not possible without the additional information.

\subsubsection{Transformer LMs}
We now show an analogous claim to \cref{thm:rnn-pfsa} for Transformer LMs with CoT.
This is simply a stepping stone towards full probabilistic Turing completeness---a similar construction will then lead us to the full proof that transformer LMs with unbounded precision can simulate PTMs.
\begin{restatable}{reTheorem}{transformerMinsky} \label{thm:pfsa-transformer}
    For any regular LM, there exists a weakly equivalent CoT-augmented constant-precision transformer LM.
\end{restatable}
\begin{proof}[Proof sketch]
    We use the construction from \citet{svete-etal-2024-transformers} which shows how to encode $n$-gram LMs in a transformer.
    Here, the alphabet of the transformer is augmented with the states of the \pfsaAcr, i.e., the symbol at each position $\tstep$ contains not only an output symbol but also the current state of the \pfsaAcr at time $\tstep$.
    Thereby each input symbol contains all the information required to compute the next-symbol probabilities (it is effectively a unigram LM).
    Because the next symbol probabilities in a \pfsaAcr only depend on the current state, the transformer does not need to use attention at all and can rely solely on the output of its residual connections.
    See \cref{app:proofs} for the detailed proof.
\end{proof}

\subsection{Neural LMs and PTMs} \label{sec:lms-cot}
We now extend the results from the previous section from representing the simple regular LM to expressing all enumerable semimeasures over strings by emulating probabilistic Turing machines.
For this, in the RNN case, we require unbounded precision and $\ReLU$ activation functions rather than fixed precision and Heaviside activations.
\subsubsection{Recurrent Neural LMs} \label{sec:rnns-cot}
First, note the following definition:
\begin{definition}
    An autoregressive LM is called a \defn{real-time LM} if it never outputs empty symbols ($\eps$).
\end{definition}
\citet{nowak-etal-2023-representational} show that RNN LMs with rationally valued activations that are not restricted to operating in real time are Turing complete and weakly equivalent to a subset of rationally weighted PTMs.
Intuitively, not operating in real time gives an LM additional computation time while storing the results of its computations in the hidden states.
This is a form of CoT reasoning since we could equivalently let the LM output additional symbols that do not count toward the final output.
\footnote{On the other hand, some CoT-augmented LMs do operate in real time but over an extended alphabet; see \cref{sec:nondet-minsky}.}
The ability to \emph{erase} certain symbols, as done by our transducer $\regfn$, helps make the setup of \citet{nowak-etal-2023-representational} realistic.
Importantly, beyond giving the LM more computation time, CoT also allows the LM to model non-determinism in PTMs.

\begin{restatable}{reTheorem}{rnnTuring}\label{thm:cot-rnn-ptm}
    For every LM induced by a non-deterministic probabilistic Turing machine, there exists a weakly equivalent CoT-augmented RNN LM with unbounded precision. 
\end{restatable}
\begin{proof}[Proof intuition]
The proof follows \citeposs{nowak-etal-2023-representational} probabilistic version of the proof by \citet{siegelmann-sontag-1992}, but extended to account for CoT reasoning.
For the ability to simulate \emph{all} PTMs rather than just a subset, we augment the output alphabet of the RNN with enough information about the current PTM configuration so that subsequent steps can uniquely identify the previous action taken. 
Concretely, the output alphabet $\outalphabet$ contains information about the current state, the symbol written on the working tape of the PTM, and the head action performed.
This additional information is then removed by our regular function at the end, yielding only the output of the simulated PTM generated according to its probability.
    A detailed proof is presented in \cref{app:proofs}.\looseness=-1
\end{proof}

\subsubsection{Transformer LMs} \label{sec:tm-turing-complete}
The representational capacity of transformer LMs has received a lot of attention in the last few years.
\citet{perez-etal-2021-attention} established that the encoder--decoder variant of the architecture is Turing complete.
Since then, concurrent work has shown non-probabilistic Turing completeness of LM-oriented decoder-only variants \citep{merrill2024the,feng2023revealing}.
We extend the work to the \emph{probabilistic} case.

\begin{restatable}{reTheorem}{transformerTuring}\label{thm:cot-transformer-ptm}
    For any PTM-induced LM, there exists a weakly equivalent unbounded-precision CoT-augmented Transformer LM.
\end{restatable}
\begin{proof}[Proof intuition.]
    The proof follows \citet{perez-etal-2021-attention} but is adapted to the probabilistic case.
    Again, the main idea is to augment the output alphabet with enough information about the current PTM configuration to reconstruct the probabilities of possible actions at each time step.
    This information and appropriate positional encodings are enough to recover the PTM's current configuration and thus the next-action distribution, allowing us to construct a weakly equivalent transformer LM.
    A regular function then removes the additional information from the string.
    See \cref{app:proofs} for details.
\end{proof}

\subsubsection{Weak Equivalence}
\cref{thm:cot-rnn-ptm,thm:cot-transformer-ptm} show that transformer and RNN LMs with CoT reasoning are at least as powerful as PTMs.
Weak equivalence requires us to also prove the reverse of these two theorems, analogous to \cref{thm:pfsa-rnn} for constant-precision Elman RNN LMs.
\vspace{-12pt}
\begin{restatable}{reTheorem}{rnnTfWeakEquivalence}\label{thm:ptm-cot-rnn-transformer}
    For any rationally valued RNN LM and transformer LM with CoT reasoning, there exists a weakly equivalent PTM.
\end{restatable}
\begin{proof}[Proof]
    RNN LMs define enumerable semimeasures \citep[App. G]{nowak-etal-2023-representational}, which can always be expressed by a PTM \citep[Thm. 3]{icard2020calibratinggm}.
    Following the same reasoning, transformer LMs define enumerable semimeasures as well; the probability of a string $\str$ is defined as the sum probabilities of all runs of the transformer which result in the output $\str$, of which there are countably many.
    Finally, there is only a countable number of ways a regular function can transform any string, so RNN LMs and Transformer LMs with CoT still define enumerable semimeasures.\looseness=-1
\end{proof}

\section{Conclusion}

Many modern language models have been shown to perform better with CoT reasoning.
Despite its empirical success, CoT reasoning has yet to be well understood formally.
Recent theoretical work in this area has analyzed the Turing completeness of CoT-augmented LMs, failing to account for the inherently probabilistic nature of LMs.
We fix this mismatch in the current literature by introducing a novel formalization of CoT reasoning.\looseness=-1

\section*{Limitations}
Our constructions simulating \pfsaAcr{}s rely on fixed-precision arithmetic in line with the finite-memory nature of such automata.
In contrast, the Turing-complete models all rely on \emph{unbounded} precision with respect to the length of the string---either to be able to encode arbitrarily large stacks in the hidden state in the case of RNN LMs or to be able to encode positional information in the case of transformer LMs.
This inevitably results from the unbounded number of computational steps per emitted symbol by a Turing machine and is distinctly different from the scaling with respect to the number of \emph{computational steps}.
In the transformer case, the precision scales logarithmically with the number of steps (since we use the same positional encodings as \citet{perez-etal-2021-attention}), which is standard for theoretical investigations of transformer models \citep{yao-etal-2021-self,merrill-etal-2022-saturated,merrill-sabharwal-2023-parallelism}.
In the case of RNNs, the precision scales linearly with the computation sequence length, due to the encoding of a stack as a rational number, where each entry on the stack occupies a single digit.
While in line with previous theoretical work, these assumptions are unrealistic in practice since neural LMs normally use fixed precision floating point arithmetic.

We do not use layer normalization with transformers---this is done for simplicity---but note that layer normalization has been found to increase the representational capacity of transformers in some cases \citep{chiang-cholak-2022-overcoming,merrill2024the}.

Moreover, we only prove theoretically the equivalence between CoT-augmented LMs and formal models of computation.
That is, we do not give a training algorithm to elicit the emergence of specific Turing complete or non-deterministic regular automata in neural LMs, and we make no claims about the training efficiency or even the feasibility of training neural LMs for this purpose.

\section*{Ethics Statement}
Since this work deals with purely theoretical properties of neural language models, we do not foresee any ethical issues arising from this work.

\bibliography{anthology, custom}
\bibliographystyle{acl_natbib}

\appendix
\newpage
\onecolumn

\section{Discussion} \label{sec:discussion}
At a high level, our claims characterize neural architectures endowed with CoT reasoning in terms of well-understood \emph{probabilistic} models of computation, giving the framework of CoT reasoning a novel probabilistic perspective.
This allows the reconciliation of theoretical results about neural LMs with the general probabilistic language modeling framework.
In doing so, we put existing results about the representational capacity of LMs into perspective and show how they can be thought of as performing ``CoT reasoning in disguise.''
Perhaps surprisingly, the inclusion of CoT reasoning steps results in a natural inclusion of non-determinism in otherwise deterministic neural LMs, as exemplified by \cref{thm:rnn-pfsa}.
Outputting the states as part of the generated string---whose generation is inherently non-deterministic---allows a neural LM to keep track of the trajectory of the current execution.
This suggests that CoT reasoning might provide an interesting avenue for exploring the non-deterministic representational capacity of neural LMs.
Note that this means CoT reasoning can endow neural LMs with higher expressivity since, e.g., non-deterministic \pfsaAcr{}s are strictly more expressive than deterministic ones.
Similarly, our result that RNNs or transformer LMs with CoT reasoning can simulate PTMs rather than regular Turing machines is important as it allows \emph{probabilistic} computation using such neural LMs.
This means one can sample from them multiple times to assess the certainty of string inclusion in a language.
Furthermore, the problems LMs could solve efficiently can be described by different complexity classes such as BPP, ZPP, etc. instead of P.
\footnote{For a comparison of these complexity classes, see e.g. \citet{papadimitriou1994computational}.}

\section{Related Work} \label{sec:related}
\paragraph{Turing completeness of RNNs.}
Plenty of existing work has investigated the representational capacity of RNNs, both as recognizers as well as LMs \citep[e.g.,][\textit{inter alia}]{McCulloch1943,Kleene1956,siegelmann-sontag-1992,hao-etal-2018-context,DBLP:journals/corr/abs-1906-06349,merrill-2019-sequential,merrill-etal-2020-formal,hewitt-etal-2020-rnns,Chung2021,merrill-etal-2022-saturated,merrill2022extracting,svete-cotterell-2023-recurrent,nowak-etal-2023-representational}.
Most relevant to our work, \citet{siegelmann-sontag-1992,Chung2021} show how RNNs with unbounded precision and unbounded computation time can simulate Turing machines and discuss the implications.
\citet{nowak-etal-2023-representational} extend this to the probabilistic setting, showing that RNN LMs can emulate certain PTMs, but require the LMs to be able to perform non-emitting steps.

\paragraph{Turing completeness of transformers.}
\citet{perez-etal-2021-attention} show the Turing completeness of hard attention encoder--decoder transformer by encoding the configuration of the Turing machine in the output of the transformer.
\citet{bhattamishra-etal-2020-computational} provide a different perspective on Turing completeness of the architecture by showing that transformers can simulate RNNs.
Turing completeness is in this sense a simple consequence of the Turing completeness of RNNs.
\citeposs{bhattamishra-etal-2020-computational} construction, however, relies on emitting symbols from an (uncountably) infinite set (rational-valued vectors), in contradiction to the requirement of the LM working over a \emph{finite} alphabet.
In this sense, they make use of the so-called \emph{regression} transformer setup, which does not lend itself well to the discrete language modeling \citep{ICLRegression}.
In concurrent work, \citet{merrill2024the} adapt \citeposs{perez-etal-2021-attention} result to the CoT setting, connecting Turing completeness of transformer decoders to CoT reasoning similar to our work.
In contrast to our work, they make statements about the model's ability to decide language membership (simulating a Turing machine that accepts or rejects an input) and not to represent probabilistic languages.
In that sense, the transformer they construct is not a language model.
Besides being probabilistic, our construction also enables the analysis of any autoregressive LM architecture, and we focus on RNN and transformer LMs.
\citet{feng2023revealing} similarly show that CoT transformers can perform arithmetic expressions and dynamic programming by generating intermediate results, backing these claims up with experimental evidence, but stopping short of showing Turing completeness.
\citet{du-etal-2023-measure} provide a related result, showing that transformer LMs with continuous transformation functions are \emph{tight}.
This might at first glance contradict the results on Turing completeness, since the latter might require the model to run indefinitely, resulting in a non-tight model.
However, the discrepancy is resolved by noting that the tightness result crucially relies on the use of the softmax normalization function over the reals $\R$ (without $\pm\infty$) in \cref{eq:repr-lm}.
This, together with the encoding function of the transformer, results in a non-diminishing probability of generating $\eos$ and thus in a tight model.
Our use of the sparsemax function sidesteps this issue by allowing us to assign $\eos$ probability $0$.\looseness=-1

\section{Additional Preliminaries}In our construction, we assume that strings are prefixed with the \underline{b}eginning-\underline{o}f-\underline{s}tring symbol $\bos$.
This is to give a base case for recurrent definitions such as \cref{def:lm}. 
This is just for ease of notation in our proofs.
Note that instead of $\pLM\left(\eossym\mid\bos\right)$, we could also equivalently write $\pLM\left(\eossym\mid\eps\right)$.

\subsection{Turing Machines} \label{app:def-tm}
Under our definition, a probabilistic Turing machine has two tapes.\footnote{Adding another tape does not increase the computational power of a Turing machine \citep[Ch. 3]{sipser13}.}
The first is the \defn{processing tape} on which symbols from the tape alphabet can be read and written.
The second is a write-only \defn{output tape} for symbols of the output alphabet.
In the beginning, the processing tape contains the designated $\bot$ symbol in the leftmost cell while all other cells contain the blank symbol $\blanksym$.
The output tape is empty at the beginning.
Starting in the initial state $\qinit$, at each time step $t$, a transition is sampled out of all available transitions for the given state $\stateq$ and the current working tape symbol $\tapesym$, and then applied.
The sampling happens according to the transition probability $w$ (recall that the transition weights $w$ are non-negative and sum to 1 for each pair of state $\stateq$ tape symbol $\tapesym$).
We write transitions as $(\stateq, \tapesym) \xrightarrow{\sym, \tmdir/w} (\stateq^\prime, \tapesym^\prime) $, where $\stateq,\stateq^\prime\in\states, \tapesym, \tapesym^\prime\in\stackalphabet, \sym\in\epsalphabet$, $w \in \Q_{\geq0}$, and $\tmdir\in\tmops$, with the following interpretation.
When the machine is in state $\stateq$ and its head is reading $\tapesym$ on the working tape, it moves to state $\stateq^\prime$, writes $\tapesym^\prime$ to the working tape, writes
$\sym$ to the output tape if $\sym\in\alphabet$ or nothing if $\sym=\eps$, and it moves the head on the working tape by one symbol in the direction $\tmdir$, i.e., to the left ($\tmleft$), or to the right ($\tmright$).
As with any non-deterministic machine, the above definition naturally gives rise to the notion of a tree of possible computations \citep[p. 48]{sipser13}.
A \defn{branch} $\apath$ of the computation tree of a PTM is a sequence of consecutive transitions
\begin{equation}
    (\qinit, \bot) \xrightarrow{\sym_1,\tmdir_1/w_1} (\stateq_2, \tapesym_1^\prime), (\stateq_2, \tapesym_2) \xrightarrow{\sym_2,\tmdir_2/w_2} (\stateq_3, \tapesym_2^\prime), \ldots, (\stateq_\pathlen, \tapesym_\pathlen) \xrightarrow{\sym_\pathlen,\tmdir_\pathlen/w_\pathlen} (\stateq_{\pathlen+1}, \tapesym_{N}^\prime)
\end{equation}
We say that a branch is \defn{accepting} if the branch reaches the final state $\qfinal$.\footnote{The final state has no outgoing transitions.}
The \defn{yield} of an accepting computation branch%
\footnote{We only consider branches that end in a final state when discussing the yield and weight of branches. This means we only take into account finite instances of computation.} is the sequence of symbols $\str\in\kleene{\alphabet}$ written on the output tape at that point in the computation, i.e., a concatenation of the (non-$\eps$) symbols $\sym_1,\ldots,\sym_\pathlen$, where $\pathlen=|\apath|$ is the number of transitions in the branch.
The \defn{weight} of an accepting branch is the product of the weights of its transitions, i.e.,\looseness=-1
\begin{equation}
    \weight(\apath) \defeq \prod_{n=0}^\pathlen w_n.
\end{equation}
We denote the branches that yield a given string $\str$ by $\paths(\tm, \str)$.
The sum of weights of all branches that yield a certain string $\str\in\kleene{\alphabet}$ is the \defn{stringsum} of that string, defined as
\begin{equation}
    \tm(\str)\defeq \sum_{\apath\in\paths(\tm, \str)} \weight(\apath).
\end{equation}
This definition gives rise to a semimeasure over strings whose sum over all possible strings is exactly the halting probability\footnote{The halting probability is the probability that the execution of the PTM will end in a halting state after finitely many steps.} of the PTM, i.e., the probability that starting from the initial state $\qinit$, $\tm$ reaches a final state $\qfinal$
\begin{equation}
    \sum_{\str\in\kleene{\alphabet}}\tm(\str) = \mathbb{P}(\tm \text{ halts}).
\end{equation}

\subsection{Pushdown Automata} \label{sec:pda}
\begin{definition}\label{def:stochastic-wpda}
    A \defn{probabilistic pushdown automaton} (\ppdaAcr{}) is a tuple $\wpdatuple$ where $\states$ is a finite set of states, $\alphabet$ is the input alphabet, $\stackalphabet$ is the stack alphabet, $\trans \subseteq \states \times \stackalphabet \times \epsalphabet \times \Qnonnegative \times \states \times \kleene{\stackalphabet}$ is a finite set of weighted transitions, and $\initialconfig$ and $\finalconfig$ are called the initial and final configuration, respectively.
    Moreover, for all states $\stateq \in \states$ and stack symbols $\stacksym \in \stackalphabet$, $\trans$ satisfies $\sum\limits_{\edge{\left(\stateq, \stacksym\right)}{\sym}{w}{\left(\stateq^\prime, \stackseq\right)} \in \trans} w = 1$.
\end{definition}

Our definition of \ppdaAcr{} is similar to that of \citet{abney-etal-1999-relating}, but, unlike theirs, our machine must end its computation in the final configuration rather than just empty the stack.
We write transitions $\left(\stateq, \stacksym, \sym, w, \stateq^\prime, \stackseq \right) \in \trans$ as $\edge{\left(\stateq, \stacksym\right)}{\sym}{w}{\left(\stateq^\prime, \stackseq\right)}$ which represent a move with weight $w$ from state $\stateq$ to $\stateq^\prime$ while scanning or outputting $\sym$, popping the symbol $\stacksym$ from the stack and pushing the sequence of symbols $\stackseq$.

\begin{definition}
    A \defn{configuration} $\pdaconfig{\stackseq}{\stateq} \in \states \times \kleene{\stackalphabet}$ is a pair containing the current state and the current contents of the stack.
\end{definition}

A \ppdaAcr $\pushdown=\wpdatuple$ is called \defn{deterministic} if for every $\stateq\in\states$, $\stacksym \in \stackalphabet$ and $\sym \in \epsalphabet$, there is at most one transition $\edge{(\stateq, \stacksym)}{\sym}{w}{\stateq^\prime, \stackseq}$ with $w>0$.
Additionally, if there is a transition $\edge{\stateq, \stacksym}{\sym}{w}{\stateq^\prime, \stackseq}$ such that $\sym\in\alphabet$, then there is no transition $\edge{\stateq, \stacksym}{\eps}{w^\prime}{\stateq^{\prime\prime}, \stackseq^\prime}$ with $w^\prime>0$.
Intuitively, there exists at most one next move given any configuration in a deterministic \ppdaAcr{}.
A \defn{run} $\apath$ of a \ppdaAcr{} is a sequence of configurations and transitions,
\begin{equation}
    \pdaconfig{\stackseq_0}{\stateq_0}, \atrans_1, \pdaconfig{\stackseq_1}{\stateq_1}, \ldots, \atrans_N, \pdaconfig{\stackseq_N}{\stateq_N},
\end{equation}
where $\pdaconfig{\stackseq_n}{\stateq_n}$ is the confguration reached by taking transition $\atrans_n$ from configuration $\pdaconfig{\stackseq_{n-1}}{\stateq_{n-1}}$ for any $n \in [1,N]$.
If $\pdaconfig{\stackseq_0}{\stateq_0}=\initialconfig$ and $\pdaconfig{\stackseq_N}{\stateq_N}=\finalconfig$, we call $\arun$ \defn{accepting}.
The \defn{yield} of an accepting run is $\yield\left(\apath\right) \defeq \sym_1\cdots \sym_N$, where $\sym_n$ is the symbol scanned by $\atrans_n$.
The \defn{weight} of an accepting run, $\weight\left(\apath\right)$, is the product of the weight of its transitions,
\begin{equation}
    \weight\left(\apath\right) \defeq \prod_{n=1}^N w_n,
\end{equation}
where $w_n$ is the weight of transition $\atrans_n$.
We write $\paths\left(\pushdown, \str\right)$ for the set of all accepting runs with yield $\str$.
The sum of the weights of all accepting runs of some \ppdaAcr{} $\pushdown$ that yield a certain string $\str \in \kleene{\alphabet}$ is called the \defn{stringsum} and is defined as
\begin{equation}
    \pushdown\left(\str\right) \defeq \sum_{\arun\in\runs\left(\pushdown,\str\right)} \weight\left(\arun\right).
\end{equation}
\subsection{Two-stack Pushdown Automata} \label{sec:twopda}
\begin{definition}
    A \defn{two-stack pushdown automaton} (\twopda) is a machine specified by the $6$-tuple $\pda = \left(\states, \alphabet, \stackalphabet, \trans, \qinit, \qfinal\right)$,
    where $\states$ is a finite set of states, $\alphabet$ is an alphabet of input symbols, $\stackalphabet$ is an alphabet of stack symbols, including the bottom-of-stack symbol $\bot$, $\trans\colon \states \times \stackalphabet \times \epsalphabet \times \Q_{\geq 0} \times \states \times \stackalphabet_{\eps}^4$ is a finite set of rationally-weighted transitions, and $\qinit$ and $\qfinal$ are the initial and the final state, respectively.
\end{definition}
We use a definition similar to that of \citeposs{nowak-etal-2023-representational}, which assumes without loss of generality that transitions are determined by the current state and the top symbol of the first stack only. See \citet[Appendix B]{nowak-etal-2023-representational} for a proof.
We write transitions as $\twoPdaEdgenoweight{q}{\stacksym}{\sym}{\stateq^\prime}{\stacksym_1}{\stacksym_2}{\stacksym_3/w}{\stacksym_4}$, which represent a move with weight $w$ from state $\stateq$, with the top of the first stack $\stacksym$, to state $\stateq^\prime$, while scanning or outputting $\sym$, popping $\stacksym_1$ and $\stacksym_2$ from the first and second stack, respectively, and pushing $\stacksym_3$ and $\stacksym_4$ to the first and second stack, respectively.
A \twopda is called \defn{probabilistic} if, for every state $\stateq$ and top symbol $\stacksym$ of the first stack, the weights of the transitions define a probability distribution, i.e.,
\begin{equation}
    \sum_{\twoPdaEdgenoweight{q}{\stacksym}{\sym}{\stateq^\prime}{\stacksym_1}{\stacksym_2}{\stacksym_3/w}{\stacksym_4} \in \trans} w = 1.
\end{equation}
As before, we define a \defn{run} in the \twopda as a sequence of consecutive transitions. 
A run is called \defn{accepting} if it ends in a final configuration.
The yield of an accepting run is the sequence of symbols $\sym\in\alphabet$ that the \twopda scans (or outputs) during the run.
And the \defn{weight} of an accepting run, $\weight\left(\apath\right)$, is the product of the weight of its transitions,
\begin{equation}
    \weight\left(\apath\right) \defeq \prod_{n=1}^N w_n,
\end{equation}
where $w_n$ is the weight of transition $\atrans_n$.
Finally, the \defn{stringsum} of string $\str$ is defined as the sum of the weights of all accepting runs that yield $\str$, i.e., $\paths\left(\pushdown, \str\right)$:
\begin{equation}
    \pushdown\left(\str\right) \defeq \sum_{\arun\in\runs\left(\pushdown,\str\right)} \weight\left(\arun\right).
\end{equation}

\subsection{Transformer Language Models} \label{app:transformers}

Transformer LMs are LMs whose conditional distributions $\pLNSM\left(\eossym_\tstep \mid \strlt\right)$ are computed by a \emph{transformer}.
A transformer is a composition of multiple transformer \emph{layers}, each of which implements the \emph{attention mechanism}.
We give definitions of these building blocks in what follows.
Our formalization and notation closely follows \citet{svete-etal-2024-transformers}.

\begin{table*} \centering\footnotesize
    \begin{tabular}{@{}clp{8cm}@{}}
        \toprule
        Symbol                                         & Type                                        & Meaning                                                                                                                           \\
        \midrule
        $\NTo{N}$                                      & $\subset \N$                                & The set $\set{1, \ldots, N}$ for $N \in \N$.                                                                                      \\
        $\alphabet, \bosalphabet, \eosalphabet$        & alphabet                                    & $\alphabet$ is a set of symbols, $\bosalphabet \defeq \alphabet \cup \set{\bos}$, $\eosalphabet \defeq \alphabet \cup \set{\eos}$ \\
        $\sym, \bossym, \eossym$                                         & $\in \alphabet$                             & A symbol, element of $\alphabet, \bosalphabet,$ or $\eosalphabet$.                                                                                                 \\
        $\str$                                         & $\in \kleene{\alphabet}$                    & A string over $\alphabet$.                                                                                                        \\
        $\str^{\idxi}_{\idxj}$                         & $\in \kleene{\alphabet}$                    & A substring of $\str$, a string.                                                                                                  \\
        $\onehot{\sym}$                                & $\in \{0, 1\}^{\nsymbols}$                  & One-hot encoding of the symbol $\sym \in \alphabet$.                                                                              \\
        $\hiddDim$                                     & $\in \N$                                    & Size of the contextual representations in the transformer or RNN.                                                                 \\
        $\simplexFun{N - 1}$                           & $\subseteq \R^{N}$                          & The $N-1$-dimensional probability simplex.                                                                                        \\
        $\tfscorefun$                                  & $\R^{\hiddDim} \times \R^{\hiddDim} \to \R$ & A scoring function.                                                                                                               \\
        $\projfunc$                                    & $\R^N \to \simplexFun{N - 1}$               & A normalization function.                                                                                                         \\
        $\qTransf$, $\kTransf$, $\vTransf$, $\oTransf$ & $\R^\hiddDim \to \R^\hiddDim$               & The query, key, value, and output functions.                                                                                      \\
        $\fTransf$                                     & $\R^\hiddDim \to \R^\hiddDim$               & The final transformer LM transformation function.                                                                                 \\
        $\enc$                                         & $\kleene{\alphabet} \to \R^\hiddDim$        & The string representation function.                    \\
        $\posInEmbedding$ & $\bosalphabet \times \N \to \R^\hiddDim$    & The position-augmented representation function.                                                                                   \\
        $\tfnumlayer$                                  & $\in \N$                                    & Number of layers.                                                                                                                 \\
        $\tfheadnum$                                   & $\in \N$                                    & Number of heads.                                                                                                                  \\
        $\tfheadCombine$                               & $\R^{\tfheadnum \hiddDim} \to \R^\hiddDim$  & The head combining function.                                                                                                      \\
        $\left(\cdot\;; \cdots; \cdot\right)$          &                                             & Vertical concatenation operator of vectors or matrices.                                                                           \\
        \bottomrule
    \end{tabular}
    \caption{A summary of the notation used in the paper.}
    \label{tab:notation}
\end{table*}

\paragraph{Notation.}
We use bold unitalicized letters such as $\vx \in \R^\hiddDim$ to denote real-valued vectors and italicized letters $\evx_\idxj \in \R$ for their entries.
Capital bold letters such as $\mX \in \R^{N \times \hiddDim}$ denote matrices.
All vectors are \emph{column} vectors unless transposed.
We define the vertical stacking operator $\left(\cdot \; ; \cdots; \; \cdot\right)$, which denotes the vertical concatenation of the $\hiddDim$-dimensional \emph{column} vectors $\vx_1, \ldots, \vx_N$ into a $N\hiddDim$-dimensional vector $\left(\vx_1; \cdots; \vx_N\right) \in \R^{N \hiddDim}$ and the concatenation of the $\hiddDim$-dimensional \emph{row} vectors $\vx^\top_1, \ldots, \vx^\top_N$ into a matrix $\mX \in \R^{N \times \hiddDim}$ with $N$ rows and $\hiddDim$ columns.
Given the matrix $\mX = \left(\vx^\top_1; \cdots; \vx^\top_N\right)$, we write $\mX_\idxn = \left(\vx^\top_1; \cdots; \vx^\top_\idxn\right)$ for the submatrix composed of the first $n$ rows.
We call a function $\tfscorefun: \R^{\hiddDim} \times \R^{\hiddDim} \to \R$ whose purpose is to evaluate the compatibility of two vectors a \defn{scoring function}.
A \defn{normalization function} $\projfunc\colon \R^N \to \simplexFun{N - 1}$ maps vectors in $\R^N$ to $N$ probabilities.
Here, $\simplexFun{N - 1} \defeq \set{\vx \in \left[0, 1\right]^N \mid \sum_{\idxn = 1}^N \evx_\idxn = 1}$ is the $N-1$-dimensional probability simplex.
This notation is summarized in \cref{tab:notation}.

\paragraph{The Attention Mechanism.}
The attention mechanism works as follows.
It takes a \defn{\underline{q}uery} vector $\vq \in \R^{\hiddDim}$ and two matrices: The matrix $\mK \in \R^{N \times \hiddDim}$ of \defn{\underline{k}eys} and the matrix $\mV\in \R^{N \times \hiddDim}$ of \defn{\underline{v}alues} and computes a weighted average of the value vectors based on the compatibilities of the key vectors to the query vector, as scored by a scoring function $\tfscorefun$.
A formal definition is given below.
\begin{definition}[Attention Mechanism] \label{def:attention}
    The \defn{attention mechanism} $\attn\colon \R^\hiddDim \times \R^{N \times \hiddDim} \times \R^{N \times \hiddDim} \to \R^\hiddDim$ is defined as
    \begin{equation} \label{eq:attention-sum}
        \attn\left(\vq, \mK, \mV\right) \defeq \sum_{\idxn = 1}^{N} \evs_\idxn\vv_\idxn
    \end{equation}
    where $\vq \in \R^{\hiddDim}$ be a query vector and let $\mK = \left(\vk^\top_1; \cdots; \vk^\top_N\right) \in \R^{N \times \hiddDim}$ and $\mV = \left(\vv^\top_1; \cdots; \vv^\top_N\right) \in \R^{N \times \hiddDim}$ be matrices of keys and values, respectively, and 
    \begin{equation}
        \vs \defeq \projfunc\left(\tfscorefun\left(\vq, \vk_1\right), \dots ,\tfscorefun\left(\vq, \vk_N\right) \right)
    \end{equation}
    is the vector of normalized scores between the query $\vq$ and the keys in $\mK$, $\tfscorefun$ is a scoring function and $\projfunc$ is a normalization function.
\end{definition}

\paragraph{Attention types.}
Attention weights are computed by normalizing the scores $\tfscorefun\left(\vq, \vk_1\right), \dots ,\tfscorefun\left(\vq, \vk_\tstep\right)$.
The choice of the projection function $\projfunc$ determines the type of attention and has concrete implications on representational capacity \citep{hao-etal-2022-formal}.
We focus on the \defn{hard attention} projection function.
\begin{definition}  \label{def:hard-attention}
    \defn{Hard attention} is computed with the $\hardmaxAvg$ projection function:
    \begin{equation}
        \hardmaxAvg\left(\vx\right)_\idxd \defeq \begin{cases}
            \frac{1}{m} & \ifcondition \idxd\in \argmax\left(\vx\right) \\
            0           & \otherwisecondition
        \end{cases}
    \end{equation}
    for $\idxd \in \NTo{\hiddDim}$, where $\vx \in \R^\hiddDim$ and $m \defeq |\argmax\left(\vx\right)|$ is the cardinality of the argmax set.
\end{definition}

\begin{definition}
A \defn{multi-layer perceptron} (MLP) $\fTransf\colon\R^\hiddDim\to\R^\hiddDim$ is a function defined as the composition of elementary functions $\vfunc_1, \ldots, \vfunc_L$ 
\begin{equation}
    \fTransf\left(\vx\right) \defeq \vfunc_L \circ \vfunc_{L-1} \circ \cdots \circ \vfunc_1 \left(\vx\right),
\end{equation}
where each function $\vfunc_\ell$ for $\ell\in \NTo{L}$ is defined as
\begin{subequations}
\begin{align}
    \vfunc_\ell(\vx) &\defeq \mlpActivation\left(\mW_\ell\vx + \vb_\ell \right)\quad \ell\in \NTo{L-1} \\
    \vfunc_L(\vx) &\defeq \mW_L\vx + \vb_L,
\end{align}
\end{subequations}
where $\mW_\ell\in\R^{\hiddDim\times\hiddDim}$ is a square weight matrix specific to layer $\ell$, $\vb_\ell\in\R^{\hiddDim}$ is a bias vector, and $\mlpActivation$ is an element-wise non-linear activation function.
The function $\vfunc_1$ is called the \defn{input layer}, the function $\vfunc_L$ is called the \defn{output layer}, and the function $\vfunc_\ell$ for $\ell = 2, \ldots, L-1$ are called \defn{hidden layers}.\footnote{Note that we refer to MLPs by the number of hidden layers, e.g., a one-layer-MLP is an MLP with one \emph{hidden} layer.\looseness=-1
}
\end{definition}

\paragraph{The Transformer Architecture.}
A transformer layer uses the attention mechanism to compute augmented representations $\vz_\tstep = \attn\left(\vq_\tstep, \mK_\tstep, \mV_\tstep\right)$ of the input representations $\mX_\tstep = \left(\vx_1; \cdots; \vx_\tstep\right)$.
The query $\vq_\tstep$, the keys $\mK_\tstep$, and values $\mV_\tstep$ are all transformations of the input representations $\mX_\tstep$.

\begin{definition} \label{def:transformer-layer}
    Given query, key, value, and \defn{\underline{o}utput} functions $\qTransf, \kTransf, \vTransf, \oTransf\colon \R^\hiddDim\!\to\!\R^\hiddDim$,
    a \defn{transformer layer} is a function $\tflayer\colon\R^{\finaltstep \times \hiddDim} \to \R^{\finaltstep \times \hiddDim}$ that computes
    \begin{equation} \label{eq:transf-layer}
        \tflayer\left(\vx_1^\top; \cdots; \vx_\strlen^\top\right) = \left(\vz_1^{\top}; \cdots;  \vz_\finaltstep^{\top}\right) \in \R^{\finaltstep \times \hiddDim}
    \end{equation}
    for $\tstep \in \NTo{\finaltstep}$ where
    \begin{subequations}
        \begin{alignat}{2} \label{eq:attn-block-1}
            \va_\tstep              & \defeq \attn\left(\vq_\tstep, \mK_\tstep, \mV_\tstep\right) + \tflayerinputsy_\tstep &  & \in \R^\hiddDim                           \\
            \tflayeroutputsy_\tstep & \defeq \oTransf\left(\va_\tstep\right) + \va_\tstep                                  &  & \in \R^\hiddDim.  \label{eq:attn-block-2}
        \end{alignat}
    \end{subequations}
    Here, we define
    \begin{subequations}
        \begin{alignat}{2}
            \vq_\tstep & \defeq \qTransf\left(\vx_\tstep\right)                                                            &  & \in \R^\hiddDim                  \\
            \mK_\tstep & \defeq \left(\kTransf\left(\vx_1\right)^\top; \cdots; \kTransf\left(\vx_\tstep\right)^\top\right) &  & \in \R^{\tstep \times \hiddDim}  \\
            \mV_\tstep & \defeq \left(\vTransf\left(\vx_1\right)^\top; \cdots; \kTransf\left(\vx_\tstep\right)^\top\right) &  & \in \R^{\tstep \times \hiddDim}.
        \end{alignat}
    \end{subequations}
    \emph{Note:} For simplicity, we do not include layer normalization.
\end{definition}
The functions $\qTransf, \kTransf, \vTransf$ are usually implemented as linear transformations and $\oTransf$ as an MLP.
Some theoretical literature, however, also considers more general function classes, e.g., all smooth functions \citep{hahn-2020-theoretical}.

Without further modification, the transformations applied by the transformer layer are position-invariant, which necessitates the addition of explicit positional information.
\begin{definition} \label{def:positional-encodings}
    A symbol \defn{representation function} is a function $\inEmbedding\colon \bosalphabet \to \R^{\hiddDim_\inEmbedding}$ and a \defn{positional encoding} is a function $\posEnc\colon \N \to \R^{\hiddDim_\posEnc}$.
    A \defn{position-augmented} representation function $\posInEmbedding\colon \bosalphabet \times \N \to \R^\hiddDim$ (with $\hiddDim = \hiddDim_\inEmbedding + \hiddDim_\posEnc$) is defined as
    \begin{equation}
        \posInEmbeddingFun{\bossym_\tstep, \tstep} \defeq \left(\inEmbeddingFun{\bossym_\tstep}; \posEncFun{\tstep}\right).
    \end{equation}
\end{definition}
\begin{definition} \label{def:static-representations}
    A \defn{static encoding} $\staticRepr$ is a function $\staticRepr\colon \alphabet^\strlen \to \R^{\strlen \times \hiddDim}$ defined for any $\strlen \in \N$ as
    \begin{equation}
        \staticRepr\left(\str\right) \defeq \left(\posInEmbeddingFun{\symone, 1}^\top; \cdots; \posInEmbeddingFun{\sym_\strlen, \strlen}^\top\right).
    \end{equation}
\end{definition}

Multiple transformer layers are stacked into a transformer, which computes the (deep) contextual representations of all symbols in the string.
\begin{definition} \label{def:transformer}
    For $\tfnumlayer \in \N$, an $\tfnumlayer$-layer \defn{transformer} $\tf$ is defined as
    \begin{equation} \label{eq:transformer-model}
        \tfFun{\staticRepr} \defeq \tflayer_\tfnumlayer \circ \cdots \circ \tflayer_1 \circ \staticRepr,
    \end{equation}
    where $\tflayer_\tflayeridx$ for $\tflayeridx \in \NTo{\tfnumlayer}$ are transformer layers and  $\staticRepr$ is a static encoding.
\end{definition}
A transformer computes the contextual representations of the symbols $\str = \sym_1 \cdots \sym_\strlen$ as
\begin{equation}
    \left(\vz_1^{\top}; \cdots; \vz_\strlen^{\top}\right) \defeq\left(\vx_1^{\tfnumlayer\top}; \cdots; \vx_\strlen^{\tfnumlayer\top}\right) \defeq  \tf\left(\staticRepr\right)\left(\str\right).
\end{equation}
If $\staticRepr$ is clear from the context or arbitrary, we will omit it as an argument to $\tf$ and just write $\tf\left(\str\right)$.
\begin{definition} \label{def:enc}
    Given a transformer $\tf$, a \underline{f}inal representation transformation function $\fTransf\colon \R^{\hiddDim} \to \R^{\hiddDim}$, and a string $\str \in \kleene{\alphabet}$ with $|\str| = \strlen$,
    we define the \defn{encoding function} $\tfencfun$ as\looseness=-1
    \begin{equation} \label{eq:enc}
        \tfencfun\left(\str\right) \defeq \fTransf\left(\vz_{\strlen}\right)
    \end{equation}
    where $\vz_{\strlen}$ is the representation of the $\strlen\textsuperscript{th}$ symbol in $\str$ computed by $\tf$, i.e., $\left(\vz_1^{\top}; \cdots; \vz_\strlen^{\top}\right) = \tf\left(\str\right)$.
\end{definition}

\paragraph{Transformer Language Models.}
So far, we have only defined how the transformer architecture can be used to compute the contextual representations of the symbols.
To complete the definition, we define a transformer \emph{language model} as follows.
\begin{definition} \label{def:transformer-plnsm}
    A \defn{transformer LM} $\tfpLM$ is the representation-based autoregressive LM with the representation function $\tfencfun$ from \cref{eq:enc}.
    That is, $\tfpLM$ defines the conditional probability distributions
    \begin{align} \label{eq:transformer-plnsm}
        \tfpLM\left(\symt \mid \strlt \right) \defeq \sparsemaxfunc{\embedMtx\,\tfencfun\left(\strlt\right)}{\symt}.
    \end{align}
\end{definition}

\subsection{Weighted Finite-state Transducers} \label{sec:transducers}

\begin{definition}
    A (rational-)\defn{weighted finite-state transducer} (WFST) is the tuple $\wfsttuple$ where $\alphabet$ and $\oalphabet$ are the input and output alphabets, respectively, $\states$ is a finite set of states, $\trans \subseteq \states \times \epsalphabet \times \oalphabet_\eps \times \Qnonnegative \times \states$ is a finite set of weighted transitions
    where we write transitions  $\left(\stateq, \sym, \symx, w, \stateq^\prime\right) \in \trans$ as $\edge{\stateq}{\sym:\symx}{w}{\stateq^\prime}$,
    and $\initf, \finalf\colon \states \rightarrow \Qnonnegative$ are functions that assign each state its initial and final weight, respectively.
\end{definition}
WFSTs are thus a special class of (weighted) finite-state automata that operate over two alphabets.
Just like \pfsaAcr{}s are a generalization of unweighted FSAs, WFSAs represent a more general class of machines than unweighted FSTs, such as the ones used in the definition of regular reducibility (cf. \cref{def:reg-reduce}).
The weighted versions will prove useful when talking about various aspects of equivalence with regular reducibility.
As a special case of weighted finite-state automata, WFSTs compute weights of their inputs---in this case, weights of \emph{pairs} of strings (the inputs and their outputs):
\begin{equation}
    \wfst\left(\str, \strx\right) \defeq \sum_{\apath \in \paths\left(\wfst, \left(\str, \strx\right)\right)} \weight\left(\apath\right).
\end{equation}
Of interest are also marginal sums of the inputs, i.e.,
\begin{equation}
    \wfst\left(\str\right) \defeq \sum_{\strx \in \kleene{\oalphabet}} \wfst\left(\str, \strx\right).
\end{equation}

\subsubsection{Operations on WFSTs}
We will use WFSTs as building blocks in the exposition of regular reducibility (cf. \cref{sec:rr}).
In this subsection, we outline some operations on the computational models that will be particularly useful.

\paragraph{Composition.}
The \defn{composition} of two WFSTs results in a WFST that, similarly to function composition, maps the strings with the first WFST, passes them to the second one, and returns the (weight of the) output of the second transducer \citep{10.7551/mitpress/3007.003.0017}.
More formally, given the WFSTs $\wfst_1 = \left(\alphabet, \yalphabet, \states, \trans, \initf, \finalf\right)$ and $\wfst_2 = \left(\yalphabet, \oalphabet, \states, \trans, \initf, \finalf\right)$, their composition $\wfst_2 \circ \wfst_1$ computes
\begin{equation} \label{eq:wfst-composition}
    \left(\wfst_2 \circ \wfst_1\right)\left(\str, \strx\right) \defeq \sum_{\strz \in \kleene{\yalphabet}} \wfst_2\left(\strz, \str\right) \cdot \wfst_1\left(\strx,\strz\right).
\end{equation}
Intuitively, \cref{eq:wfst-composition} computes the weight of all the possible ways of mapping $\str$ to $\strx$ by first mapping $\str$ to some $\strz \in \kleene{\yalphabet}$ and then mapping that $\strz$ into $\strx$.
The WFST $\left(\wfst_2 \circ \wfst_1\right)$ can be computed from $\wfst_1$ and $\wfst_2$ in time $\bigO{\left(|\states_1| + |\trans_1|\right) \left(|\states_2| + |\trans_2|\right)}$ \citep{Mohri2009}.

\paragraph{Projecting a WFST.}
Any WFST can be \defn{projected} onto an (input or output) weighted finite-state \emph{automaton} (WFSA)\footnote{A weighted finite-state automaton is a generalization of a probabilistic finite-state automaton where the weights do not have to form probability distributions.
We also extend the definition of WFSAs to allow $\eps$-transitions, i.e. with symbols from $\epsalphabet$. 
Note that every such WFSA can be converted into a weakly equivalent WFSA without $\eps$-transitions.
} that computes the cumulative weights of individual input or output strings \citep{Mohri2008}.
Given a WFST $\wfst = \left(\alphabet, \oalphabet, \states, \trans, \initf, \finalf\right)$, its input projection is the automaton $\wfsa_{\mathcal{I}} = \left(\alphabet, \states, \trans_\mathcal{I}, \initf, \finalf \right)$ that computes
\begin{equation} \label{eq:input-projection}
    \wfsa_{\mathcal{I}}\left(\str\right) \defeq \sum_{\strx \in \kleene{\oalphabet}} \wfst\left(\str, \strx\right).
\end{equation}
The output projection automaton is defined analogously.
While it might not be immediately obvious that \cref{eq:input-projection} represents the computations of a WFSA, $\wfsa_{\mathcal{I}}$ can be easily constructed by ignoring the output labels on the transitions (and additively merging any transitions that become identical after the removal of the output labels).
That is:
\begin{equation}
    \trans_\mathcal{I} = \set{\edge{\stateq}{\sym}{\sum_{\edge{\stateq}{\sym:\symx}{w}{\stateq^\prime} \in \trans} w}{\stateq^\prime} \mid \stateq, \stateq^\prime \in \states, \sym \in \epsalphabet}.
\end{equation}

\paragraph{Lifting a \pfsaAcr.}
\defn{Lifting} transforms a given \pfsaAcr into a trivial WFST that implements the identity function $\str \mapsto \str$.
Concretely, given the \pfsaAcr $\wfsa = \wfsatuple$, its lifted WFST $\wfst_\wfsa = \left(\alphabet, \alphabet, \states, \trans_\wfst, \initf, \finalf\right)$ defines the transitions
\begin{equation}
    \trans_\wfst \defeq \set{\edge{\stateq}{\sym:\sym}{w}{\stateq^\prime} \mid \edge{\stateq}{\sym}{w}{\stateq^\prime} \in \trans}.
\end{equation}
It is easy to see that $\wfst_\wfsa$ computes
\begin{equation}
    \wfst_\wfsa\left(\str,\str^\prime\right) = \ind{\str = \str^\prime} \wfsaFun{\str}.
\end{equation}
Lifting is useful when one wants to compose a \pfsaAcr with a WFST, as is the case when discussing regular reducibility.

\paragraph{Inverting a WFST.}
The \defn{inverted} WFST of some WFST $\wfst = \left(\alphabet, \oalphabet, \states, \trans, \initf, \finalf\right)$ is the WSFT $\inv{\wfst} = \left(\oalphabet, \alphabet, \states, \inv{\trans}, \initf, \finalf\right)$ that maps the outputs of $\wfst$ to their inputs.
$\inv{\trans}$ is defined as
\begin{equation}
    \inv{\trans} \defeq \set{\edge{\stateq}{\symx: \sym}{w}{\stateq} \mid \edge{\stateq}{\sym: \symx}{w}{\stateq} \in \trans},
\end{equation}
i.e., it is composed of transitions from $\trans$ with their input--output labels flipped.

\section{Separation of Alphabets}\label{app:alphabets}
In our analysis, we use a different alphabet for the chain of thought ($\tapealphabet$) than for the final output ($\alphabet$) to clarify the distinction between output that encodes automata configurations vs the output that constitutes the weighted output language of that automaton.
However, we can also define a regular function for the case where $\tapealphabet=\alphabet$.
For this, we choose a specific symbol in $\alphabet$, e.g., $\sep$, whose first occurrence signifies the boundary between the chain of thought and the language output.
Then, $\regfn$ can be defined as:
\begin{align}
    \regfn(\stry\,\sep\,\strz) \defeq \strz
\end{align}
where $\stry\in\kleene{\alphabet}\backslash \{\sep\}$, and $\strz\in\kleene{\alphabet}$.
This can be easily modeled by the following simple two-state transducer:
\begin{figure}[ht!]
    \centering
    \begin{tikzpicture}[node distance = 15mm]
        \footnotesize
        \node[state, initial] (q0) [] { $0$ };
        \node[state, accepting] (q1) [right = of q0] { $1$ };
        \draw[-{Latex[length=2mm]}] (q0) edge[loop below] node{ $\symy\in\alphabet\backslash\{\sep\}:\eps$ } (q0)
        (q0) edge[align=left, bend left, above] node{ $\sep\ : \eps$ } (q1)
        (q1) edge[loop below] node{ $\symz\in\alphabet : \symz$ } (q1);
    \end{tikzpicture}
    \caption{A simple WFST $\transducer$ implementing $\regfn$.} \label{fig:wfst-phi2}
\end{figure}

Note that in this case, the chain of thought has to be confined in its entirety as one string, with the entire final output coming afterward.
This simply means that the simulated Turing machine first performs all the computations it needs on the working tape, and then writes the final string to the output tape in real time.
The proof of \cref{thm:cot-rnn-ptm} can easily be adapted to work in this setting by only tagging the $\eps$ symbols in $\tapealphabet$, e.g. by symbols $a$ and $b$ and defining $\regfn$ to remove them until the occurrence of another symbol, e.g. $c$.

\section{Proofs: Regular Reducibility} \label{sec:rr-proofs}
\subsection{Nondeterminism in \pfsaAcr{}s}
\rrDPFSA*
\begin{proof}
    Given the \pfsaAcr $\wfsa = \wfsatuple$, we want to show the existence of a \emph{deterministic} \pfsaAcr $\wfsa^\prime$ over the alphabet $\alphabet\times \states$ that is regularly reducible to $\wfsa$.
    We construct the deterministic \pfsaAcr $\wfsa^\prime = \left(\alphabet \times \states, \alphabet \times \states, \trans^\prime, \initf^\prime, \finalf^\prime\right)$ as follows.
    \begin{subequations}
        \begin{align}
            \trans^\prime                            & \defeq \set{\edge{\left(\sym, \stateq\right)}{\left(\sym^\prime, \stateq^\prime\right)}{w}{\left(\sym^\prime, \stateq^\prime\right)} \mid \edge{\stateq}{\sym^\prime}{w}{\stateq^\prime} \in \trans, \sym \in \alphabet} \\
            \initf^\prime\left(\sym, \stateq\right)  & \defeq \frac{\initf\left(\stateq\right)}{\nsymbols}, \quad \forall \sym\in\alphabet, \stateq\in\states                                                                               \\
            \finalf^\prime\left(\sym, \stateq\right) & \defeq \finalf\left(\stateq\right), \quad \forall \sym\in\alphabet, \stateq\in\states
        \end{align}
    \end{subequations}

    We first prove that $\wfsa^\prime$ is probabilistic and deterministic.
    \begin{itemize}
        \item \textbf{Probabilistic}:
              We compute
              \begin{subequations}
                  \begin{align}
                      \sum_{\left(\sym, \stateq\right) \in \alphabet \times \states} \initf^\prime\left(\sym, \stateq\right)
                       & = \sum_{\left(\sym, \stateq\right) \in \alphabet \times \states} \frac{\initf\left(\stateq\right)}{\nsymbols} \\
                       & = \sum_{\sym \in \alphabet} \sum_{\stateq \in \states} \frac{\initf\left(\stateq\right)}{\nsymbols}           \\
                       & = \sum_{\stateq \in \states} \initf\left(\stateq\right)                                                       \\
                       & = 1
                  \end{align}
              \end{subequations}
              and, for any $\left(\sym, \stateq\right) \in \alphabet \times \states$,
              \begin{subequations}
                  \begin{align}
                      \sum_{\edge{\left(\sym, \stateq\right)}{\left(\sym^\prime, \stateq^\prime\right)}{w}{\left(\sym^\prime, \stateq^\prime\right)} \in \trans^\prime} w + \finalf^\prime\left(\sym, \stateq\right)
                       & = \sum_{\edge{\left(\sym, \stateq\right)}{\left(\sym^\prime, \stateq^\prime\right)}{w}{\left(\sym^\prime, \stateq^\prime\right)} \in \trans^\prime} w + \finalf\left(\stateq\right) \\
                       & = \sum_{\edge{\stateq}{w}{\sym^\prime}{\sym^\prime} \in \trans} w + \finalf\left(\stateq\right)                                                                   \\
                       & = 1
                  \end{align}
              \end{subequations}
        \item \textbf{Deterministic}:
              Let $\left(\sym, \stateq\right) \in \alphabet \times \states$ be a state of $\wfsa^\prime$ and $\left(\sym^\prime, \stateq^\prime\right)$ a symbol in its alphabet.
              By definition of $\trans^\prime$, $\left(\sym^\prime, \stateq^\prime\right)$ uniquely determines the target state, which is identical to the symbol---$\left(\sym^\prime, \stateq^\prime\right)$.
    \end{itemize}

    We now show that $\wfsa^\prime$ is indeed regularly reducible to $\wfsa$.
    The alphabet $\alphabet \times \states$ clearly satisfies the $\alphabet$-augmentation condition.
    We define the regular function $\regfn\left(\sym, \stateq\right) \defeq \sym$, an instance of the general function described in \cref{sec:rr}.
    We will show that $\wfsa^\prime$ is weakly equivalent to $\wfsa \circ \regfn$, or, equivalently, that $\wfsa^\prime \circ \inv{\regfn}$ is weakly equivalent to $\wfsa$.

    Since $\wfsa^\prime$ is an \emph{acceptor} of strings in $\kleene{\left(\alphabet \times \states\right)}$, we first transform it into weighted finite-state \emph{transducer} $\wfst_\wfsa$ implementing the mapping\footnote{See \cref{sec:transducers} for a definition of weighted finite-state transducers and the operations on them.}
    \begin{equation}
        \wfst_{\wfsa^\prime}\left(\strx, \strx^\prime\right) \defeq \ind{\strx = \strx^\prime} \wfsaFun{\strx}
    \end{equation}
    for $\strx, \strx^\prime \in \kleene{\outalphabet}$.
    This is a simple WFST with transitions $\trans_{\wfst_{\wfsa^\prime}} \defeq \set{\edge{\stateq}{\left(\sym, \stateq\right) : \left(\sym, \stateq\right)}{w}{\stateq^\prime} \mid \edge{\stateq}{\left(\sym, \stateq\right)}{w}{\stateq^\prime} \in \trans^\prime}$.
    $\wfst_{\wfsa^\prime}$ can then be composed with the weighed version of the \fst implementing $\inv{\regfn}$, $\wfst_{\inv{\regfn}}$:
    \begin{equation}
        \wfst_{\inv{\regfn}} \left(\str, \strx\right) \defeq \ind{\strx \in \inv{\regfn}\left(\str\right)}.
    \end{equation}
    $\wfst_{\wfsa^\prime} \circ \wfst_{\inv{\regfn}}$ then computes, for $\str \in \kleene{\alphabet}, \strx^\prime \in \kleene{\outalphabet}$,
    \begin{subequations}
        \begin{align}
            \left(\wfst_{\wfsa^\prime} \circ \wfst_{\inv{\regfn}}\right) \left(\str, \strx^\prime\right)
             & = \sum_{\strx \in \kleene{\outalphabet}} \wfst_{\inv{\regfn}}\left(\str, \strx\right) \cdot \wfst_{\wfsa^\prime}\left(\strx, \strx^\prime\right) \justification{Eq. 15.4 in \citet{10.7551/mitpress/3007.003.0017}.} \\
             & = \sum_{\strx \in \kleene{\outalphabet}} \ind{\strx \in \inv{\regfn}\left(\str\right)} \cdot \wfst_{\wfsa^\prime}\left(\strx, \strx^\prime\right)                                                                    \\
             & = \sum_{\strx \in \inv{\regfn}\left(\str\right)} \wfst_{\wfsa^\prime}\left(\strx, \strx^\prime\right)                                                                                                                \\
             & = \sum_{\strx \in \inv{\regfn}\left(\str\right)} \ind{\strx = \strx^\prime} \wfsaFun{\strx}                                                                                                                    \\
             & = \ind{\strx^\prime \in \inv{\regfn}\left(\str\right)} \wfsaFun{\strx^\prime}
        \end{align}
    \end{subequations}
    We then turn $\wfst_{\wfsa^\prime} \circ \wfst_{\inv{\regfn}}$ into a \pfsaAcr in the standard way by projecting the WFST onto the transition input labels \citep{Mohri2008}.
    This results in a \pfsaAcr that assigns the string $\str \in \kleene{\alphabet}$ the probability that equals the sum of all possible mappings of $\str$ to any $\strx^\prime \in \kleene{\outalphabet}$:
    \begin{subequations}
        \begin{align}
            \left(\wfst_{\wfsa^\prime} \circ \wfst_{\inv{\regfn}}\right) \left(\str\right)
             & \defeq \sum_{\strx^\prime \in \kleene{\outalphabet}} \left(\wfst_{\wfsa^\prime} \circ \wfst_{\inv{\regfn}}\right) \left(\str, \strx^\prime\right) \\
             & = \sum_{\strx^\prime \in \kleene{\outalphabet}} \ind{\strx^\prime \in \inv{\regfn}\left(\str\right)} \wfsaFun{\strx^\prime}                       \\
             & = \sum_{\strx^\prime \in \inv{\regfn} \left(\str\right)} \wfsaFun{\strx^\prime}                                                             \\
             & = \wfsa^\prime\left(\inv{\regfn} \left(\str\right)\right)                                                                             \\
             & = \left(\wfsa^\prime \circ \inv{\regfn}\right) \left(\str\right)                                                                      \\
             & = \wfsa \left(\str\right)
        \end{align}
    \end{subequations}
    This shows that $\wfsa^\prime \circ \inv{\regfn}$ is weakly equivalent to $\wfsa$, which finishes the proof.
\end{proof}

\subsection{Nondeterminism in \ppdaAcr{}s}

\begin{theorem} \label{thm:rr-ppda}
    Let $\pushdown = \wpdatuple$ be a \ppdaAcr. Then, there exists a \emph{deterministic} \ppdaAcr $\pushdown^\prime$ over the alphabet $\left(\epsalphabet\right) \times \states \times \stackalphabet$ with the state space $\left(\epsalphabet\right) \times \states \times \stackalphabet$ that is regularly reducible to $\pushdown$.
\end{theorem}
\begin{proof}
    Given a \ppdaAcr{} $\pushdown=\wpdatuple$, we prove the existence of a \emph{deterministic} \ppdaAcr{} $\pushdown^\prime$ over the alphabet $\epsalphabet \times \states \times \stackalphabet$ that is regularly reducible to $\pushdown$.

    We construct the deterministic \ppdaAcr{} $\pushdown^\prime = \left(\epsalphabet \times \states \times \stackalphabet,\epsalphabet \times \states \times \stackalphabet, \stackalphabet, \trans^\prime, \initialconfig, \finalconfig\right)$ with the set of transitions
    \begin{align}
        \trans^\prime \defeq \{\edge{(\sym, \stateq, \stacksym), \stacksym^\prime}{(\sym^\prime, \stateq^\prime, \stacksym^\prime)}{w}{(\sym^\prime, \stateq^\prime, \stacksym^\prime), \stackseq} \mid & \edge{\stateq, \stacksym^\prime}{\sym^\prime}{w}{\stateq^\prime, \stackseq} \in \trans, \nonumber \\
        & \sym \in \epsalphabet, \stacksym \in \stackalphabet\}.
    \end{align}

    \begin{itemize}
        \item \textbf{Probabilistic:}
              For any $\left(\sym, \stateq, \stacksym\right) \in \epsalphabet \times \states \times \stackalphabet$ and $\stacksym^\prime \in \stackalphabet$,
              \begin{subequations}
                  \begin{align}
                      \sum_{\edge{(\sym, \stateq, \stacksym), \stacksym^\prime}{(\sym^\prime, \stateq^\prime, \stacksym^\prime)}{w}{(\sym^\prime, \stateq^\prime, \stacksym^\prime), \stackseq}} w & = \sum_{\edge{\stateq, \stacksym^\prime}{\sym^\prime}{w}{\stateq^\prime, \stackseq}} w \\
                                                                                                                                     & = 1
                  \end{align}
              \end{subequations}
        \item \textbf{Deterministic:} Just like in the finite-state case, let $\left(\sym, \stateq,\stacksym\right) \in \epsalphabet \times \states \times \stackalphabet$ be a state, $\stacksym^\prime \in \stackalphabet$ a stack symbol and $\left(\sym^\prime, \stateq^\prime,\stacksym^\prime\right) \in \epsalphabet \times \states \times \stackalphabet$ an input symbol.
              Then $\left(\sym^\prime, \stateq^\prime,\stacksym^\prime\right)$ uniquely determines the next configuration.
    \end{itemize}
    We now show that $\pushdown^\prime$ is regularly reducible to $\pushdown$.
    More precisely, we define the regular function $\regfn(\sym,\stateq,\stacksym)\defeq \sym$ and show that $\pushdown^\prime \circ \inv{\regfn}$ is weakly equivalent to $\pushdown$.

    Let $\grammar^\prime$ be a \cfg{}\footnote{Similar to the notation used for \ppdaAcr{}s, we use $\grammar\left(\str\right)$ to denote the weight of the string $\str$ under the grammar $\grammar$.} that is equivalent to $\pushdown^\prime$ \cite[see][Section 6.3.2 for the construction]{hopcroft01}.
    Using grammars, rather than pushdown automata, will simplify the proof as there are well-known constructions (e.g., \citet{barhillel-etal-1961}, \citet{pasti-etal-2023-intersection}) for composing \cfg{}s and \fst{}s.
    We therefore prove that  $\grammar^\prime \circ \inv{\regfn}$ is weakly equivalent to $\pushdown$.

    Recall that $\inv{\regfn}$ can be implemented as an \fst{} $\transducer_\inv{\regfn}$ defined as
    \begin{equation}
        \transducer_\inv{\regfn} \left(\str, \strx\right) \defeq \ind{\strx \in \inv{\regfn}\left(\str\right)}.
    \end{equation}

    Adapting Corollary 1 from \citet{pasti-etal-2023-intersection} to the \fst{} case, we get
    \begin{subequations}
        \begin{align}
            \left(\grammar^\prime \circ \transducer_\inv{\regfn}\right) \left(\str, \strx\right) & = \grammar^\prime \left(\strx\right) \cdot  \transducer_\inv{\regfn} \left(\str, \strx\right) \\
                                                                                                 & = \grammar^\prime \left(\strx\right) \cdot \ind{\strx \in \inv{\regfn}\left(\str\right)}.
        \end{align}
    \end{subequations}

    Then, just like in the regular case, we compute $\left(\grammar^\prime \circ \transducer_\inv{\regfn}\right) \left(\str\right)$ by summing over all possible $\strx \in \kleene{\alphabet}$ values:
    \begin{subequations}
        \begin{align}
            \left(\grammar^\prime \circ \transducer_\inv{\regfn}\right) \left(\str\right) & \defeq \sum_{\strx \in \kleene{\alphabet}} \left(\grammar^\prime \circ \transducer_\inv{\regfn}\right) \left(\str, \strx\right) \\
                                                                                          & = \sum_{\strx \in \kleene{\alphabet}} \grammar^\prime \left(\strx\right) \ind{\strx \in \inv{\regfn}\left(\str\right)}          \\
                                                                                          & = \sum_{\strx \in \inv{\regfn}\left(\str\right)} \grammar^\prime \left(\strx\right)                                             \\
                                                                                          & = \sum_{\strx \in \inv{\regfn}\left(\str\right)} \pushdown^\prime \left(\strx\right)                                            \\
                                                                                          & = \pushdown^\prime\left(\inv{\regfn}\left(\str\right)\right)                                                                    \\
                                                                                          & = \left(\pushdown^\prime \circ \inv{\regfn}\right)\left(\str\right)                                                             \\
                                                                                          & = \pushdown\left(\str\right).
        \end{align}
    \end{subequations}
    This shows that $\grammar^\prime \circ \transducer_\inv{\regfn}$ is weakly equivalent to $\pushdown$, thus $\pushdown^\prime \circ \transducer_\inv{\regfn}$ is weakly equivalent to $\pushdown$.
    This concludes the proof.
\end{proof}

\section{Proofs: Representational Capacity of Neural LMs} \label{app:proofs}

\subsection{Finite-state Language Models}
\minskyRNN*
\begin{proof}
    Let $\wfsa = \wfsatuple$ be a \pfsaAcr.
    We divide the definition of a weakly equivalent Elman LM $\rnn = \left(\alphabet \times \states, \hiddDim, \recMtx, \inMtx, \biasVech, \initstate\right)$ with the output matrix $\outMtx$ into multiple steps.\footnote{Notice that $\rnn$ is CoT-augmented by construction, as it works over the alphabet $\outalphabet \defeq \alphabet \times \states$. $\outalphabet$ also satisfies the $\alphabet$-augmentation condition.}

    \textbf{Note:}
    The paper assumes a $\bos$-padding of the input strings.
    Since CoT-augmented LMs work over a potentially larger alphabet, the padding symbol has to be changed accordingly.
    Particularly, the CoT-augmented RNN will work over the alphabet $\alphabet \times \states$.
    For simplicity, assume that any input string is padded by the symbol $\left(\bos, \stateq_0\right)$ for some arbitrary (but fixed) state $\stateq_0 \in \states$.

    \paragraph{Hidden states.}
    Let $\hiddDim = |\bosalphabet_\eps| \nstates$.
    We will represent the input symbols with the one-hot representation function $\onehot{\cdot, \cdot}$ that computes $\onehot{\bossym, \stateq} \in \set{0, 1}^{|\bosalphabet_\eps| \nstates}$:\footnote{Throughout the paper, we index vectors and matrices directly with set elements rather than integer values, meaning each index corresponds to exactly one element from the given set.\looseness=-1}
    \begin{equation} \label{eq:one-hot}
        \onehot{\bossym, \stateq}_{\bossym, \stateq} \defeq 1
    \end{equation}
    while other entries of the vector are zero.\footnote{If any of the two arguments are empty, its corresponding component is also zero.}
    We can define $\initstate \defeq \zero_\hiddDim$.

    \paragraph{Recurrence.}
    Conveniently, the CoT-augmented RNN will store the current \pfsaAcr state in its output ``symbol''.
    Therefore, its next hidden state $\hiddStatet$ will not, in fact, depend on the previous one ($\hiddStatetminus$); $\hiddStatetminus$ will only be used to compute the next-symbol--state output distribution through \cref{eq:repr-lm}.
    The RNN therefore only has to appropriately incorporate the input information.
    With this in mind, we define the following Elman RNN parameters $\recMtx, \inMtx$, and $\biasVech$:
    \begin{subequations}
        \begin{alignat}{2}
            \recMtx   & \defeq \mO_\hiddDim                 &  & \recMtx \in \R^{\hiddDim \times \hiddDim} \\
            \inMtx    & \defeq \mI_\hiddDim                 &  & \inMtx \in \R^{\hiddDim \times \hiddDim}  \\
            \biasVech & \defeq \zero_\hiddDim \qquad \qquad &  & \biasVech \in \R^\hiddDim
        \end{alignat}
    \end{subequations}
    where $\mO_\hiddDim$ is a $\hiddDim\times\hiddDim$-dimensional matrix of zeros, $\mI_\hiddDim$ is the $\hiddDim$-dimensional identity matrix, and $\zero_\hiddDim$ is a $\hiddDim$-dimensional vector of zeros.
    Then, we define the output matrix $\outMtx \in \R^{|\bosalphabet_\eps| \nstates \times \hiddDim}$ as
    \begin{subequations}
        \begin{alignat}{2} \label{eq:rnn-outmtx}
            \eOutMtx_{\left(\sym^\prime, \stateq^\prime\right), \left(\sym, \stateq\right)}  & \defeq \pLM\left(\stateq^\prime, \sym^\prime \mid \stateq\right) \qquad \qquad &  & \stateq, \stateq^\prime \in \states, \quad \sym \in \alphabet_\eps, \sym^\prime \in \alphabet_\eps \\
            \eOutMtx_{\left(\eps, \stateq\right), \left(\bos, \stateq_0\right)} & \defeq \initfFun{\stateq}                                         &  & \stateq \in \states \\
            \eOutMtx_{\left(\eos, \stateq\right), \left(\sym, \stateq\right)}  & \defeq \finalfFun{\stateq} \qquad \qquad &  & \stateq \in \states, \quad \sym \in \alphabet_\eps.
        \end{alignat}
    \end{subequations}

    \paragraph{Computation of string probabilities.}
    We now show that $\rnn$ with $\outMtx$ defines the same conditional distributions as $\wfsa$, meaning that it implicitly defines the same probability distribution over strings.
    As mentioned above, we assume that the input string is prefixed with the $\left(\bos, \stateq_0\right)$ symbol.
  Let the previously generated output symbol $\in \outalphabet$ be $\left(\bossym, \stateq\right)$ (in the first step, the ``generated'' pair is $\left(\bos, \stateq_0\right)$).
    We get that
    \begin{subequations}
        \begin{align}
            \hiddState_\tstep
             & = \heavisideFun{\recMtx \hiddState_{\tstep - 1} + \inMtx \onehot{\bossym, \stateq} + \biasVech}           \\
             & = \heavisideFun{\mO_\hiddDim \hiddState_{\tstep - 1} + \inMtx \onehot{\bossym, \stateq} + \zero_\hiddDim} \\
             & = \heavisideFun{\inMtx \onehot{\bossym, \stateq}}                                                         \\
             & = \onehot{\bossym, \stateq}
        \end{align}
    \end{subequations}
    The one-hot encoding $\onehot{\bossym, \stateq}$ is then used to ``index'' the appropriate column of the output matrix $\outMtx$ as $\outMtx \hiddStatet$, which contains the probabilities of the \emph{next} symbol--state pairs defined by the \pfsaAcr, as per \cref{eq:rnn-outmtx}:
    \begin{equation}
        \left(\outMtx \hiddStatet\right)_{\left(\stateq^\prime, \eossym^\prime\right)} = \eOutMtx_{\left(\eossym^\prime, \stateq^\prime\right), \left(\bossym, \stateq\right)} = \begin{cases}
            \initfFun{\stateq^\prime}                                 & \ifcondition \bossym = \bos, \stateq = \stateq_0, \eossym^\prime = \eps \\
            \pLM\left(\stateq^\prime, \eossym^\prime \mid \stateq\right) & \otherwisecondition                                         \\
        \end{cases}.
    \end{equation}
    Passing $\outMtx \hiddStatet$ through the normalization function, the next state and symbol are sampled and passed as input to the RNN, repeating the update step.
    Since the Elman RNN following these dynamics by constructions generates paths (sequences of states in the output symbols) with the same probabilities as the input automaton, it enumerates all accepting paths of any string $\str \in \kleene{\alphabet}$ with the same probabilities as the original \pfsaAcr.
    Applying the transformation $\regfn$ to the produced outputs and summing over sequences that yield the same string will thus result in strings sampled from the \pfsaAcr.
    This finishes the proof.
\end{proof}

\minskyRNNBack*
\begin{proof}
    Let $\pLMcot$ be a CoT Heaviside Elman RNN LM over the alphabet $\alphabet$ with the regular function $\regfn\colon \kleene{\outalphabet} \to \kleene{\alphabet}$ for an $\alphabet$-augmented alphabet $\outalphabet$.
    Let $\pLM$ be the underlying RNN LM over $\kleene{\outalphabet}$, that is,
    \begin{align*}
        \pLMcot & = \pLM \circ \inv{\regfn} &  & \text{and} &  & \pLM = \pLMcot \circ \regfn
    \end{align*}
    by \cref{def:cot-lm}.
    We want to show the existence of a possibly non-deterministic \pfsaAcr $\cotMacro{\wfsa}$ that is weakly equivalent to $\pLMcot$.
    We write $\pLM\weakeq\pLM^\prime$ to mean $\pLM$ is weakly equivalent to $\pLM^\prime$ (cf. \cref{def:lm-equivalence}).

    By \citet[Lem. 4.1]{svete-cotterell-2023-recurrent}, there exists a \dpfsaAcr{} $\wfsa = \left(\outalphabet, \states, \trans, \initf, \finalf\right)$ over the augmented alphabet $\outalphabet$ weakly equivalent to $\pLM$.
    This means that
    \begin{equation}
        \wfsa \weakeq \pLM \weakeq \pLMcot \circ \regfn
    \end{equation}
    and thus
    \begin{equation}
        \pLMcot \weakeq \wfsa \circ \inv{\regfn}.
    \end{equation}
    Just like in the proof of \cref{thm:rr-pfsa}, $\wfsa \circ \inv{\regfn}$ can be computed as a composition of (lifted) WFSTs that is then projected onto the input component, resulting in a \pfsaAcr weakly equivalent to $\pLMcot$.
    This finishes the proof.
\end{proof}

\transformerMinsky*

Before we proceed to the full proof of \cref{thm:pfsa-transformer}, we first show the following simple but useful lemma.
\begin{lemma} \label{lem:layer-identity}
    Define the following linear transformation functions:
    \begin{subequations}
        \begin{align}
            \vTransf\left(\vx\right) & \defeq \zero, \\
            \oTransf\left(\vx\right) & \defeq \zero
        \end{align}
    \end{subequations}
    for all $\vx\in\R^\hiddDim$.
    A transformer layer $\tflayer$ with the parameters $\vTransf$ and $\oTransf$ and residual connections implements the identity function irrespective of the parameters $\qTransf, \kTransf$, and $\tfscorefun$:
    \begin{equation}
        \tflayer\left(\vx\right) = \vx
    \end{equation}
    for all $\vx \in \R^\hiddDim$.
\end{lemma}
\begin{proof}
    By definition of a transformer layer (cf. \cref{eq:transf-layer,eq:attn-block-1,eq:attn-block-2}), $\tflayer$ computes
    \begin{subequations}
        \begin{alignat}{2}
            \va                      & = \attn\left(\vq, \mK, \mV\right) + \tflayerinputsy = \sum_{\idxn = 1}^{N} \evs_\idxn\vv_\idxn  + \tflayerinputsy = \sum_{\idxn = 1}^{N} \evs_\idxn\zero  + \tflayerinputsy = \tflayerinputsy                       \\
            \tflayer\left(\vx\right) & = \oTransf\left(\va\right) + \va = \zero + \va = \va  = \vx                                                                                                                                                       .
        \end{alignat}
    \end{subequations}
\end{proof}

This allows us to show \cref{thm:pfsa-transformer}.
\begin{proof}
    Let $\wfsa = \wfsatuple$ be a \pfsaAcr.
    Like the CoT-augmented RNN in \cref{thm:pfsa-rnn}, the CoT-augmented transformer LM will work over the alphabet $\outalphabet \defeq \bosalphabet_\eps \times \states$.
    It will start by generating a random initial state according to the initial-state distribution $\initf$ upon reading the designated padding symbol $\left(\bos, \stateq_0\right)$.
    Then, at step $\tstep$ of generation, it will generate the next symbol--state pair by sampling from the next-transition distribution $\pLM\left(\stateq_\tstep, \bossym_\tstep\mid \stateq_{\tstep - 1}\right)$ given the state stored in the previously generated output tuple.

    More formally, define $\bos^\prime \defeq \left(\bos, \stateq_0\right)$ for some arbitrary but fixed $\stateq_0 \in \states$.
    Further, define the static representation function
    \begin{equation}
        \staticRepr\left(\left(\bossym, \stateq\right), \tstep\right) \defeq \onehot{\bossym, \stateq} \in \set{0, 1}^{|\bosalphabet_\eps| \nstates}
    \end{equation}
    as in \cref{eq:one-hot}.
    Let $\tflayer$ be the identity transformer layer from \Cref{lem:layer-identity} and $\tf$ the transformer with the single layer $\tflayer$.
    Then, by \Cref{lem:layer-identity},
    \begin{equation}
        \tf\left(\staticRepr\right)\left(\left(\bos, \stateq_0\right), \left(\sym_1, \stateq_1\right), \ldots, \left(\sym_{\tstep - 1}, \stateq_{\tstep - 1}\right)\right) = \left(\onehot{\bos, \stateq_0}^\top; \onehot{\sym_1, \stateq_1}^\top; \cdots; \onehot{\sym_{\tstep - 1}, \stateq_{\tstep - 1}}^\top\right)
    \end{equation}
    for any $\tstep$.
    Defining the transformation $\fTransf\left(\vx\right) \defeq \vx$, it therefore holds that
    \begin{equation}
        \tfencfun\left(\left(\bos, \stateq_0\right), \left(\sym_1, \stateq_1\right), \ldots, \left(\sym_{\tstep - 1}, \stateq_{\tstep - 1}\right)\right) = \fTransf\left(\vx_{\tstep-1}^1\right) = \onehot{\sym_{\tstep - 1}, \stateq_{\tstep - 1}}.
    \end{equation}
    In words, $\enc\left(\strlt\right)$ contains the current symbol--state pair of the \pfsaAcr.
    The CoT-augmented transformer LM can therefore sample the next symbol--state pair from the conditional distribution defined by the \pfsaAcr by setting the values of the output matrix $\outMtx$ as in \cref{thm:rnn-pfsa}:
    \begin{subequations}
        \begin{alignat}{2} \label{eq:transformer-outmtx}
            \eOutMtx_{\left(\sym^\prime, \stateq^\prime\right), \left(\sym, \stateq\right)}  & \defeq \pLM\left(\stateq^\prime, \sym^\prime \mid \stateq\right) \qquad \qquad &  & \stateq, \stateq^\prime \in \states, \quad \sym, \sym^\prime \in \alphabet \\
            \eOutMtx_{\left(\eps, \stateq^\prime\right), \left(\bos, \stateq_0\right)} & \defeq \initfFun{\stateq}                                         &  & \stateq \in \states \\
            \eOutMtx_{\left(\eos,\stateq\right),\left(\sym,\stateq\right)} & \defeq \finalfFun{\stateq} & & \stateq\in\states, \sym\in\epsalphabet.
        \end{alignat}
    \end{subequations}
    Clearly, the identical conditional probabilities defined by the CoT-augmented transformer $\tf$ result in strings over $\bosalphabet_\eps \times \states$ sampled with probabilities equal to the path probabilities from the \pfsaAcr, as in \cref{thm:rnn-pfsa}.
    Applying the transformation $\regfn$ to the produced outputs will thus result in strings sampled from the \pfsaAcr, giving us a CoT-augmented transformer LM weakly equivalent to $\wfsa$.
    Lastly, since all the representations in the model are position-invariant, the constructed transformer is of constant precision.
\end{proof}

\subsection{Probabilistic Turing Machine Language Models}
\paragraph{RNN LMs with CoT reasoning.}
Next, we show that, for every PTM, its induced LM can be encoded in an unbounded precision CoT RNN LM.

\rnnTuring*

\begin{proof}
    We rely mainly on existing results by \citet{nowak-etal-2023-representational} but use the additional information afforded by the augmented output alphabet, resulting in a simpler construction and interpretation.
    By \citet[Prop. 3.1 and Thm. 3.1]{nowak-etal-2023-representational}, the LM families induced by PTMs and probabilistic \twopda are weakly equivalent, so it suffices to show that a CoT RNN LM can encode any given \twopda.

    We first reiterate their main result together with its condition.
    \begin{definition}
        A \twopda $\pda$ is called \defn{$\alphabet$-deterministic} if, for any current state $\stateq$ any top symbol on its first stack $\stacksym$ and any output symbol from $\eosalphabet_\eps$, there is at most one transition with non-zero weight.
    \end{definition}
    \begin{theorem}
        \citet[Thm 3.2]{nowak-etal-2023-representational}
        Every $\alphabet$-deterministic probabilistic $\twopda$ can be encoded in an RNN LM that can output empty tokens $\eps$.
    \end{theorem}
    \citet[Thm. 3.2]{nowak-etal-2023-representational} shows that an RNN LM with output alphabet $\epsalphabet$ can simulate any probabilistic \twopda over $\alphabet$ if it is $\alphabet$-deterministic.
    Let $\pda$ be an arbitrary probabilistic \twopda, with $\pda = \left(\states, \alphabet, \stackalphabet, \trans, \qinit, \qfinal\right)$.
    Now define another \twopda $\pda=(\states,\outalphabet,\stackalphabet,\trans^\prime,\qinit,\qfinal)$, which generates outputs over the extended alphabet $\outalphabet \defeq \states \times \tapealphabet_\eps^4 \times \alphabet_\eps$.
    Define the transitions, add the following transitions to $\trans^\prime$ in $\pda^\prime$:
    \begin{equation}
         \trans^\prime \defeq \set{\twoPdaEdgenoweight{q}{\stacksym}{(\stateq, \stacksym_1,\stacksym_2,\stacksym_3,\stacksym_4, \sym)}{\stateq^\prime}{\stacksym_1}{\stacksym_2}{\stacksym_3/w}{\stacksym_4} \mid  \twoPdaEdgenoweight{q}{\stacksym}{\sym}{\stateq^\prime}{\stacksym_1}{\stacksym_2}{\stacksym_3/w}{\stacksym_4} \in \trans} 
    \end{equation}
    
    The constructed $\pda^\prime$ satisfies the required $\alphabet$-determinism condition, meaning that we can construct a weakly equivalent RNN LM with $\eps$ outputs.
    Finally, define the following regular function $\regfn\colon\outalphabet_\eps\to\alphabet_\eps$ that removes all additional information post-hoc:
    \begin{equation} \label{eq:rnn-cot-regfn}
        \regfn(\symx) \defeq \begin{cases}
            \symy & \ifcondition \symx = (\stateq, \stacksym_1,\stacksym_2,\stacksym_3,\stacksym_4, \sym)           \\
            \eps  & \ifcondition \symx = \eps
        \end{cases}
    \end{equation}
    The RNN LM therefore directly simulates runs from the \twopda $\pda^\prime$.
    Applying $\regfn$ to the outputs of the RNN LM results in strings in $\kleene{\alphabet}$.
    Since the transformation groups all execution runs that output $\str \in \kleene{\alphabet}$ and the resulting CoT-augmented LM $\pLMcot$ (cf. \cref{def:cot-lm}) sums over them, we are left with a weakly equivalent LM.
\end{proof}

\transformerTuring*
\begin{proof}
    Here, we adapt the construction by \citet{perez-etal-2021-attention} to the decoder-only case with prompt length (denoted by $n$ in \citet{perez-etal-2021-attention}) zero since we only care about \emph{generating} strings from the CoT-augmented transformer LM.
    Moreover, rather than implementing a deterministic transition function of a Turing machine (which is simulated by a particular layer of their transformer), we implement the probabilistic transition function of a rational-valued PTM in the \emph{sampling} step of the transformer LM.
    As in our previous constructions, this step endows the model with non-determinism.

    At a high level, we will construct a CoT-augmented transformer LM over $\kleene{\alphabet}$ that will output symbols from the augmented alphabet $\outalphabet \defeq \states \times \tapealphabet \times \overline{\bosalphabet}_\eps \times \actions \times \actions$.
    Here, $\actions \defeq \set{-1, 0, 1}$, where we will for conciseness identify the action $\texttt{LEFT}$ with $-1$ and the action $\texttt{RIGHT}$ with $1$.
    Note that we, like \citet{perez-etal-2021-attention}, \emph{do not} allow the action $\texttt{NOOP}$ in our construction, but we include the value $0$ in the action set as it will be useful for the starting conditions of the constructed transformer.
    As explained later, such an alphabet $\outalphabet$ contains enough information for the transformer LM to be able to reconstruct the configuration of the PTM at every time step, and thus match its conditional probability distributions.

    \paragraph{Notation.}
    Let $\tm = \qptmtuple$ be a $\qptm$ and $\pLM$ the LM over $\kleene{\alphabet}$ it induces.
    In the following, we denote with $\stateq_\tstep \in \states$ the state of the PTM, with $v_\tstep \in \tapealphabet$ the symbol written to the working tape, with $a_\tstep$ the action performed, with $s_\tstep$ the symbol read, and with $\sym_\tstep$ the symbol written to the output tape, all at time step $\tstep$.

    \paragraph{High-level idea of the construction.}
    Before formally describing the components of the CoT-augmented transformer LM in separate lemmata, we give a high-level overview of the construction.
    The two-layer (single-head) transformer will, when computing $\enc\left(\cotstrlt\right)$ for $\cotstrlt \in \kleene{\outalphabet}$, perform the following computations, very similar to those described by \citet{perez-etal-2021-attention}:
    \begin{enumerate}
        \item \textbf{Input representations} ($\staticRepr$; \Cref{lem:input-representations}): The input representation function represents the input symbols $\cotsym \in \outalphabet$ with a multi-hot encoding (one that contains an individual one-hot encoding of each of the components $\stateq_\tstep, v_{\tstep - 1}, \sym_{\tstep - 1}, a_{\tstep - 1}, a_{\tstep - 2}$) of the form\footnote{In the following, \textcolor{ETHGreen}{green} color denotes the components that are added or computed by each of the described components.}\textsuperscript{,}\footnote{Notice that the symbol $\sym \in \alphabet$ is not part of the internal representation, since it is not required for the simulation of the PTM. It is only used to construct the final output string stored in the output tape.}
              \begin{equation}
                  \posInEmbeddingFun{\cotsym, \tstep} = \begin{pmatrix}
                      \textcolor{ETHGreen}{\onehot{\stateq_\tstep}}                              \\
                      \textcolor{ETHGreen}{\onehot{v_{\tstep - 1}} }                             \\
                      \textcolor{ETHGreen}{\onehot{a_{\tstep - 1}} }                             \\
                      \textcolor{ETHGreen}{a_{\tstep - 1}}, \textcolor{ETHGreen}{a_{\tstep - 2}} \\
                      0, 0                                                                       \\
                      0, \zero_{|\tapealphabet|}                                                 \\
                      \textcolor{ETHGreen}{1, \tstep + 1, \frac{1}{\tstep + 1}, \frac{1}{\left(\tstep + 1\right)^2}}
                  \end{pmatrix}
              \end{equation}
              for $\tstep \in \Nzero$.
              The representations also include additional components (the ``empty'' zero values in the vector, explained below) used for processing and the positional encoding of the time step, $1, \tstep + 1, \frac{1}{\tstep + 1}, \frac{1}{\left(\tstep + 1\right)^2}$.
              This component is explained in more detail in \Cref{lem:input-representations}.
        \item \textbf{Layer 1} ($\tflayer_1$; \Cref{lem:layer-1}): The first layer uses the information about the actions performed at all previous time steps (contained in the input static representations of the CoT-augmented symbols) to compute the locations of $\tm$'s head at each time step.
              This results in the internal representations of the form
              \begin{equation}
                  \vx^1_\tstep = \begin{pmatrix}
                      \onehot{\stateq_\tstep}                                                                                              \\
                      \onehot{v_{\tstep - 1}}                                                                                              \\
                      \onehot{a_{\tstep - 1}}                                                                                              \\
                      a_{\tstep - 1}, a_{\tstep - 2}                                                                                       \\
                      \textcolor{ETHGreen}{\frac{\posAt{\tstep}}{\tstep + 1}}, \textcolor{ETHGreen}{\frac{\posAt{\tstep - 1}}{\tstep + 1}} \\
                      0, \zero_{|\tapealphabet|}                                                                                           \\
                      1, \tstep + 1, \frac{1}{\tstep + 1}, \frac{1}{\left(\tstep + 1\right)^2}
                  \end{pmatrix}
              \end{equation}
              for $\tstep \in \Nzero$, where $\posAt{\tstep}$ denotes the position of $\tm$'s head at time $\tstep$.
              This component is identical to the \emph{second} layer of the transformer from \citet{perez-etal-2021-attention} and is explained in more detail in \Cref{lem:layer-1}.
        \item \textbf{Layer 2} ($\tflayer_2$; \Cref{lem:layer-2}): Uses the $\tflayer_1$-computed information about the head locations at each time step to (almost) compute $s_\tstep$---the symbol read by the head of the PTM at time step $\tstep$.
              This results in the internal representations of the form
              \begin{equation}
                  \vx^2_\tstep = \begin{pmatrix}
                      \onehot{\stateq_\tstep}                                                                \\
                      \onehot{v_{\tstep - 1}}                                                                \\
                      \onehot{a_{\tstep - 1}}                                                                \\
                      a_{\tstep - 1}, a_{\tstep - 2}                                                         \\
                      \frac{\posAt{\tstep}}{\tstep + 1}, \frac{\posAt{\tstep - 1}}{\tstep + 1}               \\
                      \textcolor{ETHGreen}{\lastVisit{\tstep}}, \textcolor{ETHGreen}{\onehot{v_{\lastVisit{\tstep}}}} \\
                      1, \tstep + 1, \frac{1}{\tstep + 1}, \frac{1}{\left(\tstep + 1\right)^2}
                  \end{pmatrix}
              \end{equation}
              for $\tstep \in \Nzero$, where $\lastVisit{\tstep}$ denotes the last time step when $\tm$'s head wrote to $\posAt{\tstep}$.
              This component is analogous to the \emph{third} layer of the transformer from \citet{perez-etal-2021-attention} and is explained in more detail in \Cref{lem:layer-2}.
        \item \textbf{Output function} ($\fTransf$; \Cref{lem:output-function}): The function $\fTransf$ uses the information computed by the two layers of the transformer to compute the one-hot encoding of the current configuration of the PTM.
              In particular, this includes $\tm$'s current state ($\stateq_\tstep$), the symbol read by $\tm$'s head ($s_\tstep$), and, for reasons that we explain shortly, the action performed by $\tm$ at the previous time step ($a_{\tstep - 1}$).
              This results in the representation
              \begin{equation}
                  \enc\left(\cotstrlt\right) = \onehot{\stateq_\tstep, s_\tstep, a_{\tstep - 1}},
              \end{equation}
              which is used for sampling the next transition of the PTM.
              This component resembles the output function of \citet{perez-etal-2021-attention} and is explained in more detail in \Cref{lem:output-function}.
        \item \textbf{Sampling} (\Cref{lem:sampling}): The representation $\onehot{\stateq_\tstep, s_\tstep, a_{\tstep - 1}}$ is used to index the output matrix $\outMtx$ that contains the conditional probabilities $\pLM\left(\cdot\mid \stateq_\tstep, s_\tstep\right)$.
              This can be used to sample the next augmented symbol $\onehot{\stateq^{\tstep+1}, v_{\tstep}, \eossym_\tstep, a_\tstep, a_{\tstep - 1}}$ (here, the last component, $a_{\tstep - 1}$, equals the ``input'' to the sampling step).
              This component is explained in more detail in \Cref{lem:sampling}.
    \end{enumerate}
    The idea of the construction---iteratively computing and storing the modifications to the PTM configuration in the generated string $\cotstr$---is therefore identical to \citeposs{perez-etal-2021-attention} one.
    In particular, the two components that perform the bulk of the simulation---Layer 1 and Layer 2---are identical to the components from \citet{perez-etal-2021-attention}.
    We describe and show the correctness of the components in the lemmata in the rest of the section.
    The correctness of the construction follows from the correctness of the components.

    \paragraph{Weak equivalence.}
    To define a CoT-augmented LM with such a transformer, we can define a transducer that projects the outputs in $\kleene{\left(\states \times \tapealphabet \times \overline{\bosalphabet}_\eps \times \actions \times \actions\right)}$ onto $\kleene{\alphabet}$ in the standard way by retaining only the outputs in $\overline{\bosalphabet}_\eps$.
    This collapses all the executions of the transformer LM (which simulates the executions of the PTM with the same probabilities, as per \Cref{lem:sampling}) yielding the same string in $\kleene{\alphabet}$, resulting in a CoT-augmented transformer LM weakly equivalent to $\pLM$.

    \paragraph{Precision.}
    The constructed transformer computes and stores position-dependent values at several points during the computation.
    The precision required for the representation of these values grows logarithmically with the number of computational steps.
    However, since the number of computational steps performed by a PTM generating a string is potentially unbounded in the length of the string \citep[p. 339]{hopcroft01}, the transformer's precision is subsequently unbounded as well \citep{nowak-etal-2023-representational}.
\end{proof}

We begin with a general lemma about the disjunction of one-hot encodings.
\begin{lemma} \label{lem:disjunction}
    Let $\sS_1, \ldots, \sS_n$ be finite sets and let $\onehot{s_1, \ldots, s_n} \in \set{0, 1}^{|\sS_1| \cdots |\sS_n|}$ denote the one-hot encoding of the tuple $\left(s_1, \ldots, s_n\right) \in \sS_1 \times \cdots \times \sS_n$.
    Define the matrices
    \begin{equation}
        \mW^{\sS_i} \in \set{0, 1}^{|\sS_i| \times |\sS_1| \cdots |\sS_n|}
    \end{equation}
    element-wise as
    \begin{equation} \label{eq:disjunction-matrix-def}
        \mW^{\sS_i}_{\textcolor{ETHRed}{s}, \left(s_1, \ldots, s_{i-1}, \textcolor{ETHRed}{s}, s_{i + 1}, \ldots, s_n \right)} \defeq 1 \qquad \text{ for all } s_j \in \sS_j, j\in1,\ldots, n, j \neq i.
    \end{equation}
    Then, it holds that
    \begin{equation}
        \mW^{\sS_i} \onehot{s_1, \ldots, s_n} = \onehot{s_i}.
    \end{equation}
\end{lemma}
\begin{proof}
    \cref{eq:disjunction-matrix-def} can equivalently be written as
    \begin{equation}
        \mW^{\sS_i}_{\textcolor{ETHRed}{s}, \left(s_1, \ldots, s_{i-1}, s_i, s_{i + 1}, \ldots, s_n \right)} \defeq \ind{s_i = s}.
    \end{equation}
    Indexing the elements of $\mW^{\sS_i} \onehot{s_1, \ldots, s_n}$ directly with $s \in \sS_i$, we then compute
    \begin{subequations}
        \begin{align}
            \left(\mW^{\sS_i} \onehot{s_1, \ldots, s_n}\right)_s
             & = \sum_{\substack{s^\prime_j \in \sS_j                                                  \\ j = 1, \ldots, n}} \mW^{\sS_i}_{s, \left(s^\prime_1, \ldots, s^\prime_n \right)} \onehot{s_1, \ldots, s_n}_{s^\prime_1, \ldots, s^\prime_n} \\
             & = \sum_{\substack{s^\prime_j \in \sS_j                                                  \\ j = 1, \ldots, n}} \ind{s^\prime_i = s} \onehot{s_1, \ldots, s_n}_{s^\prime_1, \ldots, s^\prime_n} \\
             & = \sum_{\substack{s^\prime_j \in \sS_j                                                  \\ j = 1, \ldots, n}} \ind{s^\prime_i = s} \ind{s^\prime_1 = s_1, \ldots, s^\prime_n = s_n} \\
             & = \sum_{s^\prime_i \in \sS_i} \ind{s^\prime_i = s} \underbrace{\sum_{\substack{s^\prime_j \in \sS_j \\ j = 1, \ldots, n, \; j \neq i}} \ind{s^\prime_1 = s_1, \ldots, s^\prime_n = s_n}}_{=1} \label{eq:equals-1-equality} \\
             & = \sum_{s^\prime_i \in \sS_i} \ind{s^\prime_i = s},
        \end{align}
    \end{subequations}
    which is the definition of the elements of $\onehot{s_i}$.
    The equality in \cref{eq:equals-1-equality} follows from the fact that the summand is non-zero exactly when $s^\prime_j = s_j$ for all $j \neq i$.
\end{proof}

\begin{lemma}[Input representations] \label{lem:input-representations}
    Define the following static representation function of the CoT-augmented symbols $\cotsym \defeq \left(\stateq, v, \overline{\bossym}, a^\prime, a\right) \in \outalphabet$:
    \begin{equation}
        \posInEmbeddingFun{\cotsym, \tstep} \defeq \begin{pmatrix}
            \mW \onehot{\stateq, v, \overline{\bossym}, a^\prime, a} \\
            0, 0                                 \\
            0, \zero_{|\tapealphabet|}           \\
            1, \tstep + 1, \frac{1}{\tstep + 1}, \frac{1}{\left(\tstep + 1\right)^2}
        \end{pmatrix} \in \R^{\hiddDim}
    \end{equation}
    where
    $\onehot{\stateq, v, \overline{\bossym}, a^\prime, a} \in \set{0, 1}^{\nstates |\stackalphabet| |\overline{\bosalphabet}_\eps| |\actions| |\actions|}$ denotes the one-hot encoding of the tuple $\left(\stateq, v, \overline{\bossym}, a^\prime, a\right)$,
    \begin{equation}
        \mW \defeq \left(\mW^\states; \mW^{\tapealphabet}; \mW^{\actions}; \mA^\prime; \mA\right) \in \R^{\hiddDim \times \left(\nstates |\stackalphabet| |\overline{\bosalphabet}_\eps| |\actions| |\actions|\right)}
    \end{equation}
    for matrices $\mW^\states, \mW^{\tapealphabet}, \mW^{\actions}$ from \Cref{lem:disjunction}, $\mA^\prime, \mA \in \R^{1 \times \hiddDim}$ defined as\footnote{Recall that we can identify the actions of the PTM with the integers $-1$, $0$, and $1$.}
    \begin{subequations}
        \begin{align}
            \mA^\prime_{1, \left(\stateq, v, \overline{\bossym}, a^\prime, a\right)} & \defeq a^\prime, \qquad q \in \states, v \in \tapealphabet, \overline{\bossym} \in \overline{\bosalphabet}_\eps, a, a^\prime \in A  \\
            \mA_{1, \left(\stateq, v, \overline{\bossym}, a^\prime, a\right)}  & \defeq a, \qquad q \in \states, v \in \tapealphabet, \overline{\bossym} \in \overline{\bosalphabet}_\eps, a, a^\prime \in A,
        \end{align}
    \end{subequations}
    and $\hiddDim \defeq \left[\nstates + |\stackalphabet| + |\actions| + 2\right] + 2 + \left[1 + |\stackalphabet|\right] + 4$.
    Then, it holds that
    \begin{equation}
        \posInEmbeddingFun{\cotsym, \tstep} = \begin{pmatrix}
            \onehot{\stateq_\tstep}    \\
            \onehot{v}                 \\
            \onehot{a}                 \\
            a^\prime, a                      \\
            0, 0                       \\
            0, \zero_{|\tapealphabet|} \\
            1, \tstep + 1, \frac{1}{\tstep + 1}, \frac{1}{\left(\tstep + 1\right)^2}
        \end{pmatrix}
    \end{equation}
\end{lemma}
\begin{proof}
    The first part (computing the one-hot encodings of $\stateeq$, $v$, and $a^\prime$) follows from the construction of the matrices $\mW^\states, \mW^{\tapealphabet}$, and $\mW^{\actions}$ from \Cref{lem:disjunction}.
    The second part (computing the values $a^\prime$ and $a$ with the matrices $\mA^\prime$ and $\mA$) follows from the definition of the matrices $\mA$ and $\mA^\prime$: By construction, the values of $a^\prime$ and $a$ will be copied from the one-hot encodings into the resulting entry of the vector.
\end{proof}

We now describe the two layers of the transformer that compute the quantities required to be able to determine the current configuration of the PTM at each time step.
Here, we heavily rely on the construction by \citet{perez-etal-2021-attention}.

Let $\tm = \qptmtuple$ be a rationally-weighted PTM.
Let $\tstep \in \N$ and define
\begin{equation} \label{eq:def-c}
    \posAt{\tstep} \defeq \sum_{\idxj = 0}^{\tstep - 1} a_\idxj
\end{equation}
as the position of the head of the PTM at time $\tstep$.
Further, define the set
\begin{equation}
    \visitedAt{\tstep} \defeq \set{\idxj \mid \posAt{\idxj} = \posAt{\tstep}, \idxj \in \NTo{\tstep - 1}}.
\end{equation}
$\visitedAt{\tstep}$ contains the time steps at which the PTM $\tm$ so far visited (and thus wrote to) the tape cell read by $\tm$ at time $\tstep$.
Then, define $\lastVisit{\tstep}$ as
\begin{equation} \label{eq:l-def}
    \lastVisit{\tstep} \defeq \begin{cases}
        \max{\visitedAt{\tstep}} & \ifcondition |\visitedAt{\tstep}| > 0 \\
        \tstep                   & \otherwisecondition.
    \end{cases}
\end{equation}
In words, $\lastVisit{\tstep}$ denotes the time step at which the PTM $\tm$ last visited (and thus wrote on) the tape cell read by $\tm$ at time $\tstep$ if this cell was visited yet.
Otherwise, $\lastVisit{\tstep}$ equals $\tstep$.

Now, define the $\bos$ symbol over the augmented alphabet $\stackalphabet$ as
\begin{equation}
    \cotbos \defeq \left(\stateq_0, \bot, \bos, 0, 0\right) \in \states \times \tapealphabet \times \overline{\bosalphabet}_\eps \times \actions \times \actions.
\end{equation}
Let $\qinit, \stateq_1, \ldots, \stateq_\strlen$, $\bot, v_1, \ldots, v_\strlen$, $\bos, \sym_1, \ldots, \sym_\strlen$, and $a_0, a_1, \ldots, a_\strlen$ be the sequences of states visited, symbols written on the processing tape, symbols written on the output tape, and actions performed by $\tm$ in the first $\strlen$ steps.
Define $a_{-1} = a_0 \defeq 0$,
\begin{equation}
    \vx^0_\tstep \defeq
    \begin{cases}
        \posInEmbeddingFun{\cotbos, \tstep}                                                                            & \ifcondition \tstep = 0 \\
        \posInEmbeddingFun{\left(\stateq_\tstep, v_\tstep, \sym_\tstep, a_{\tstep - 1}, a_{\tstep - 2}\right), \tstep} & \otherwisecondition,
    \end{cases}
\end{equation}
and
\begin{equation}
    \mX^0 \defeq \left(\vx^0_0; \vx^0_1; \cdots; \vx^0_\strlen\right).
\end{equation}
These will be the inputs to the transformer layers and thus,\footnote{Note that since the transformer over the augmented alphabet is deterministic, this representation is unique.}
\begin{equation}
    \staticRepr\left(\cotstrlt\right) \defeq \left(\vx^0_0; \vx^0_1; \cdots; \vx^0_{\tstep - 1}\right).
\end{equation}

With this, we can describe and prove the correctness of the first layer of the transformer we are building.
\begin{lemma}[Layer 1] \label{lem:layer-1}
    There exists a transformer layer $\tflayer_1$ that, given the inputs $\mX^0$, computes the values $\frac{\posAt{\tstep}}{\tstep + 1}$ and $\frac{\posAt{\tstep - 1}}{\tstep + 1}$ for all $\tstep = 1, \ldots, \strlen$.
    More precisely, denoting
    \begin{equation}
        \mX^1 \defeq \left(\vx^1_0; \vx^1_1; \cdots; \vx^1_\strlen\right) = \tflayer_1\left(\mX^0\right),
    \end{equation}
    it holds that $\vx^1_\tstep$ contains the entries containing the values $\frac{\posAt{\tstep}}{\tstep + 1}$ and $\frac{\posAt{\tstep - 1}}{\tstep + 1}$:
    \begin{equation}\label{eq:cs-result}
        \vx^1_\tstep = \begin{pmatrix}
            \onehot{\stateq_\tstep}                                                                                              \\
            \onehot{v_{\tstep - 1}}                                                                                              \\
            \onehot{a_{\tstep - 1}}                                                                                              \\
            a_{\tstep - 1}, a_{\tstep - 2}                                                                                       \\
            \textcolor{ETHGreen}{\frac{\posAt{\tstep}}{\tstep + 1}}, \textcolor{ETHGreen}{\frac{\posAt{\tstep - 1}}{\tstep + 1}} \\
            0, \zero_{|\tapealphabet|}                                                                                           \\
            1, \tstep + 1, \frac{1}{\tstep + 1}, \frac{1}{\left(\tstep + 1\right)^2}
        \end{pmatrix}
    \end{equation}
\end{lemma}
\begin{proof}
    This follows from \citet[Lemma 9]{perez-etal-2021-attention}, which we summarize here.\footnote{More precisely, since their construction stores the actions $a_\tstep$ and $a_{\tstep - 1}$ (the former is possible because the action is not sampled but deterministically computed based on the configuration of the PTM), their layer computes the values $\frac{\posAt{\tstep}}{\tstep + 1}$ and $\frac{\posAt{\tstep}}{\tstep + 1}$. As we show later in \Cref{lem:layer-2}, this does not affect the correctness of the construction.}
    Concretely, $\tflayer_1$ is implemented by a transformer layer with trivial (zero-valued) query and key transformations and a value transformation $\vTransf$ that copies the values of the actions $a_{\tstep - 1}$ and $a_{\tstep - 2}$ to the entry that will hold the values $\frac{\posAt{\tstep}}{\tstep + 1}$ and $\frac{\posAt{\tstep - 1}}{\tstep + 1}$.
    Since the current head location $\posAt{\tstep}$ is simply the sum of those values (cf. \cref{eq:def-c}), attending to all previous positions results in \cref{eq:cs-result}.

    Formally, we define
    \begin{subequations}
        \begin{align}
            \mQ^1                                & \defeq \zero_{\hiddDim \times \hiddDim}, \quad \mK^1 \defeq \zero_{\hiddDim \times \hiddDim}, \\
            \tfscorefun\left(\vq, \vk\right)     & \defeq \innerProd{\vq}{\vk}                                                                   \\
            \mV_{\idxn^\prime, :}^1 = \mV_{\idxn, :}^1 & = \begin{pmatrix}
                                                         \zero_\nstates             \\
                                                         \zero_{|\tapealphabet|}    \\
                                                         \zero_{|\actions|}         \\
                                                         0, 0                       \\
                                                         1, 1                       \\
                                                         0, \zero_{|\tapealphabet|} \\
                                                         0, 0, 0, 0
                                                     \end{pmatrix}^\top                                                                  \\
            \mO^1                                & = \zero_{\hiddDim \times \hiddDim}
        \end{align}
    \end{subequations}
    Here, $\idxn^\prime$ and $\idxn$ refer to the indices of the rows at which the values $\frac{\posAt{\tstep}}{\tstep + 1}$ and $\frac{\posAt{\tstep - 1}}{\tstep + 1}$ will be stored (that is, the rows below the values of $a_{\tstep - 1}$ and $a_{\tstep - 2}$).
    All other rows of $\mV^1$ are zero.

    Then, for $\tstep \in \Nzero$, it holds that
    \begin{equation}
        \tfscorefun\left(\vq_\tstep, \vk_\idxj\right) = \tfscorefun\left(\qTransf\left(\vx^0_\tstep\right), \kTransf\left(\vx^0_\idxj\right)\right) = \tfscorefun\left(\mQ^1 \vx^0_\tstep, \mK^1 \vx^0_\idxj\right) = \innerProd{\zero}{\zero} = 0
    \end{equation}
    for all $\idxj \leq \tstep$, resulting in $\vs_\tstep = \hardmax\left(\zero\right) = \frac{1}{\tstep + 1} \one_\tstep$, and
    \begin{equation}
        \vTransf\left(\vx^0_\idxj\right) = \begin{pmatrix}
            \zero_\nstates               \\
            \zero_{|\tapealphabet|}      \\
            \zero_{|\actions|}           \\
            0, 0                         \\
            a_{\idxj - 1}, a_{\idxj - 2} \\
            0, \zero_{|\tapealphabet|}   \\
            0, 0, 0, 0
        \end{pmatrix}.
    \end{equation}
    This results in
    \begin{subequations}
        \begin{align}
            \attn\left(\vq_\tstep, \mK_\tstep, \mV_\tstep\right)
             & = \sum_{\idxj = 0}^{\tstep} \evs_\idxj \vTransf\left(\vx^0_\idxj\right)                  \\
             & = \sum_{\idxj = 0}^{\tstep} \frac{1}{\tstep + 1} \vTransf\left(\vx^0_\idxj\right)        \\
             & = \frac{1}{\tstep + 1} \sum_{\idxj = 0}^{\tstep} \vTransf\left(\vx^0_\idxj\right)        \\
             & = \frac{1}{\tstep + 1} \sum_{\idxj = 0}^{\tstep} \begin{pmatrix}
                                                                    \zero_\nstates               \\
                                                                    \zero_{|\tapealphabet|}      \\
                                                                    \zero_{|\actions|}           \\
                                                                    0, 0                         \\
                                                                    a_{\idxj - 1}, a_{\idxj - 2} \\
                                                                    0, \zero_{|\tapealphabet|}   \\
                                                                    0, 0, 0, 0
                                                                \end{pmatrix}            \\
             & = \frac{1}{\tstep + 1} \begin{pmatrix}
                                          \zero_\nstates                     \\
                                          \zero_{|\tapealphabet|}            \\
                                          \zero_{|\actions|}                 \\
                                          0, 0                               \\
                                          \posAt{\tstep}, \posAt{\tstep - 1} \\
                                          0, \zero_{|\tapealphabet|}         \\
                                          0, 0, 0, 0
                                      \end{pmatrix}                                \\
             & = \begin{pmatrix}
                     \zero_\nstates                                                           \\
                     \zero_{|\tapealphabet|}                                                  \\
                     \zero_{|\actions|}                                                       \\
                     0, 0                                                                     \\
                     \frac{\posAt{\tstep}}{\tstep + 1}, \frac{\posAt{\tstep - 1}}{\tstep + 1} \\
                     0, \zero_{|\tapealphabet|}                                               \\
                     0, 0, 0, 0
                 \end{pmatrix}.
        \end{align}
    \end{subequations}
    Furthermore, we have that
    \begin{align}
        \va_\tstep   & = \attn\left(\vq_\tstep, \mK_\tstep, \mV_\tstep\right) + \vx^0_\tstep = \begin{pmatrix}
                                                                                                   \zero_\nstates                                                           \\
                                                                                                   \zero_{|\tapealphabet|}                                                  \\
                                                                                                   \zero_{|\actions|}                                                       \\
                                                                                                   0, 0                                                                     \\
                                                                                                   \frac{\posAt{\tstep}}{\tstep + 1}, \frac{\posAt{\tstep - 1}}{\tstep + 1} \\
                                                                                                   0, \zero_{|\tapealphabet|}                                               \\
                                                                                                   0, 0, 0, 0
                                                                                               \end{pmatrix} + \begin{pmatrix}
                                                                                                                   \onehot{\stateq_\tstep}        \\
                                                                                                                   \onehot{v_{\tstep - 1}}        \\
                                                                                                                   \onehot{a_{\tstep - 1}}        \\
                                                                                                                   a_{\tstep - 1}, a_{\tstep - 2} \\
                                                                                                                   0, 0                           \\
                                                                                                                   0, \zero_{|\tapealphabet|}     \\
                                                                                                                   1, \tstep + 1, \frac{1}{\tstep + 1}, \frac{1}{\left(\tstep + 1\right)^2}
                                                                                                               \end{pmatrix}  = \begin{pmatrix}
                                                                                                                                    \onehot{\stateq_\tstep}                                                  \\
                                                                                                                                    \onehot{v_{\tstep - 1}}                                                  \\
                                                                                                                                    \onehot{a_{\tstep - 1}}                                                  \\
                                                                                                                                    a_{\tstep - 1}, a_{\tstep - 2}                                           \\
                                                                                                                                    \frac{\posAt{\tstep}}{\tstep + 1}, \frac{\posAt{\tstep - 1}}{\tstep + 1} \\
                                                                                                                                    0, \zero_{|\tapealphabet|}                                               \\
                                                                                                                                    1, \tstep + 1, \frac{1}{\tstep + 1}, \frac{1}{\left(\tstep + 1\right)^2}
                                                                                                                                \end{pmatrix} \\
        \vx^1_\tstep & = \oTransf\left(\va_\tstep\right) + \va_\tstep = \zero_\hiddDim + \begin{pmatrix}
                                                                                             \onehot{\stateq_\tstep}                                                  \\
                                                                                             \onehot{v_{\tstep - 1}}                                                  \\
                                                                                             \onehot{a_{\tstep - 1}}                                                  \\
                                                                                             a_{\tstep - 1}, a_{\tstep - 2}                                           \\
                                                                                             \frac{\posAt{\tstep}}{\tstep + 1}, \frac{\posAt{\tstep - 1}}{\tstep + 1} \\
                                                                                             0, \zero_{|\tapealphabet|}                                               \\
                                                                                             1, \tstep + 1, \frac{1}{\tstep + 1}, \frac{1}{\left(\tstep + 1\right)^2}
                                                                                         \end{pmatrix} = \begin{pmatrix}
                                                                                                             \onehot{\stateq_\tstep}                                                  \\
                                                                                                             \onehot{v_{\tstep - 1}}                                                  \\
                                                                                                             \onehot{a_{\tstep - 1}}                                                  \\
                                                                                                             a_{\tstep - 1}, a_{\tstep - 2}                                           \\
                                                                                                             \frac{\posAt{\tstep}}{\tstep + 1}, \frac{\posAt{\tstep - 1}}{\tstep + 1} \\
                                                                                                             0, \zero_{|\tapealphabet|}                                               \\
                                                                                                             1, \tstep + 1, \frac{1}{\tstep + 1}, \frac{1}{\left(\tstep + 1\right)^2}
                                                                                                         \end{pmatrix},
    \end{align}
    which is what we needed to show.
\end{proof}

\begin{lemma}[Layer 2] \label{lem:layer-2}
    There exists a transformer layer $\tflayer_2$ that, given the outputs $\mX^1$ of $\tflayer_1$ from \Cref{lem:layer-1}, computes the values $\lastVisit{\tstep}$ and $\onehot{v_{\lastVisit{\tstep}}}$ for all $\tstep = 1, \ldots, \strlen$.
    More precisely, denoting
    \begin{equation}
        \mX^2 \defeq \left(\vx^2_0; \vx^2_1; \cdots; \vx^2_\strlen\right) = \tflayer_2\left(\mX^1\right),
    \end{equation}
    it holds that $\vx^2_\tstep$ contains the entries containing the values $\lastVisit{\tstep} + 1$ and $\onehot{v_{\lastVisit{\tstep}}}$:
    \begin{equation} \label{eq:ls-result}
        \vx^2_\tstep = \begin{pmatrix}
            \onehot{\stateq_\tstep}                                                                             \\
            \onehot{v_{\tstep - 1}}                                                                             \\
            \onehot{a_{\tstep - 1}}                                                                             \\
            a_{\tstep - 1}, a_{\tstep - 2}                                                                      \\
            \frac{\posAt{\tstep}}{\tstep + 1}, \frac{\posAt{\tstep - 1}}{\tstep + 1}                            \\
            \textcolor{ETHGreen}{\lastVisit{\tstep} + 1}, \textcolor{ETHGreen}{\onehot{v_{\lastVisit{\tstep}}}} \\
            1, \tstep + 1, \frac{1}{\tstep + 1}, \frac{1}{\left(\tstep + 1\right)^2}
        \end{pmatrix}.
    \end{equation}
\end{lemma}
\begin{proof}
    This follows from \citet[Lemma 10]{perez-etal-2021-attention}.
    The idea of the construction is for the self-attention mechanism at time $\tstep$ to attend to exactly the entry from time step $\lastVisit{\tstep} + 1$ (i.e., to compute the query and key vectors such that $\argmax\left(\vs_\tstep\right) = \lastVisit{\tstep} + 1$) since that entry will contain the information about the symbol written at the time step before---at $\lastVisit{\tstep}$.
    Then, the values $\lastVisit{\tstep}$ and $v_{\lastVisit{\tstep}}$ are obtained by copying the corresponding values of the positional encoding and the written symbol from that time step.

    Formally, we define
    \begin{subequations}
        \begin{align}
            \mQ^2                            & \defeq
            \begin{blockarray}{ccccccc}
                & \cdots & \overbrace{\hspace{8mm}}^{\frac{\posAt{\tstep}}{\tstep + 1}} & \cdots & \overbrace{\hspace{8mm}}^{\frac{1}{\tstep + 1}} & \overbrace{\hspace{8mm}}^{\frac{1}{\left(\tstep + 1\right)^2}} & \cdots \\
                \begin{block}{(ccccccc)}
                    & & 1 & & & & \\
                    & & & & 1 & & \\
                    & & & & & \frac{1}{3} & \\
                \end{block}
            \end{blockarray} \\
            \mK^2                            & \defeq
            \begin{blockarray}{ccccccc}
                & \cdots & \overbrace{\hspace{8mm}}^{\frac{\posAt{\tstep - 1}}{\tstep + 1}} & \cdots & \overbrace{\hspace{8mm}}^{\frac{1}{\tstep + 1}} & \overbrace{\hspace{8mm}}^{\frac{1}{\left(\tstep + 1\right)^2}} & \cdots \\
                \begin{block}{(ccccccc)}
                    & & & & 1 & & \\
                    & & -1 & & & & \\
                    & & & & & \frac{1}{3} & \\
                \end{block}
            \end{blockarray} \\
            \tfscorefun\left(\vq, \vk\right) & \defeq -\abs{\innerProd{\vq}{\vk}}                                                                                                                                                                               \\
            \mV^2                            & \defeq
            \begin{blockarray}{cccccccc}
                & \cdots & \overbrace{\hspace{8mm}}^{\onehot{v_{\tstep - 1}}} & \cdots & & \overbrace{\hspace{8mm}}^{\tstep + 1} & \cdots \\
                \begin{block}{(ccccccc)c}
                    & & & & & & & \vdots \\
                    & & & & & 1 & & \lastVisit{\tstep} + 1\\
                    & & \mI_{|\tapealphabet|} & & & & & \onehot{v_{\lastVisit{\tstep}}} \\
                    & & & & & & & \vdots \\
                \end{block}
            \end{blockarray}                                                                                   \\
            \mO^1                            & = \zero_{\hiddDim \times \hiddDim}
        \end{align}
    \end{subequations}
    \citet[Lemma 10]{perez-etal-2021-attention} show that, given the parameters above, the output of the scoring function $\tfscorefun$ is maximized at the entry $\lastVisit{\tstep} \in \set{0, \ldots, \tstep}$.\footnote{More precisely, their construction results in the maximum being at $\lastVisit{\tstep + 1}$, since layer 1 in their construction, in contrast to ours, computes the values $\frac{\posAt{\tstep + 1}}{\tstep + 1}$ and $\frac{\posAt{\tstep}}{\tstep + 1}$, shifting all computations by one step. This is the consequence of the aforementioned difference between their deterministic and our probabilistic framework.}
    Since $\mV^2$ copies the second value from the positional encoding of the symbol at time step $\lastVisit{\tstep}$, which is $\lastVisit{\tstep} + 1$, this entry appears in the output of the attention mechanism.
    Furthermore, $\mV^2$ also copies the value $\onehot{v_{\lastVisit{\tstep}}}$ from the same entry.
    This, together with the zero-valued function $\oTransf$ and residual connections, results in \Cref{eq:ls-result}.
\end{proof}

\begin{lemma}[Correctness of the output function] \label{lem:output-function}
    There exists an MLP $\fTransf$ that, given the outputs $\mX^2$ of $\tflayer_2$, computes the one-hot encoding of the current configuration of the PTM.
    More concretely, it holds that
    \begin{equation}
        \fTransf\left(\vx^2_\tstep\right) = \onehot{\stateq_\tstep, s_\tstep, a_{\tstep - 1}}.
    \end{equation}
\end{lemma}
\begin{proof}
    Here, we define a function $\fTransf$ similar to that of \citet[Lemma 11]{perez-etal-2021-attention}, but with an additional layer that handles the addition of the $\bot$ symbol not handled by \citet{perez-etal-2021-attention}.
    The logic nonetheless remains the same: $\fTransf$ receives the output of $\tflayer_2$ of the form
    \begin{equation}
        \vx^2_\tstep = \begin{pmatrix}
            \onehot{\stateq_\tstep}                                                                             \\
            \onehot{v_{\tstep - 1}}                                                                             \\
            \onehot{a_{\tstep - 1}}                                                                             \\
            a_{\tstep - 1}, a_{\tstep - 2}                                                                      \\
            \frac{\posAt{\tstep}}{\tstep + 1}, \frac{\posAt{\tstep - 1}}{\tstep + 1}                            \\
            \textcolor{ETHGreen}{\lastVisit{\tstep} + 1}, \textcolor{ETHGreen}{\onehot{v_{\lastVisit{\tstep}}}} \\
            1, \tstep + 1, \frac{1}{\tstep + 1}, \frac{1}{\left(\tstep + 1\right)^2}
        \end{pmatrix}.
    \end{equation}
    and
    \begin{enumerate*}[label=\textit{(\arabic*)}]
        \item copies the value of $\stateq_\tstep$,
        \item copies the value of $a_{\tstep - 1}$
        \item compares the value of $\lastVisit{\tstep}$ to $\tstep$ to determine whether $s_\tstep = \blanksym$ or $s_\tstep = v_{\lastVisit{\tstep}}$,
        \item compares the value of $\tstep$ to $0$ to determine whether $s_\tstep = \bot$ or $s_\tstep = v_{\lastVisit{\tstep}}$.
    \end{enumerate*}

    Concretely, $\fTransf$ will take the form of a three-layer MLP
    \begin{equation}
        \fTransf\left(\vx\right) \defeq \ReLUfunc{\mW_5 \ReLUfunc{\mW_4 \ReLUfunc{\mW_3 \ReLUfunc{\mW_2 \ReLUfunc{\mW_1 \vx + \vb_1} + \vb_2} + \vb_3} + \vb_4} + \vb_5}
    \end{equation}
    where
    \begin{subequations}
        \begin{align}
            \mW^1 & \in \R^{\left(\nstates + |\tapealphabet| + |\actions| + |\tapealphabet| + |\tapealphabet| + 2 \right) \times \hiddDim}                                                                                      \\
            \mW^2 & \in \R^{\left(\nstates + |\tapealphabet| + |\actions| + |\tapealphabet| + |\tapealphabet| + 2 \right) \times \left(\nstates + |\tapealphabet| + |\actions| + |\tapealphabet| + |\tapealphabet| + 2 \right)} \\
            \mW^3 & \in \R^{\left(\nstates + |\tapealphabet| + |\actions| + |\tapealphabet| + |\tapealphabet| \right) \times \left(\nstates + |\tapealphabet| + |\actions| + |\tapealphabet| + |\tapealphabet| + 2 \right)}     \\
            \mW^4 & \in \R^{\left(\nstates + |\tapealphabet| + |\actions| \right) \times \left(\nstates + |\tapealphabet| + |\actions| + |\tapealphabet| + |\tapealphabet| \right)}                                             \\
            \mW^5 & \in \R^{\left(\nstates |\tapealphabet| |\actions| \right) \times \left(\nstates + |\tapealphabet| + |\actions| \right)}                                                                                     \\
        \end{align}
    \end{subequations}
    are the weights of the MLP and $\vb^1, \ldots, \vb^5$ are the corresponding biases.
    For conciseness, we will denote
    \begin{subequations}
        \begin{align}
            C_1 & \defeq \nstates              \\
            C_2 & \defeq C_1 + |\tapealphabet| \\
            C_3 & \defeq C_2 + |\actions|      \\
            C_4 & \defeq C_3 + |\tapealphabet| \\
            C_5 & \defeq C_4 + |\tapealphabet| \\
            C_6 & \defeq C_5 + 2
        \end{align}
    \end{subequations}
    and
    \begin{subequations}
        \begin{align}
            D_1 & \defeq \nstates                  \\
            D_2 & \defeq D_1 + |\tapealphabet|     \\
            D_3 & \defeq D_2 + |\actions|          \\
            D_4 & \defeq D_3 + 2                   \\
            D_5 & \defeq D_4 + 2                   \\
            D_6 & \defeq D_5 + 1 + |\tapealphabet| \\
            D_7 & \defeq D_6 + 4
        \end{align}
    \end{subequations}
    Then, we define
    \begin{center}
        \noindent\begin{minipage}{0.7\linewidth}
            \begin{equation}
                \mW^1 = \begin{blockarray}{cccccccc}
                    \overbrace{\hspace{8mm}}^{1: D_1} & \overbrace{\hspace{8mm}}^{D_1: D_2} & \overbrace{\hspace{8mm}}^{D_2: D_3} & \cdots & \overbrace{\hspace{8mm}}^{D_5 + 1} & \overbrace{\hspace{8mm}}^{D_6 + 1} & \overbrace{\hspace{8mm}}^{D_6 + 2} & \cdots \\
                    \begin{block}{(cccccccc)}
                        \mI_{\nstates} & & & & & & & \\
                        & \mI_{|\tapealphabet|} & & & & & & \\
                        & & \mI_{|\actions|} & & & & & \\
                        & & & \vphantom{\mI_{|\tapealphabet|}} & & & & \\
                        & & & & \vphantom{\mI_{|\tapealphabet|}} & & & \\
                        & & & & 1 & 1 & -1 & \\
                        & & & & & 2 & -1 & \\
                    \end{block}
                \end{blockarray}
            \end{equation}
        \end{minipage}
        \begin{minipage}{0.25\linewidth}
            \begin{equation}
                \vb^1 = \begin{blockarray}{c}
                    \vphantom{1} \\
                    \begin{block}{(c)}
                        \zero_\nstates \\
                        \zero_{|\tapealphabet|} \\
                        \zero_{|\actions|} \\
                        \onehot{\blanksym} \\
                        \onehot{\bot} \\
                        0 \\
                        0 \\
                    \end{block}
                \end{blockarray}
            \end{equation}
        \end{minipage}
    \end{center}
    \begin{center}
        \noindent\begin{minipage}{0.7\linewidth}
            \begin{equation}
                \mW^2 = \begin{blockarray}{ccccccc}
                    \overbrace{\hspace{8mm}}^{1: C_1} & \overbrace{\hspace{8mm}}^{C_1: C_2} & \overbrace{\hspace{8mm}}^{C_2: C_3} & \overbrace{\hspace{8mm}}^{C_3: C_4} & \overbrace{\hspace{8mm}}^{C_4: C_5} & \overbrace{\hspace{8mm}}^{C_5 + 1} & \overbrace{\hspace{8mm}}^{C_5 + 2} \\
                    \begin{block}{(ccccccc)}
                        \mI_{\nstates} & & & & & & \\
                        & \mI_{|\tapealphabet|} & & & & & \\
                        & & \mI_{|\actions|} & & & & \\
                        & & & \mI_{|\tapealphabet|} & & & \\
                        & & & & \mI_{|\tapealphabet|} & & \\
                        & & & & & 1 & -1 \\
                        & & & & & & 1 \\
                    \end{block}
                \end{blockarray}
            \end{equation}
        \end{minipage}
        \begin{minipage}{0.25\linewidth}
            \begin{equation}
                \vb^2 = \begin{blockarray}{c}
                    \vphantom{1} \\
                    \begin{block}{(c)}
                        \zero_\nstates \\
                        \zero_{|\tapealphabet|} \\
                        \zero_{|\actions|} \\
                        \zero_{|\tapealphabet|} \\
                        \zero_{|\tapealphabet|} \\
                        0 \\
                        0 \\
                    \end{block}
                \end{blockarray}
            \end{equation}
        \end{minipage}
    \end{center}
    \begin{center}
        \noindent\begin{minipage}{0.7\linewidth}
            \begin{equation}
                \mW^3 = \begin{blockarray}{cccccccc}
                    \overbrace{\hspace{8mm}}^{1: C_1} & \overbrace{\hspace{8mm}}^{C_1: C_2} & \overbrace{\hspace{8mm}}^{C_2: C_3} & \overbrace{\hspace{8mm}}^{C_3: C_4} & \overbrace{\hspace{8mm}}^{C_4: C_5} & \overbrace{\hspace{8mm}}^{C_5 + 1} & \overbrace{\hspace{8mm}}^{C_5 + 2} \\
                    \begin{block}{(cccccccc)}
                        \mI_{\nstates} & & & & & & & \\
                        & \mI_{|\tapealphabet|} & & & & -\one_{|\tapealphabet|} & -\one_{|\tapealphabet|} & \\
                        & & \mI_{|\actions|} & & & & & \\
                        & & & \mI_{|\tapealphabet|} & & \one_{|\tapealphabet|} & &\\
                        & & & & \mI_{|\tapealphabet|} & & \one_{|\tapealphabet|} &\\
                    \end{block}
                \end{blockarray}
            \end{equation}
        \end{minipage}
        \begin{minipage}{0.25\linewidth}
            \begin{equation}
                \vb^3 = \begin{blockarray}{c}
                    \vphantom{1} \\
                    \begin{block}{(c)}
                        \zero_\nstates \\
                        \zero_{|\tapealphabet|} \\
                        \zero_{|\actions|} \\
                        -\one_{|\tapealphabet|} \\
                        -\one_{|\tapealphabet|} \\
                    \end{block}
                \end{blockarray}
            \end{equation}
        \end{minipage}
    \end{center}
    \begin{center}
        \noindent\begin{minipage}{0.6\linewidth}
            \begin{equation}
                \mW^4 = \begin{blockarray}{cccccc}
                    \overbrace{\hspace{8mm}}^{1: C_1} & \overbrace{\hspace{8mm}}^{C_1: C_2} & \overbrace{\hspace{8mm}}^{C_2: C_3} & \overbrace{\hspace{8mm}}^{C_3: C_4} & \overbrace{\hspace{8mm}}^{C_4: C_5} \\
                    \begin{block}{(cccccc)}
                        \mI_{\nstates} & & & & & \\
                        & \mI_{|\tapealphabet|} & & \mI_{|\tapealphabet|} & \mI_{|\tapealphabet|} & \\
                        & & \mI_{|\actions|} & & &\\
                    \end{block}
                \end{blockarray}
            \end{equation}
        \end{minipage}
        \begin{minipage}{0.3\linewidth}
            \begin{equation}
                \vb^4 = \begin{blockarray}{c}
                    \vphantom{1} \\
                    \begin{block}{(c)}
                        \zero_\nstates \\
                        \zero_{|\tapealphabet|} \\
                        \zero_{|\actions|} \\
                    \end{block}
                \end{blockarray}
            \end{equation}
        \end{minipage}
    \end{center}
    Lastly, we set $\mW^5$ and $\vb^5$ such that the the function $\vx \mapsto \ReLUfunc{\mW^5 \vx + \vb^5}$ computes the one-hot encoding of the three input one-hot encodings $\onehot{\stateq_\tstep}, \onehot{s_\tstep}, \onehot{a_{\tstep - 1}}$ by implementing the logic \texttt{AND} operations, as in \citet[Lemma B.1]{svete-etal-2024-transformers}.

    By construction, we get that
    \begin{subequations}
        \begin{align}
            \vy^1
             & = \ReLUfunc{\mW^1 \vx^2_\tstep + \vb^1}                                           \\
             & = \begin{pmatrix}
                     \onehot{\stateq_\tstep}                                         \\
                     \onehot{v_{\tstep - 1}}                                         \\
                     \onehot{a_{\tstep - 1}}                                         \\
                     \onehot{\blanksym}                                              \\
                     \onehot{\bot}                                                   \\
                     \ReLUfunc{\lastVisit{\tstep} + 1 + 1 - \left(\tstep + 1\right)} \\
                     \ReLUfunc{2 - \left(\tstep + 1\right)}                          \\
                 \end{pmatrix} \\
             & = \begin{pmatrix}
                     \onehot{\stateq_\tstep}                    \\
                     \onehot{v_{\tstep - 1}}                    \\
                     \onehot{a_{\tstep - 1}}                    \\
                     \onehot{\blanksym}                         \\
                     \onehot{\bot}                              \\
                     \ReLUfunc{\lastVisit{\tstep} + 1 - \tstep} \\
                     \ReLUfunc{1 - \tstep}                      \\
                 \end{pmatrix}                      \\
             & = \begin{pmatrix}
                     \onehot{\stateq_\tstep}              \\
                     \onehot{v_{\tstep - 1}}              \\
                     \onehot{a_{\tstep - 1}}              \\
                     \onehot{\blanksym}                   \\
                     \onehot{\bot}                        \\
                     \ind{\lastVisit{\tstep} \geq \tstep} \\
                     \ind{\tstep < 1}                     \\
                 \end{pmatrix}                            \\
             & = \begin{pmatrix}
                     \onehot{\stateq_\tstep}           \\
                     \onehot{v_{\tstep - 1}}           \\
                     \onehot{a_{\tstep - 1}}           \\
                     \onehot{\blanksym}                \\
                     \onehot{\bot}                     \\
                     \ind{\lastVisit{\tstep} = \tstep} \\
                     \ind{\tstep = 0}                  \\
                 \end{pmatrix}
        \end{align}
    \end{subequations}
    \begin{subequations}
        \begin{align}
            \vy^2
             & = \ReLUfunc{\mW^2 \vy^1 + \vb^2}                                                  \\
             & = \begin{pmatrix}
                     \onehot{\stateq_\tstep}                                         \\
                     \onehot{v_{\tstep - 1}}                                         \\
                     \onehot{a_{\tstep - 1}}                                         \\
                     \onehot{\blanksym}                                              \\
                     \onehot{\bot}                                                   \\
                     \ReLUfunc{\ind{\lastVisit{\tstep} = \tstep} - \ind{\tstep = 0}} \\
                     \ind{\tstep = 0}                                                \\
                 \end{pmatrix} \\                                     \\
             & = \begin{pmatrix}
                     \onehot{\stateq_\tstep}                                     \\
                     \onehot{v_{\tstep - 1}}                                     \\
                     \onehot{a_{\tstep - 1}}                                     \\
                     \onehot{\blanksym}                                          \\
                     \onehot{\bot}                                               \\
                     \ind{\ind{\lastVisit{\tstep} = \tstep} \& \ind{\tstep > 0}} \\
                     \ind{\tstep = 0}                                            \\
                 \end{pmatrix}
        \end{align}
    \end{subequations}
    \begin{subequations}
        \begin{align}
            \vy^3
             & = \ReLUfunc{\mW^3 \vy^2 + \vb^3}                                                                                                                                    \\
             & = \begin{pmatrix}
                     \onehot{\stateq_\tstep}                                                                                                                        \\
                     \ReLUfunc{\onehot{v_{\tstep - 1}} - \left(\ind{\lastVisit{\tstep} = \tstep \land \tstep > 0} + \ind{\tstep = 0}\right) \one_{|\tapealphabet|}} \\
                     \onehot{a_{\tstep - 1}}                                                                                                                        \\
                     \ReLUfunc{\onehot{\blanksym} + \ind{\lastVisit{\tstep} = \tstep \land \tstep > 0} \one_{|\tapealphabet|} - \one_{|\tapealphabet|}}             \\
                     \ReLUfunc{\onehot{\bot} + \ind{\tstep = 0} \one_{|\tapealphabet|} - \one_{|\tapealphabet|}}                                                    \\
                 \end{pmatrix} \\
             & = \begin{pmatrix}
                     \onehot{\stateq_\tstep}                                                                                                       \\
                     \ind{\neg\left(\lastVisit{\tstep} = \tstep \land \tstep > 0\right) \land \neg\left(\tstep = 0\right)} \onehot{v_{\tstep - 1}} \\
                     \onehot{a_{\tstep - 1}}                                                                                                       \\
                     \ind{\lastVisit{\tstep} = \tstep \land \tstep > 0} \onehot{\blanksym}                                                         \\
                     \ind{\tstep = 0} \onehot{\bot}                                                                                                \\
                 \end{pmatrix}                     \\
             & = \begin{pmatrix}
                     \onehot{\stateq_\tstep}                                                                                 \\
                     \ind{\left(\lastVisit{\tstep} < \tstep \lor \tstep = 0\right) \land \tstep > 0} \onehot{v_{\tstep - 1}} \\
                     \onehot{a_{\tstep - 1}}                                                                                 \\
                     \ind{\lastVisit{\tstep} = \tstep \land \tstep > 0} \onehot{\blanksym}                                   \\
                     \ind{\tstep = 0} \onehot{\bot}                                                                          \\
                 \end{pmatrix}                                           \\
             & = \begin{pmatrix}
                     \onehot{\stateq_\tstep}                                                    \\
                     \ind{\lastVisit{\tstep} < \tstep \land \tstep > 0} \onehot{v_{\tstep - 1}} \\
                     \onehot{a_{\tstep - 1}}                                                    \\
                     \ind{\lastVisit{\tstep} = \tstep \land \tstep > 0} \onehot{\blanksym}      \\
                     \ind{\tstep = 0} \onehot{\bot}                                             \\
                 \end{pmatrix}
        \end{align}
    \end{subequations}
    \begin{subequations}
        \begin{align}
            \vy^4
             & = \ReLUfunc{\mW^4 \vy^3 + \vb^4}                                                                                                                                                                                \\
             & = \begin{pmatrix}
                     \onehot{\stateq_\tstep}                                                                                                                                                             \\
                     \ind{\lastVisit{\tstep} < \tstep \land \tstep > 0} \onehot{v_{\tstep - 1}} + \ind{\lastVisit{\tstep} = \tstep \land \tstep > 0} \onehot{\blanksym} + \ind{\tstep = 0} \onehot{\bot} \\
                     \onehot{a_{\tstep - 1}}                                                                                                                                                             \\
                 \end{pmatrix} \label{eq:log-expr} \\
             & \defeq \begin{pmatrix}
                          \onehot{\stateq_\tstep} \\
                          \onehot{w_\tstep}       \\
                          \onehot{a_{\tstep - 1}} \\
                      \end{pmatrix}
        \end{align}
    \end{subequations}
    Since the three logical expressions in \Cref{eq:log-expr} are complementary, it holds that the second component of $\vy^4$ holds either the value of $v_{\tstep - 1}$, $\blanksym$, or $\bot$ depending on the value of $\lastVisit{\tstep}$ and $\tstep$.
    This is exactly the symbol that will be read by the PTM at time step $\tstep$, i.e., $s_\tstep$: If $\lastVisit{\tstep} < \tstep$ (which means that cell $\posAt{\tstep}$ has been visited before), then $s_\tstep = v_{\tstep - 1}$; if $\lastVisit{\tstep} = \tstep$ and $\tstep > 0$ (which means that cell $\posAt{\tstep}$ has not been visited before, and $\tstep > 0$), then $s_\tstep = \blanksym$; and if $\tstep = 0$ (meaning that the PTM just started executing and is still reading the initial symbol $\bot$), then $s_\tstep = \bot$.
    By the construction of $\mW^5$ and $\vb^5$, we get that $\vy^5 = \onehot{\stateq_\tstep, s_\tstep, a_{\tstep - 1}}$.
\end{proof}

\begin{lemma}[Correctness of the sampling step] \label{lem:sampling}
    Define the output matrix $\outMtx \in \R^{\nstates |\stackalphabet| |\eosalphabet_\eps| |\actions| |\actions| \times \nstates |\tapealphabet| |\actions|}$ as
    \begin{equation}
        \eOutMtx_{\left(\stateq^\prime, v, \eossym, a^\prime, a\right), \left(\stateq, s, a\right)} \defeq \begin{cases}
            \pLM\left(\stateq^\prime, v, \eossym, a^\prime, a \mid \stateq, s\right) & \ifcondition \eossym \neq \eos \\
            \pLM\left(\qfinal, v, \eossym, a^\prime, a \mid \stateq, s\right)  & \otherwisecondition
        \end{cases}
    \end{equation}
    for $\stateq, \stateq^\prime \in \states$, $v \in \tapealphabet$, $\eossym \in \eosalphabet_\eps$, $a \in \actions$, $a^\prime \in \set{-1, 1}$, and $s \in \stackalphabet$.
    Then, it holds that
    \begin{equation}
        \outMtx \, \fTransf\left(\vx^2_\tstep\right) = \pLM\left(\cdot \mid \stateq_\tstep, s_\tstep\right).
    \end{equation}
\end{lemma}
\begin{proof}
    Follows directly from \Cref{lem:output-function} and the construction of $\outMtx$:By \Cref{lem:output-function}, we have that
    \begin{equation}
        \fTransf\left(\vx^2_\tstep\right) = \onehot{\stateq_\tstep, s_\tstep, a_{\tstep - 1}}.
    \end{equation}
    By construction of $\outMtx$, we have that
    \begin{subequations}
        \begin{align}
            \left(\outMtx \, \fTransf\left(\vx^2_\tstep\right)\right)_{\left(\stateq^\prime, v, \eossym, a^\prime, a\right)}
             & = \left(\outMtx \, \onehot{\stateq_\tstep, s_\tstep, a_{\tstep - 1}}\right)_{\left(\stateq^\prime, v, \eossym, a^\prime, a\right)} \\
             & = \outMtx_{\left(\stateq_\tstep, s_\tstep, a_{\tstep - 1}\right), \left(\stateq^\prime, v, \eossym, a^\prime, a\right)}            \\
             & = \begin{cases}
                     \pLM\left(\stateq^\prime, v, \eossym, a^\prime, a \mid \stateq, s\right) & \ifcondition \eossym \neq \eos \\
                     \pLM\left(\qfinal, v, \eossym, a^\prime, a \mid \stateq, s\right)  & \otherwisecondition
                 \end{cases}.
        \end{align}
    \end{subequations}
\end{proof}

\end{document}